\documentclass{article}

\usepackage[T1]{fontenc}
\usepackage[utf8]{inputenc}
\usepackage[english]{babel}
\usepackage{CJKutf8}

\usepackage{hyperref}
\usepackage[accepted]{icml2026}
\hypersetup{colorlinks=false, hidelinks}

\usepackage{microtype}

\usepackage{amsmath, amssymb, amsthm, bm, mathtools}

\usepackage{booktabs}
\usepackage{multirow}

\usepackage{xcolor}
\definecolor{tabblue}{HTML}{1F77B4}
\definecolor{tabred}{HTML}{D62728}

\usepackage{caption}

\usepackage[lined,linesnumbered,ruled,vlined,algo2e]{algorithm2e}
\SetKw{KwForIn}{in}
\SetKw{KwUntil}{until}
\SetKwInput{KwInputs}{Inputs}
\SetKwInput{KwOutputs}{Outputs}

\newtheorem{assumption}{Assumption}
\newtheorem{theorem}{Theorem}
\newtheorem{proposition}{Proposition}
\newtheorem{lemma}{Lemma}
\newtheorem{corollary}{Corollary}
\newtheorem{remark}{Remark}

\usepackage{tikz}
\usetikzlibrary{fadings}
\usetikzlibrary{patterns}
\usetikzlibrary{shadows.blur}
\usetikzlibrary{shapes}

\newcommand{\wrapbox}[2]{\resizebox{#1\linewidth}{!}{%
	#2
}%
}

\newcommand{\cu}[1]{
	\ifcat\noexpand#1\relax
	\bm{#1}
	\else
	\mathbf{#1}
	\fi
}

\newcommand{\diff}{\mathop{}\!\mathrm{d}}

\newcommand{\expp}{\mathrm{e}}

\newcommand{\cond}{{\;|\;}}

\let\sup\relax
\let\inf\relax
\let\lim\relax
\DeclareMathOperator*{\sup}{sup\,}  
\DeclareMathOperator*{\inf}{inf\,}  
\DeclareMathOperator*{\lim}{lim\,}  

\let\grad\relax
\DeclareMathOperator{\grad}{\nabla\!}

\newcommand{\expecsym}{\operatorname{\mathbb{E}}}     
\newcommand{\covsym}{\operatorname{Cov}}     
\newcommand{\varrsym}{\operatorname{Var}}     
\newcommand{\diagsym}{\operatorname{diag}}     
\newcommand{\tracesym}{\operatorname{tr}}           

\let\expec\relax
\let\cov\relax
\let\varr\relax
\let\diag\relax
\let\trace\relax

\makeatletter
\newcommand{\expec}{\@ifstar{\@expecauto}{\@expecnoauto}}
\newcommand{\@expecauto}[1]{\expecsym \left[ #1 \right]}
\newcommand{\@expecnoauto}[1]{\expecsym [#1]}
\newcommand{\expecbig}[1]{\expecsym \bigl[ #1 \bigr]}
\newcommand{\expecBig}[1]{\expecsym \Bigl[ #1 \Bigr]}
\newcommand{\expecbigg}[1]{\expecsym \biggl[ #1 \biggr]}

\newcommand{\cov}{\@ifstar{\@covauto}{\@covnoauto}}
\newcommand{\@covauto}[1]{\covsym \left[ #1 \right]}
\newcommand{\@covnoauto}[1]{\covsym [#1]}

\newcommand{\varr}{\@ifstar{\@varrauto}{\@varrnoauto}}
\newcommand{\@varrauto}[1]{\varrsym \left[ #1 \right]}
\newcommand{\@varrnoauto}[1]{\varrsym [#1]}

\newcommand{\varrBig}[1]{\varrsym \Bigl[ #1 \Bigr]}

\newcommand{\diag}{\@ifstar{\@diagauto}{\@diagnoauto}}
\newcommand{\@diagauto}[1]{\diagsym \left( #1 \right)}
\newcommand{\@diagnoauto}[1]{\diagsym (#1)}

\newcommand{\trace}{\@ifstar{\@traceauto}{\@tracenoauto}}
\newcommand{\@traceauto}[1]{\tracesym \left( #1 \right)}
\newcommand{\@tracenoauto}[1]{\tracesym (#1)}

\makeatother

\newcommand*{\trans}{{\mkern-1.5mu\mathsf{T}}}

\newcommand*{\R}{\mathbb{R}} 

\let\norm\relax
\DeclarePairedDelimiter{\normbracket}{\lVert}{\rVert}
\newcommand{\norm}{\normbracket}

\let\innerp\relax
\DeclarePairedDelimiter{\innerpbracket}{\langle}{\rangle}
\newcommand{\innerp}{\innerpbracket}
\newcommand{\innerpbig}[1]{\bigl\langle #1 \bigr\rangle}

\let\abs\relax
\DeclarePairedDelimiter{\absbracket}{\lvert}{\rvert}
\newcommand{\abs}{\absbracket}
\newcommand{\absbig}[1]{\bigl \lvert #1 \bigr\rvert}

\newcommand{\absbigg}[1]{\biggl \lvert #1 \biggr\rvert}

\newcommand{\klbig}[2]{\mathrm{KL}\bigl( #1 \, \Vert \, #2 \bigr)}

\makeatletter
\@ifundefined{thmenvcounter}{}
{%
	\newtheorem{envcounter}{EnvcounterDummy}[\thmenvcounter]
	
	\newtheorem{proposition}[envcounter]{Proposition}
	
	\newtheorem{corollary}[envcounter]{Corollary}
	\newtheorem{remark}[envcounter]{Remark}

	\newtheorem{assumption}[envcounter]{Assumption}
}
\makeatother

\usepackage[textsize=tiny]{todonotes}

\begin{document}

\twocolumn[
\icmltitle{Diffusion differentiable resampling}




\begin{icmlauthorlist}
\icmlauthor{Jennifer Rosina Andersson}{ja}
\icmlauthor{Zheng Zhao}{zz}
\end{icmlauthorlist}

\icmlaffiliation{ja}{Department of Information Technology, Uppsala University, Sweden}
\icmlaffiliation{zz}{Division of Statistics and Machine Learning, Link\"{o}ping University, Sweden}

\icmlcorrespondingauthor{Zheng Zhao (\begin{CJK*}{UTF8}{bsmi}趙正\end{CJK*})}{zheng.zhao@liu.se}

\icmlkeywords{Machine Learning, diffusion models, particle filters, sequential Monte Carlo, resampling, differentiable, reparametrisation, Feynman--Kac}

\vskip 0.3in
]



\printAffiliationsAndNotice{}  

\newcommand{\refmeasure}{\pi_{\mathrm{ref}}}

\begin{abstract}
This paper is concerned with differentiable resampling in the context of sequential Monte Carlo (e.g., particle filtering). 
Drawing on reparametrisation, we propose a new resampling method that is informative and instantly differentiable, based on a training-free diffusion model surrogate. 
We theoretically prove that our diffusion resampling method provides a consistent resampling distribution, and we show empirically that it outperforms the state-of-the-art differentiable resampling methods on multiple filtering and parameter estimation benchmarks. 
Finally, we show that it achieves competitive end-to-end performance when used in learning a complex dynamics-decoder model with high-dimensional image observations.
\end{abstract}

\section{Introduction}
\label{sec:intro}
Consider a distribution $\pi$ and a population of weighted samples $\lbrace (w_i, X_i) \rbrace_{i=1}^N \sim \pi$, where $\lbrace X_i \rbrace_{i=1}^N$ are often identically and independently drawn from another proposal distribution, calibrated by the weights. 
The goal of \emph{resampling} is to transform these weighted samples into an un-weighted set while preserving the original distribution $\pi$. 
In the weak sense, unbiased resampling is defined as a mapping
\begin{equation}
    \begin{split}
        \lbrace (w_i, X_i) \rbrace_{i=1}^N &\mapsto \Bigl\lbrace \Bigl(\frac{1}{N}, X_i^*\Bigr)\Bigr\rbrace_{i=1}^N, \\
        \text{s.t.} \quad \expecbigg{\frac{1}{N}\sum_{i=1}^N\psi(X_i^*)} &= \expec{\psi(X)}
    \end{split}
\end{equation}
for any bounded and continuous test function $\psi$, where $\lbrace X_i^* \rbrace_{i=1}^N$ stands for the re-samples, and $X\sim \pi$. 

Resampling is a key component in sequential Monte Carlo (SMC) samplers for state-space models~\citep[SSMs,][]{ChopinBook2020}.
It not only mitigates particle degeneracy in practice, but has also been shown to represent a jump Markov process in a continuous-time limit~\citep{Chopin2022PDMP}.
However, resampling typically hinders (gradient-based) parameter estimation of SSMs due to discrete randomness. 
For instance, with the commonly used multinomial resampling, one draws indices $I_i \sim \mathrm{Categorical}(w_1, w_2, \ldots, w_N)$ independently for $i=1,2,\ldots, N$, and then defines the resampling mapping by indexing $X^*_i \coloneqq X_{I_i}$. 
If the sample $X^\theta_i$ (and weight) depend on some unknown parameters $\theta$, the pathwise derivative of the re-sample, $\partial X_i^{\theta, *} / \, \partial \theta$, is not defined. 
Moreover, most automatic differentiation libraries (e.g., JAX) will typically drop the undefined derivatives, resulting in erroneous gradient estimates~\citep[see, e.g.,][]{Naesseth2018}. 

Various methods have thus been proposed to make the resampling step differentiable in the context of SMC and SSMs~\citep[see recent surveys in][]{Xiongjie2025, Brady2025pydpf}. 
One line of work focuses on the expectation derivative $\partial \expecbig{X_i^{\theta, *}} / \, \partial \theta$, which is often well defined, rather than the pathwise derivative. 
This can be achieved by combing Fisher's score~\citep{Poyiadjis2011} and stochastic derivatives~\citep{Arya2022}, resulting in, for example, the stop-gradient based method by~\citet{Scibior2021differentiable}. 
However, these REINFORCE-based methods often suffer from high variance and may consequently require a large sample size $N$. 

Another line of work focuses on developing new resampling methods that naturally come with well-defined pathwise derivatives (i.e., reparapemtrisation). 
Notable examples include soft~\citep{Karkus2018} and Gumbel-Softmax resampling~\citep{Jang2017categorical}. 
Both methods essentially form an interpolation between multinomial resampling (which has no gradient) and uninformative resampling (which has a gradient) via a calibration parameter. 
Although they have been empirically shown to work well for certain models, they are fundamentally biased. 
Crucially, one has to decide on a trade-off between the gradient bias and the statistical performance of the forward resampling mapping. 
In addition, \citet{Zhu2020towards} propose to parametrise the resampling with a neural network, which adds both training complexity and additional sources of bias.
\citet{Kviman24a} introduce a deterministic resampling, which also comes with irreducible biases~\citep{Finke2026aistats}.

The perhaps first fully-differentiable-by-construction reparametrisation with consistency guarantees is due to~\citet{Malik2011}, who make a smooth approximation of the empirical cumulative distribution function (CDF) of the weighted samples. 
However, their method focuses on univariate $X$. 
This was later generalised by~\citet{Li2024differentiable} using kernel jittering to approximate the CDF gradient. 
In a similar vein, \citet{Corenflos2021} propose an optimal transport (OT) based resampling method. 
The idea is to learn a (linear) transportation map between the target distribution $\pi$ and the proposal, and approximate the resampling by an ensemble transformation~\citep{Reich2013}
\begin{equation}
    X_i^* = N \, \sum_{j=1}^N P_{i, j}^{\varepsilon} \, X_{j}, \quad i=1,2,\ldots, N,
    \label{equ:eot}
\end{equation}
where $P^\varepsilon\in\R^{N\times N}$ denotes the $\varepsilon$-regularised entropic optimal coupling.
The main problem of OT-based resampling lies in computation, since one has to compute for the transportation plan $P^{\varepsilon}$. 
The cost scales quadratically in the number of samples $N$ with a Sinkhorn implementation, which in turn depends exponentially on the entropy parameter $1\,/\,\varepsilon$~\citep{Luo2023OT, Burns2025linear} to converge. 
In the context of SMC, the method may not work well when the proposal/reference does not well approximate the target. 
\citet[][Fig. 1]{Li2024differentiable} also show a (hypothetical) case when linear transformation of OT is insufficient for exploiting the distribution manifold. 
Nevertheless, this line of work has inspired us to develop a new transportation-based resampling method that avoids these issues. 

Other than making the resampling differentiable, one can also modify the SMC algorithm itself to produce smooth estimates of the marginal log likelihood of SSMs. 
For instance, \citet{Klaas2005} introduce a mixture of SMC proposals to marginalise out the need for resampling, however, in practice, drawing samples from mixture distributions typically also requires discrete sampling. 
This was addressed by~\citet{Lai2022} using implicit reparametrisation, but the approach remains restricted to structured proposals.
Overall, this class of methods rely on customised SMC samplers, potentially limiting their applicability in general. 

Therefore, our main motivation in this paper is to develop a differentiable resampling method that can be generically applied \emph{as is}, without altering the SMC (or SSMs) construction, sacrificing consistency, or increasing computational cost. 
We take inspiration from the transportation-based approach~\citep{Corenflos2021}, but our key departure here is the transport map construction: it needs not to be \emph{computed} but rather \emph{specified}, thereby mitigating the computation issue. 
Moreover, our construction allows for integrating additional information of the target into the map to make it statistically more efficient. 
Our contributions are as follows.
\begin{itemize}
    \item We introduce \emph{diffusion resampling}, a new reparametrisation paradigm that instantly enables automatic differentiation for $\partial X_i^{\theta, *} / \, \partial \theta$, and consequently also for the expectation $\partial \expecbig{X_i^{\theta, *}} / \, \partial \theta$.
    We apply the method for filtering and gradient-based parameter estimation in state-space models with SMC samplers.
    \item We prove that our diffusion resampling method is consistent in the number of samples. 
    We show an informative error bound in Wasserstein distance, explicitly quantifying the error propagation of the resampling. 
    \item We empirically validate our method through both ablation and comparison experiments. 
    The results show that our method consistently outperforms the commonly used differentiable resampling methods for both filtering and parameter estimation problems. 
    Notably, our method is computationally efficient and stable, allowing for practical and robust usage in applications. 
\end{itemize}

See Table~\ref{tbl:app-compare-all} for a comparison of diffusion resampling to the commonly used differentiable resampling schemes. 

\section{Diffusion differentiable resampling}
\label{sec:diffres}
In this section we show how we can make use of a diffusion model~\citep[without training, cf.][]{Baker2025, Wan2025diffpf} to construct a differentiable resampling scheme and apply it for sequential Monte Carlo. 
The idea is akin to Equation~\eqref{equ:eot}, which \emph{computes} a linear transportation map, but here we instead use a diffusion model to \emph{construct} a non-linear map. 
We define the diffusion model via a (forward-time) Langevin stochastic differential equation (SDE)
\begin{equation}
    \begin{split}
        \diff X(t) &= b^2\grad \log \refmeasure(X(t)) \diff t + \sqrt{2} \, b \diff W(t), \\
        X(0) &\sim \pi
    \end{split}
    \label{equ:fwd}
\end{equation}
initialised at the target $\pi$, where $\refmeasure$ is a user-chosen reference distribution from which we can easily sample (e.g., Gaussian), $W$ is a Brownian motion, and $b$ is a dispersion constant. 
Under mild conditions~\citep{Meyn2009}, the marginal distribution $p_t$ of $X(t)$ converges to $\refmeasure$ geometrically fast as $t\to\infty$. 
Importantly, \citet{Song2021scorebased, Anderson1982} show that we can leverage this construction to sample from $\pi$ if we can sample $p_T$ at some terminal time $T>0$ and simulate the reverse-time SDE\footnote{To simplify later analysis we assume that the Brownian motions in Equations~\eqref{equ:fwd} and~\eqref{equ:rev} are the same.}
\begin{align}
    \diff U(t) &= b^2 \bigl[-\grad \log \refmeasure(U(t)) + 2\grad\log p_{T - t}(U(t)) \bigr]\diff t \nonumber\\
    &\quad + \sqrt{2} \, b \diff W(t), \label{equ:rev}\\
    U(0) &\sim p_T.\nonumber
\end{align}
At time $T$, the marginal distribution $q_T$ of $U(T)$ equals $\pi$ by construction, since $X(T - t)$ and $U(t)$ solve the same Kolmogorov forward equation. 
Resampling can thus be achieved by sampling from this reversal at $T$. 
The challenge is that the score function $\grad \log p_t$ is intractable, and in the context of generative diffusion models the score is usually learnt from samples of $\pi$, introducing demanding computations~\citep{Zhao2024rsta}. 
However, given that we have access to $\lbrace (w_i, X_i) \rbrace_{i=1}^N \sim \pi$, we can approximate the score~\citep{BaoEnsScore2024} without training according to
\begin{equation}
    \begin{split}
        &\grad \log p_t(x_t) \\
        &\quad= \frac{\int \grad \log p_{t|0}(x_t \cond x_0) \, p_{t|0}(x_t \cond x_0)\, \pi(x_0) \diff x_0}{\int p_{t|0}(x_t \cond x_0) \, \pi(x_0) \diff x_0} \\
        &\quad\approx \frac{\sum_{i=1}^N w_i \grad \log p_{t|0}(x_t \cond X_i) \, p_{t|0}(x_t \cond X_i)}{\sum_{j=1}^N w_j \, p_{t|0}(x_t \cond X_j)},
    \end{split}
\end{equation}
where the transition $p_{t | 0}$ is analytically tractable for many useful choices of the reference $\refmeasure$. 
For later analysis, we define the score approximation by
\begin{equation}
    \begin{split}
        \begin{split}
            \grad\log p_t(x) &\approx s_N(x, t) \\
            &\coloneqq \sum_{i=1}^N \alpha_i(x, t) \grad\log p_{t|0}(x \cond X_i),
        \end{split}
    \end{split}
    \label{equ:ensemble-score}
\end{equation}
where $\alpha_i(x, t) \coloneqq w_i \, p_{t|0}(x \cond X_i) \, / \sum_{j=1}^N w_j \, p_{t|0}(x \cond X_j)$ stands for the normalised weight. 
This approximation exactly functions as importance sampling, where $\pi(x_0)$ and $p_{t|0}(\cdot \cond x_0)$ stand for the prior/proposal and likelihood, respectively. 
As such, the established $N\to\infty$ consistency properties of importance sampling apply~\citep{ChopinBook2020} at least pointwise for $(x, t) \mapsto s_N(x, t)$.
Evaluation of the function $s$ has an $O(N)$ computational cost if implemented na\"{i}vely, and a logarithmic cost if implemented in parallel~\citep{Lee01012010}. 

Therefore, the resampling can be approximately achieved by simulating the reverse SDE~\eqref{equ:rev} using the ensemble score $s_N$ in Equation~\eqref{equ:ensemble-score} until a terminal time $T$. 
Similar to~\citet{Corenflos2021}, this diffusion process too defines an optimal transportation but in the sense of Jordan--Kinderlehrer--Otto scheme~\citep{Jordan1998}. 
The key distinction is that this map is \emph{given by construction}, and does not need to be \emph{computed} like in OT with Sinkhorn. 
Although the diffusion also assumes $T\to\infty$, we show in Section~\ref{sec:analysis} that $T$ scales better than the entropy parameter $1\,/\,\varepsilon$ as a function of $N$. 

We summarise the diffusion resampling in Algorithm~\ref{alg:diffres} using a simple Euler--Maruyama discretisation for pedagogy. 
It is immediate by construction that this function is differentiable, since the only source of randomness is Gaussian which is reparameterisable. 

\begin{remark}
    \label{remark:doob}
    The ensemble score in Equation~\eqref{equ:ensemble-score} characterises a Doob's $h$-function:
    \begin{equation}
        s_N(x, t) = \grad \log \sum_{i=1}^N h_i(x, t),
    \end{equation}
    where $h_i(x, t) \coloneqq w_i \, p_{t | 0}(x \cond X_i)$ verifies the martingale (harmonic) property, that is, $\sum_{i=1}^N h_i$ is a valid $h$-function under SDE~\eqref{equ:fwd}. 
    As a result, the diffusion resampling process at the terminal time will obtain $\sum_{i=1}^N \gamma_i \, \delta_{X_i}$ for weights $\lbrace \gamma_i \rbrace_{i=1}^N$ that depend on the spatial location, akin to the OT approach.
    When setting $p_T(x)=\sum_{i=1}^N h_i(x, T)$, we obtain a special case $\lbrace \gamma_i=w_i \colon i=1,\ldots, N \rbrace$,  
    and diffusion resampling may thus be viewed as a continuous and differentiable reparametrisation of discrete multinomial resampling. 
    However, the key advantage here is that the diffusion resampling can leverage the additional information from $\refmeasure$ (which implicitly defines a transportation cost and a Rao--Blackwellisation condition) achieving better statistical properties (e.g., variance). 
    See Appendix~\ref{app:elaboration-doob} for elaboration.
\end{remark}

\begin{algorithm2e}[t!]
    \SetAlgoLined
    \DontPrintSemicolon
    \KwInputs{Weighted samples $\lbrace (w_i, X_i) \rbrace_{i=1}^N$, reference distribution $\refmeasure$, time grid $0=t_0 < t_1 < \cdots < t_K=T$, and $b$. }
    \KwOutputs{Differentiable re-samples $\lbrace (\frac{1}{N}, X^{*}_i) \rbrace_{i=1}^N$}
    \For(\tcp*[f]{parallel}){$i=1,2,\ldots, N$}{%
        $U_{i, 0} \sim \refmeasure$\;
        \For{$k=1, 2, \ldots, K$}{%
            $\Delta_k = t_k - t_{k-1}$\;
			Draw $\xi_k^i \sim \mathrm{N}(0, 2 \,  b^2 \, \Delta_k \, I_d)$\;
            $U_{i, t_{k}} = U_{i, t_{k-1}} - b^2 \bigl[\grad \log \refmeasure(U_{i, t_{k-1}}) - 2 \, s_N(U_{i, t_{k-1}}, T - t_{k-1}) \bigr]\, \Delta_k + \xi_k^i$
        }
        $X^*_i \gets U_{i, T}$
    }
    \caption{Diffusion resampling \texttt{diffres}}
    \label{alg:diffres}
\end{algorithm2e}

\subsection{Differentiable sequential Monte Carlo}
\label{sec:diffres-smc}
The diffusion resampling in Algorithm~\ref{alg:diffres} is particularly useful for SMC and SSMs for two reasons. 
First, the resampling method is pathwise-differentiable by construction. 
Moreover, recent advances provide methods for propagating gradients through SDE solvers~\citep[see, e.g.,][]{Bartosh2025sde, LiXueChen2020, Kidger2021gradient} and these methods have been well implemented in common automatic differentiation libraries~\citep{Kidger2021on}.
Secondly, we can fully leverage the sequential structure of SMC to adaptively choose for a suitable reference distribution $\refmeasure$ based on any previous SMC marginal distribution.
To see these aspects, let us begin by considering a parametrised Feynman--Kac model
\begin{equation}
    Q^\theta_{0:J}(z_{0:J}) = \frac{1}{L(\theta)} \prod_{j=0}^J M_{j}^\theta(z_j \cond z_{j-1}) \, G_j^\theta(z_j, z_{j-1}), 
	\label{equ:fk}
\end{equation}
where $M^\theta_j$ and $G_j^\theta$ are Markov transition and potential functions, respectively, and $L(\theta)$ is the marginal likelihood that we often aim to maximise. 
Take an SSM with state transition $p_\theta(z_j \cond z_{j-1})$ and measurement $p_\theta(y_j \cond z_j)$ for example. 
In this case a bootstrap construction of the corresponding Feynman--Kac model is simply $M_j^\theta(z_j \cond z_{j-1}) = p_\theta(z_j \cond z_{j-1})$ and $G_j^\theta(z_j, \cdot) = p_\theta(y_j \cond z_j)$. 
One can draw samples of a Feynman--Kac model with an SMC sampler as in Algorithm~\ref{alg:smc}.
For detailed exposition of Feynman--Kac models and SMC samplers, we refer the readers to~\citet{ChopinBook2020, DelMoral2004}.

\begin{algorithm2e}[h]
	\SetAlgoLined
	\DontPrintSemicolon
	\KwInputs{Feyman--Kac model $Q_{0:J}^\theta$, number of samples $N$, and \texttt{diffres}.}
	\KwOutputs{Weighted samples of $Q_{0:J}^\theta$ and marginal likelihood estimate $L(\theta)$. }
	Draw i.i.d. samples $\lbrace Z_{0, i} \rbrace_{i=1}^N \sim M_0^\theta$.\;
    $L_0(\theta) \gets \sum_{i=1}^N G^\theta_0(Z_{0, i})$\;
	Weight $w_{0, i} \gets G_0^\theta(Z_{0, i}) \, / \, L_0(\theta)$ \;
	\For(\tcp*[f]{$i=1,2,\ldots, N$}){$j=1,2,\ldots, J$}{%
        \uIf{resampling needed}{%
            $Z_{j-1, i}^* \gets \texttt{diffres}\Bigl(\bigl\lbrace (w_{j-1, i}, Z_{j-1, i}) \bigr\rbrace_{i=1}^N \Bigr)$\;
            $w_{j-1, i} \gets 1 \, / \, N$\;
        }
        \Else{
            $Z_{j-1, i}^* \gets Z_{j-1, i}$ \;
        }
		Draw $Z_{j, i} \sim M^\theta_{j}(\cdot \cond Z^*_{j-1, i})$ \;
        $\overline{w}_{j, i} \gets w_{j-1, i} \, G_j^\theta(Z_{j, i}, Z_{j-1, i}^*)$\;
		$L_j(\theta) \gets \sum_{i=1}^N \overline{w}_{j, i}$\;
        Weight $w_{j, i} \gets \overline{w}_{j, i} \, / \, L_j(\theta)$ \;
	}
    $L(\theta) \gets \prod_{j=0}^J L_j(\theta)$\;
	\caption{Differentiable sequential Monte Carlo (SMC) for sampling Feynman--Kac model $Q_{0:J}$. }
	\label{alg:smc}
\end{algorithm2e}

This choice of reference $\refmeasure$ is more flexible compared to that of~\citet{Corenflos2021} who explicitly use the predictive samples $\lbrace (w_{j-1, i}, Z_{j, i})\rbrace_{i=1}^N$ at the $j$-th SMC step as the reference to resample $\lbrace (w_{j, i}, Z_{j, i})\rbrace_{i=1}^N$. 
In contrast, diffusion resampling allows to use, e.g., the posterior samples $\lbrace (w_{j, i}, Z_{j, i})\rbrace_{i=1}^N$ to establish the reference, which can be more informative than the predictive one. 

\subsection{Mean-reverting Gaussian reference}
\label{sec:ref-distribution}
A remaining question is how to choose the reference distribution $\refmeasure$.
In generative sampling, the reference is usually a unit Normal. 
However, this becomes suboptimal for resampling when the target distribution $\pi$ is geometrically far away from $\mathrm{N}(0, I_d)$, and as a consequence we would need large enough $T$ for convergence. 
A more informative choice of $\refmeasure$ is a Gaussian approximation of $\pi$. 
Given that we have access to the target samples $\lbrace (w_i, X_i) \rbrace_{i=1}^N \sim \pi$, we can make use of moment matching, where $\mu_N \coloneqq \sum_{i=1}^N w_i \, X_i$ and $\Sigma_N \coloneqq \sum_{i=1}^N w_i \, (X_i - \mu_N) \, (X_i - \mu_N)^\trans$ stand for the empirical mean and covariance, respectively~\citep{Yang2013, Kang2025}. 
We can then choose the reference measure to be
\begin{equation}
    \begin{split}
        \grad\log \refmeasure(x) = -\Sigma_N^{-1} \, (x - \mu_N), 
    \end{split}
\end{equation}
resulting in a mean-reverting forward SDE
\begin{equation}
    \begin{split}
        \diff X(t) = -b^2 \, \Sigma_N^{-1} \, (X(t) - \mu_N) \diff t + \sqrt{2} \, b \diff W(t)
    \end{split}
    \label{equ:gaussian-fwd}
\end{equation}
whose forward transition required in Equation~\eqref{equ:ensemble-score} is
\begin{equation*}
    \begin{split}
        p_{t|0}(x_t \cond x_0) &= \mathrm{N}(x_t ; m_t(x_0), V_t), \\
        m_t(x_0) &\coloneqq x_0 \, \expp^{-b^2 \, \Sigma_N^{-1} \, t} + \mu_N \, \bigl(1 - \expp^{-b^2 \, \Sigma_N^{-1} \, t}\bigr), \\
        V_t &\coloneqq \Sigma_N \, \bigl( 1 - \expp^{-2\,b^2 \, \Sigma_N^{-1} \, t} \bigr).
    \end{split}
\end{equation*}
To combine with the SMC in Algorithm~\ref{alg:smc}, we compute $\mu_N$ and $\Sigma_N$ based on $\bigl\lbrace (w_{j-1, i}, Z_{j-1, i}) \bigr\rbrace_{i=1}^N$, see Appendix~\ref{app:diffres-gaussian} for details. 
It is also possible to leverage any Gaussian filter which is commonly used for approximating the optimal proposal~\citep{Vandermerwe2000}, to establish the reference. 
The gist here is to exploit the sequential structure of SMC to adaptively and informatively choose the reference, instead of assuming a static and uninformative one, such as $\mathrm{N}(0, I_d)$. 
Using mean-reverting SDEs for informative generative sampling have also been used in domain applications, such as image restoration~\citep{Luo2023IRSDE, Luo2024rsta}.

In practice, to avoid solving the inversion $\Sigma_N^{-1}$, we can approximate it as a diagonal, or directly estimate an empirical precision matrix~\citep{Yuan2007, Fan2016}. 
When the Gaussian construction is insufficient for multi-mode targets, one may use a Gaussian mixture, although the associated semigroup needs to be approximated. 
Another option is to transform with a diffeomorphism to obtain a flexible yet tractable reference process~\citep{Deng2020}. 

\subsection{Exponential integrators}
\label{sec:exp-integrator}
When a Gaussian reference $\refmeasure$ is chosen, the reversal corresponding to the forward Equation~\eqref{equ:gaussian-fwd} will have a semi-linear structure. 
Hence, we can leverage this structure by applying exponential integrators to accelerate the sampling so as to reduce the computation caused by discretisation. 
Consider any semi-linear SDE of the form
\begin{equation}
    \diff U(t) = A \, U(t) + f(U(t), t) \diff t + \sqrt{2} \, b \diff W(t), 
\end{equation}
where in our scenario the linear and non-linear parts respectively correspond to
\begin{equation}
    \begin{split}
        A &= b^2 \, \Sigma_N^{-1}, \\
        f(u, t) &= b^2 (2\grad\log p_{T - t}(u) - \Sigma_N^{-1} \, \mu_N).
    \end{split}
\end{equation}
Provided that $A$ is invertible, \citet{Jentzen2009} propose an exponential integrator
\begin{equation*}
    \begin{split}
        U_{t_{k}} &= \expp^{A \, \Delta_k} \, U_{t_{k-1}} + A^{-1} \, (\expp^{A \, \Delta_k} - I_d) \, f(U_{t_{k-1}}) + B_k, \\
        B_k &= \sqrt{2} \, b \int^{t_{k}}_{t_{k-1}} \expp^{(t_{k} - \tau) \, A} \diff W(\tau),
    \end{split}
\end{equation*}
where $\Delta_k \coloneqq t_{k} - t_{k-1}$, and the Wiener integral $B_k$ simplifies to $B_k \sim \mathrm{N}\bigl(0, \Sigma_N \, (\expp^{2 \, A \, \Delta_k} - I_d)\bigr)$. 
The integrator works effectively if the stiffness of the linear part dominates that of the non-linear part~\citep[see conditions in e.g.,][]{Buckwar2011}. 
Indeed, the structure of the approximate score $s_N$ in Equation~\eqref{equ:ensemble-score} is essentially a product between a Softmax function and a linear one, which is Lipschitz. 
However, we note the Lipschitz constant is not uniform for all $t>0$, resulting in explosive $s_N$ near $t=0$, such as with $V_t$. 
This exponential integrator has been empirically shown to work well for generative diffusion models in practice, for instance by~\citet{Lu2025expn}.

In the case when the matrix $A$ is not invertible, which rarely happens for Gaussian $\refmeasure$ but still possibly numerically, one can also use another lower-order integrator by~\citet{Lord2004}:
\begin{equation*}
    \begin{split}
        U_{t_k} &= \expp^{A \, \Delta_k} \, U_{t_{k-1}} + \Delta_k \, \expp^{A \, \Delta_k} f(U_{t_{k-1}}, t_{k-1}) + B_k, \\
        B_k &\sim \mathrm{N}\bigl(0, 2 \, b^2 \, \expp^{2 \, A \, \Delta_k} \Delta_k\bigr),
    \end{split}
\end{equation*}
which was also considered by~\citet{Zhang2023fast}.

In light of Remark~\ref{remark:doob}, it is even possible to simulate the resampling SDE fully in continuous time~\citep[see, e.g.,][]{Schauer2017, Baker2024}, although most of the currently established techniques are still not (yet) pragmatic enough compared to just using a fine discretisation. 

\section{Convergence analysis}
\label{sec:analysis}
In this section we analyse the convergence properties of diffusion resampling in Algorithm~\ref{alg:diffres}. 
Recall the ideal re-sampler in Equation~\eqref{equ:rev}
\begin{equation*}
    \begin{split}
        \diff U(t) &= b^2 \bigl[-\grad \log \refmeasure(U(t)) + 2\grad\log p_{T - t}(U(t)) \bigr]\diff t \\
        &\quad + \sqrt{2} \, b \diff W(t), \\
        U(0) &\sim p_T.
    \end{split}
\end{equation*}
and the corresponding approximation
\begin{align}
    \diff \widetilde{U}(t) &= b^2 \bigl[-\grad \log \refmeasure(\widetilde{U}(t)) + 2 \, s_N(\widetilde{U}(t), T - t) \bigr]\diff t \nonumber\\
    &\quad + \sqrt{2} \, b \diff W(t), \label{equ:diffres-rev} \\
    \widetilde{U}(0) &\sim \refmeasure.\nonumber
\end{align}
Here, we have access to the weighted samples $\lbrace (w_i, X_i) \rbrace_{i=1}^N$ from $\pi$ independent of the considered SDEs. 
We denote the distributions of $U(t)$ and $\widetilde{U}(t)$ by $q_t=p_{T - t}$ and $\widetilde{q}_t$, respectively, and similarly denote the true and approximate re-sample by $U(T)$ and $\widetilde{U}(T)$.
Clearly, there are two sources of errors: 1) the ensemble score approximation $\grad\log p_{t} \approx s_N$, and 2) the initial distribution approximation $p_T \approx \refmeasure$ due to finite time horizon $T$. 
For clarity, we here focus on continuous-time analysis, although discretisation errors may be considered within a similar framework~\citep[see e.g.,][]{Lord2014}. 
We aim to analyse the geometric distance between the resampling distribution $\widetilde{q}_t$ and the target $\pi$ under these errors in relation to $t$ and $N$. 

Unless otherwise needed, for any parameter $T>0$ we assume the usual textbook linear growth and Lipschitz conditions on $X$ and $U$ so that a strong solution exists and $\int^t_0 \abs{X(\tau)} \diff \tau$ has finite variance, see, for instance, \citet[][pp. 289]{Karatzas1991} or~\citet[][Thm. 5.2.1]{Oksendal2003}. 
This also ensures a smooth transition density $p_{t|\tau}$ for all $0\leq \tau < t\leq T $ so that the ensemble score $s_N$ and its gradient in Equation~\eqref{equ:ensemble-score} are pointwise well defined. 
We also assume the existence of the reverse process in the sense of~\citet{Anderson1982}, i.e., the reversal solves the same Kolmogorov forward equation in reverse time, although the established conditions for this are still implicit~\citep[see, e.g.,][]{Haussmann1986, Millet1989}.
Denote the Wasserstein distance by $\mathsf{W}_l^l(p, q) = \inf_{\gamma \in\Gamma(p, q)}\expec{\abs{X - Y}^l}$ for $(X, Y) \sim \gamma$, where $\Gamma$ is the set of all couplings of $(p, q)$. 
We say a distribution $\nu$ is $z$-strongly log-concave if
\begin{equation*}
    \innerpbig{\grad\log \nu(x) - \grad\log \nu(x'), x-x'} \leq -z \, \abs{x - x'}^2,
\end{equation*}
and we denote it by $\nu \preceq z$. 
The assumptions that we globally use are as follows. 

\begin{assumption}[Diffusion conditions]
    \label{assumption:diffusion}
    There exist positive constants $2 \, C_p < C_{\mathrm{ref}} < 2 \, C_{\mathrm{ref}}^- + 2 \, C_p$ such that
    \begin{equation*}
        \begin{split}
            \absbig{\grad\log \refmeasure(x) - \grad\log \refmeasure(x')} &\leq C_{\mathrm{ref}} \, \abs{x - x'}, 
        \end{split}
    \end{equation*}
    for all $x, x'\in\R^d$, $\refmeasure \preceq C_{\mathrm{ref}}^-$, and $p_t \preceq C_p$ for all $t\geq 0$.
\end{assumption}

\begin{assumption}[Ensemble score condition] There exist a constant $r>0$ and a positive non-increasing smooth function $t\mapsto C_{e}(t)$, such that
    \label{assumption-score-approx}
    \begin{equation*}
        \sup_{x\in\R^d} \expecbig{\abs{\grad\log p_t(x) - s_N(x, t)}^2}^{1/2} \leq \frac{C_e(t)}{N^r}.
    \end{equation*}
    for all $t>0$, where $\mathbb{E}$ takes on the weighted samples.
\end{assumption}
Recall that the ensemble score defined in Equation~\eqref{equ:ensemble-score} is exactly a self-normalised importance sampling, where $\pi$ is the proposal, $\pi(x_0 \cond x_t) \propto p_{t|0}(x_t \cond x_0) \, \pi(x_0)$ is the target, and $\grad\log p_{t|0}(x_t \cond x_0)$ is the test function. 
The condition in Assumption~\ref{assumption-score-approx} is thus akin to non-asymptotic variance bounds of importance sampling, and this has been well established by, for example,~\citet{Agapiou2017} and~\citet[][Thm. 8.5]{ChopinBook2020}. 
Typically, the Monte Carlo order is $r = 1\,/\,2$. 
Since the ensemble score may be unbounded as $t\to0$, we allow $C_e(t)$ to depend on $t$ without further imposing a specific decay rate. 
Similar assumptions have been used in~\citet[][A3]{DeBortoli2022convergence} and~\citet{DeBortoli2025convergence}. 
We have the following results. 

\begin{proposition}
    \label{proposition:main}
    Under Assumptions~\ref{assumption:diffusion} and~\ref{assumption-score-approx} we have 
    \begin{equation*}
        \begin{split}
            \mathsf{W}_2^2(\widetilde{q}_t, q_t) &\leq \mathsf{W}_2^2(p_T, \refmeasure) \, \expp^{b^2 \,(C_{\mathrm{ref}} - 2\, C_p) \, t}  \\
            &\quad+ 2 \, b^2 \, N^{-r} \overline{C}_e(t, T),
        \end{split}
    \end{equation*}
    for all $t\in[0, T)$, where $\overline{C}_e(t, T)$ is defined in Appendix~\ref{app:proof-main}.
\end{proposition}
\begin{proof}
    See Appendix~\ref{app:proof-main}.
\end{proof}
Proposition~\ref{proposition:main} quantifies how the two errors due to score approximation and $p_T \approx \refmeasure$ contribute to the resampling error. 
For a fixed $t$, the score error diminishes as $N\to\infty$ at the same rate $r$ but the resampling error has an irreducible bias due to $\mathsf{W}_2(p_T, \refmeasure)$.
Conversely, if we fix $N$ while increasing $t$, the error bound increases exponentially in $t$, since the SDE may accumulate the score error over time. 
This suggests that $N$ should increase fast enough as a function of $t$ to compensate the two errors. 
A more specific result is thus obtained as follows.

\begin{corollary}
    \label{corollary:asymp}
    Choose $N^{r-c} = \overline{C}_e(t, T)$ for any constant $0<c<r$, then 
    \begin{equation*}
        \begin{split}
            \mathsf{W}_2^2(\widetilde{q}_t, q_t) &\leq 2 \, b^2 \, N^{-c} \\
            &\quad+ \expp^{b^2 \,(C_{\mathrm{ref}} - 2\, C_p) \, t - 2 \, b^2 \, C_\mathrm{ref}^- \, T} \, \mathsf{W}_2^2(\pi, \refmeasure). 
        \end{split}
    \end{equation*}
    Furthermore, there exists a linear choice $t\mapsto T(t)$ such that $\lim_{t\to\infty} \mathsf{W}_2(\widetilde{q}_t, q_t) = 0$.
\end{corollary}
\begin{proof}
    See Appendix~\ref{app:proof-corollary}.
\end{proof}
Corollary~\ref{corollary:asymp} shows that the diffusion resampling converges as $t\to\infty$, noting that $N$ is a function of $t$.  
The required assumptions are rather mild, and notably, we did not assume any specific decaying rate on $C_e$. 
The result also suggests that we should choose the reference $\refmeasure$ close to $\pi$, and that the spectral gap of $\pi$ is not too large.
There is likely an optimal $b$ (Lambert function) that minimises the error bound but its explicit value may be hard to know prior to running the algorithm. 
One may also note that the factor $N^{-c}$ seems suboptimal, as it becomes slower than the importance sampling rate $N^{-r}$. 
The factor can be further tightened, as explained in Appendix~\ref{app:proof-corollary}.
\begin{remark}
	\label{remark:gaussian}
    When choosing the reference $\refmeasure$ to be a Gaussian, $t\mapsto\overline{C}_e(t, T(t))$ can grow no faster than a polynomial in any of the two arguments, since here $t\mapsto C_e(t)$ decays exponentially. 
    See also~\citet[][Remark 3]{DeBortoli2025convergence}. 
\end{remark}
Remark~\ref{remark:gaussian} show that the required sample size $N$ needs to grow only \emph{polynomially} in $T$ (or equivalently $t$). 
This result improves upon that of~\citet{Corenflos2021} whose sample size scales \emph{exponentially} in their entropy regularisation term $1 \, / \, \varepsilon$, and hence potentially implying cheaper computation. 

A natural follow-up question is whether we can improve the error bound by eliminating the error due to $p_T \approx \refmeasure$. 
This may be achieved by using the technique in~\citet{Corenflos2024FBS, Zhao2024rsta}, where a forward-backward Gibbs chain $(p_{T | 0}, q_{T|0})$ is constructed to avoid explicit sampling from the marginal $p_T$. 
This essentially transforms the error $p_T \approx \refmeasure$ into a statistical correlation of the chain. 
However, since Markov chains typically converge geometrically as well, the required number of chain steps is likely comparable to the diffusion time $T$~\citep{Meyn2009, Andrieu2024}. Furthermore, in practice we only have access to the approximate conditional $\widetilde{q}_{T|0}$. 
Nonetheless, we may still benefit computationally from the fact that a fixed $T$ can allow for moderate discretisation. 

\section{Experiments}
\label{sec:experiments}
In this section we empirically validate our method in both synthetic and real settings. 
The goal is fourfold. 
First, we show that the diffusion resampling, regardless of the differentiability, is a useful resampling method and is at least as good as commonly used resampling methods (Section~\ref{sec:gm-importance-resampling}). 
Second, we perform ablation experiments to test if the diffusion resampling provides a better estimate of the log marginal likelihood function (Section~\ref{sec:exp-lgssm}). 
Third and finally, we apply the diffusion resampling for learning neural network-parametrised SSMs using gradient-based optimisation in complex systems (Sections~\ref{sec:lokta} and~\ref{sec:real-experiment}).
Within all these experiments, we focus on comparing to optimal transport~\citep[OT,][]{Corenflos2021}, Gumbel-Softmax~\citep{Jang2017categorical}, and Soft resampling~\citep{Karkus2018}. 
Our implementations are publicly available at \url{https://github.com/zgbkdlm/diffres}.

\subsection{Gaussian mixture importance resampling}
\label{sec:gm-importance-resampling}
Consider a target distribution $\pi(x) \propto \phi(x) \, p(y \cond x)$ with Gaussian mixture prior $\phi(x) = \sum_{i=1}^c \omega_i \, \mathrm{N}(x; m_i, v_i)$ and likelihood $p(y \cond x) = \mathrm{N}(y; H \, x, \Xi)$. 
The posterior distribution $\pi$ is also a Gaussian mixture available in closed form (see Appendix~\ref{app:exp-gm}).
We use the prior as the proposal to generate importance samples, and then apply the resampling methods to obtain resampled particles.
These are compared to samples directly drawn from $\pi$, using the sliced 1-Wasserstein distance (SWD) and resampling variance for the mean estimator. 
We run experiments 100 times independently with sample size $N=10,000$.

\begin{table}[t!]
    \caption{Sliced 1-Wasserstein distance (SWD, scaled by $10^{-1}$) and resampling variance for the mean estimator (scaled by $10^{-2}$), for the Gaussian mixture resampling experiment in Section~\ref{sec:gm-importance-resampling}.}
    \label{tbl:gm-compare}
    \centering
    \wrapbox{.99}{%
        \begin{tabular}{@{}lll@{}}
        \toprule
        Method                 & SWD    & Resampling variance \\ \midrule
        Diffusion ($T=1$, $K=8$)   & $1.64 \pm 0.35$          & $6.87 \pm 5.87$                        \\
        Diffusion ($T=3$, $K=128$)    & $\mathbf{0.80} \pm 0.21$ & $3.74 \pm 2.99$                        \\
        OT ($\varepsilon=0.3$) & $0.84 \pm 0.22$          & $3.42 \pm 3.26$                        \\
        OT ($\varepsilon=0.6$) & $0.97\pm 0.21$           & $\mathbf{3.41} \pm 3.29$                        \\
        OT ($\varepsilon=0.9$) & $1.14 \pm 0.20$          & $\mathbf{3.41} \pm 3.29$                        \\
        Multinomial            & $0.82 \pm 0.25$          & $3.78 \pm 4.43$                        \\
        Gumbel-Softmax ($0.1$) & $1.40 \pm 0.24$          & $3.92 \pm 3.74$                        \\
        Soft (0.9)             & $0.83 \pm 0.24$          & $3.75 \pm 3.77$                      \\ \bottomrule
        \end{tabular}
    }
\end{table}

Table~\ref{tbl:gm-compare} shows that the diffusion resampling at $K=128$ gives the best SWD, followed by the baseline multinomial resampling and OT ($\varepsilon=0.3$). 
We also observe that the diffusion resampling performance highly depends on the choice of integrator and discretisation, which at $K=8$ is not superior than OT or multinomial. 
In terms of resampling variance for estimating the posterior mean, diffusion resampling is not as performant as OT but is still better than the baseline multinomial. 
OT seems to have stable resampling variance robust in $\varepsilon$. 
Gumbel-Softmax and the soft resampling methods are not comparable to diffusion or OT.
Further results are given in Appendix~\ref{app:exp-gm}.

\begin{figure}[t!]
    \centering
    \includegraphics[width=.99\linewidth]{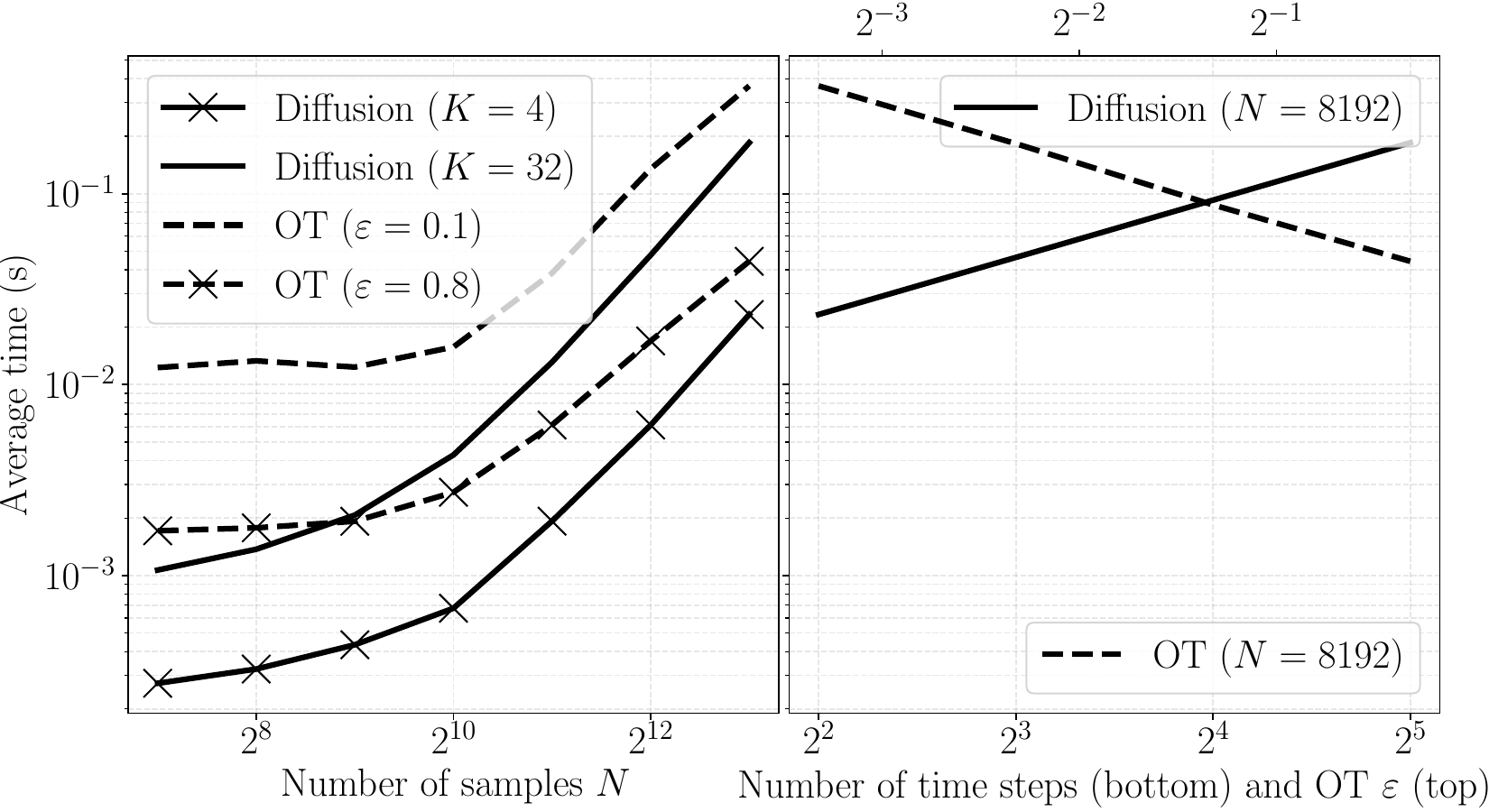}
    \caption{Average running times of diffusion resampling and OT.}
    \label{fig:time}
\end{figure}

The time costs of diffusion resampling and OT are shown in Figure~\ref{fig:time}, and we have two important observations. 
The left figure shows that both methods roughly scale polynomially in the sample size $N$, where diffusion $(K=4)$ is better than OT $(\varepsilon=0.8)$. 
On the right side of the figure, both methods scale linearly in their parameters $K$ and $1 \, / \, \varepsilon$, which aligns with the summary in Table~\ref{tbl:app-compare-all}, and the cross point shows that they have a similar computation at $K \approx 6 \, / \, \varepsilon$ when $N=8,192$. 
We also observe a trend that when $N$ increases, the cross point moves left, showing that diffusion scales better than OT in $N$. 
Details are given in Appendix~\ref{app:time}. 

\subsection{Linear Gaussian SSM}
\label{sec:exp-lgssm}
We now consider particle filtering and compare different resampling methods on the model
\begin{equation}
    \begin{split}
        Z_j \cond Z_{j-1} &\sim \mathrm{N}(z_j ; \theta_1 \, z_{j-1}, I_d), \quad Z_0 \sim \mathrm{N}(0, I_d), \\
        Y_j \cond Z_j &\sim \mathrm{N}(y_j ; \theta_2 \, z_j, 0.5 \, I_d)
    \end{split}
\end{equation}
with parameters $\theta_1=0.5$ and $\theta_2=1$. 
For each experiment, we generate a measurement sequence with 128 time steps and apply a Kalman filter to compute the true filtering solution and log-likelihood $L(\theta_1, \theta_2)$ evaluated at Cartesian $\Theta = [\theta_1 - 0.1, \theta_1 + 0.1]\times [\theta_2 - 0.1, \theta_2 + 0.1]$.
Three errors are evaluated: 1) the error of log-likelihood function estimate $\norm{L - \widehat{L}}^2_2 \coloneqq \int_\Theta (L(\theta_1, \theta_2) - \widehat{L}(\theta_1, \theta_2))^2 \diff \theta_1\diff\theta_2$, where $\widehat{L}$ stands for the estimate by a particle filter with resampling; 2) the KL divergence between the particle samples and the true filtering solution; 3) the estimated parameters $\widehat{\theta}$ by L-BFGS compared to the truth under Euclidean norm $\norm{\theta - \widehat{\theta}}_2$. 
All results are averaged over 100 independent runs, with 32 particles used. 
For details see Appendix~\ref{app:lgssm}.

\begin{figure*}[t!]
    \centering
    \includegraphics[width=.99\linewidth]{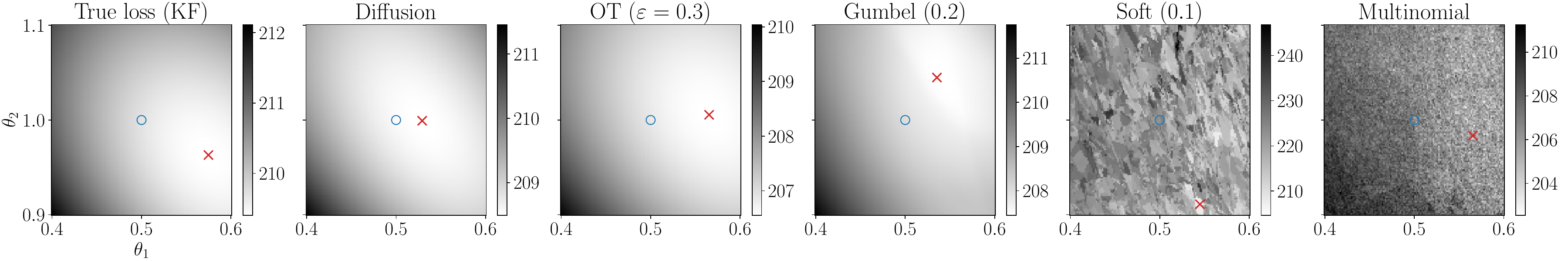}
    \caption{Visualisation of the particle filter estimated log-likelihoods using different resampling methods associated with the LGSSM experiment. 
    Blue circle \textcolor{tabblue}{$\circ$} stands for the true parameter, while the red cross \textcolor{tabred}{$\times$} represents the minimum of the estimated loss function. 
    We see in this example that the diffusion resampling gives an estimate closest to both the true parameter and the true loss minimum. }
    \label{fig:lgssm-loss}
\end{figure*}

\begin{table}[t!]
    \caption{The errors of loss function, filtering (scaled by $10^{-1}$), and parameter estimation (scaled by $10^{-1}$) for Section~\ref{sec:exp-lgssm}. 
    Divergent NaN results are explained in Appendix~\ref{app:lgssm}.}
    \label{tbl:lgssm-compare}
    \centering
    \wrapbox{.99}{%
        \begin{tabular}{@{}llll@{}}
        \toprule
        Method                   & $\norm{L - \widehat{L}}_2$ & Filtering KL          & $\norm{\theta - \widehat{\theta}}_2$ \\ \midrule
        Diffusion ($T=1$, $K=4$) & $2.61 \pm 2.08$          & $4.94 \pm 6.92$          & $\mathbf{1.28} \pm 0.70$           \\
        Diffusion ($T=2$, $K=8$) & $2.61 \pm 1.89$          & $4.40 \pm 4.78$          & $1.29 \pm 0.78$                    \\
        Diffusion ($T=3$, $K=8$) & $\mathbf{2.55} \pm 1.89$ & $\mathbf{4.26} \pm 4.49$ & $1.58 \pm 0.75$                    \\
        OT ($\varepsilon=0.4$)   & $2.64 \pm 2.13$          & $5.07 \pm 6.21$          & $1.53 \pm 1.16$                    \\
        OT ($\varepsilon=0.8$)   & $2.68 \pm 2.16$          & $5.07 \pm 5.70$          & $1.58 \pm 1.22$                    \\
        OT ($\varepsilon=1.6$)   & $2.75 \pm 2.20$          & $5.11 \pm 5.17$          & $1.49 \pm 0.97$                    \\
        Multinomial              & $2.80 \pm 1.84$          & $5.49 \pm 6.87$          & NaN                                \\
        Gumbel-Softmax (0.1)     & $2.79 \pm 2.14$          & $4.83 \pm 5.76$          & NaN                                \\
        Soft (0.9)               & $2.85 \pm 1.80$          & $4.66 \pm 5.68$          & NaN                                \\ \bottomrule
        \end{tabular}
    }
\end{table}

The results in Table~\ref{tbl:lgssm-compare} clearly show that the diffusion approach significantly outperforms other resampling methods across all the three metrics with less variance. 
Importantly, even without considering the differentiability, the diffusion resampling already gives the best filtering estimate, showing that it is a useful resampling method in itself. 
This is likely due to the fact that in diffusion resampling, the reference is given by the current posterior samples instead of the predictive samples like OT.
In contrast to the Gaussian mixture experiment where the diffusion needs larger $K$ to be comparable to OT, the needs for fine discretisation is moderate for this model. 
Also based on Figure~\ref{fig:time}, the diffusion resampling achieves substantially better results than OT when evaluated at the same level of computational cost.

We also observe that almost all the resampling methods encounter divergent parameter estimations to some extent (see more statistics in Appendix~\ref{app:lgssm}). 
In particular, the divergence of Soft and Gumbel-Softmax is too significant to give a meaningful result, and these entries are thus marked as NaN. 
The reason is due to the quasi-Newton optimiser L-BFGS-B which heavily depends on stable and accurate gradient estimates~\citep{Xie2020}. 
This shows that the diffusion and OT approaches are more applicable for generic optimisers, such as second-order ones. 
Therefore, in the next sections we will focus on comparison using first-order gradient-based methods.

\subsection{Prey-predator model}
\label{sec:lokta}
Consider a more challenging Lokta--Voltera model 
\begin{equation}
    \begin{split}
        \diff C(t) &= C(t) \, (\alpha - \beta \, R(t)) \diff t + \sigma \, C(t) \diff W_{1}(t),\\
        \diff R(t) &= R(t) \, (\zeta \, C(t) - \gamma) \diff t + \sigma \, R(t) \diff W_{2}(t), \\
        Y_j &\sim \mathrm{Poisson}\bigl(\lambda(C(t_j), R(t_j))\bigr),
    \end{split}
\end{equation}
where the configuration is detailed in Appendix~\ref{app:prey-predator}.
We generate data at time $t\in[0, 3]$ discretised by Milstein's method with 256 steps. 
We model the dynamics entirely by a neural network, without assuming known the dynamic structure. 
After the neural network has been learnt, we make 100 predictions from the learnt model and then compare them to a reference trajectory generated by the true dynamics. 
All results are averaged over 20 independent runs, with $N=64$ particles used. 
We have also compared to the REINFORCE framework implemented with a stopped gradient method by~\citet{Scibior2021differentiable}.

\begin{figure}[t!]
    \centering
    \includegraphics[width=.99\linewidth]{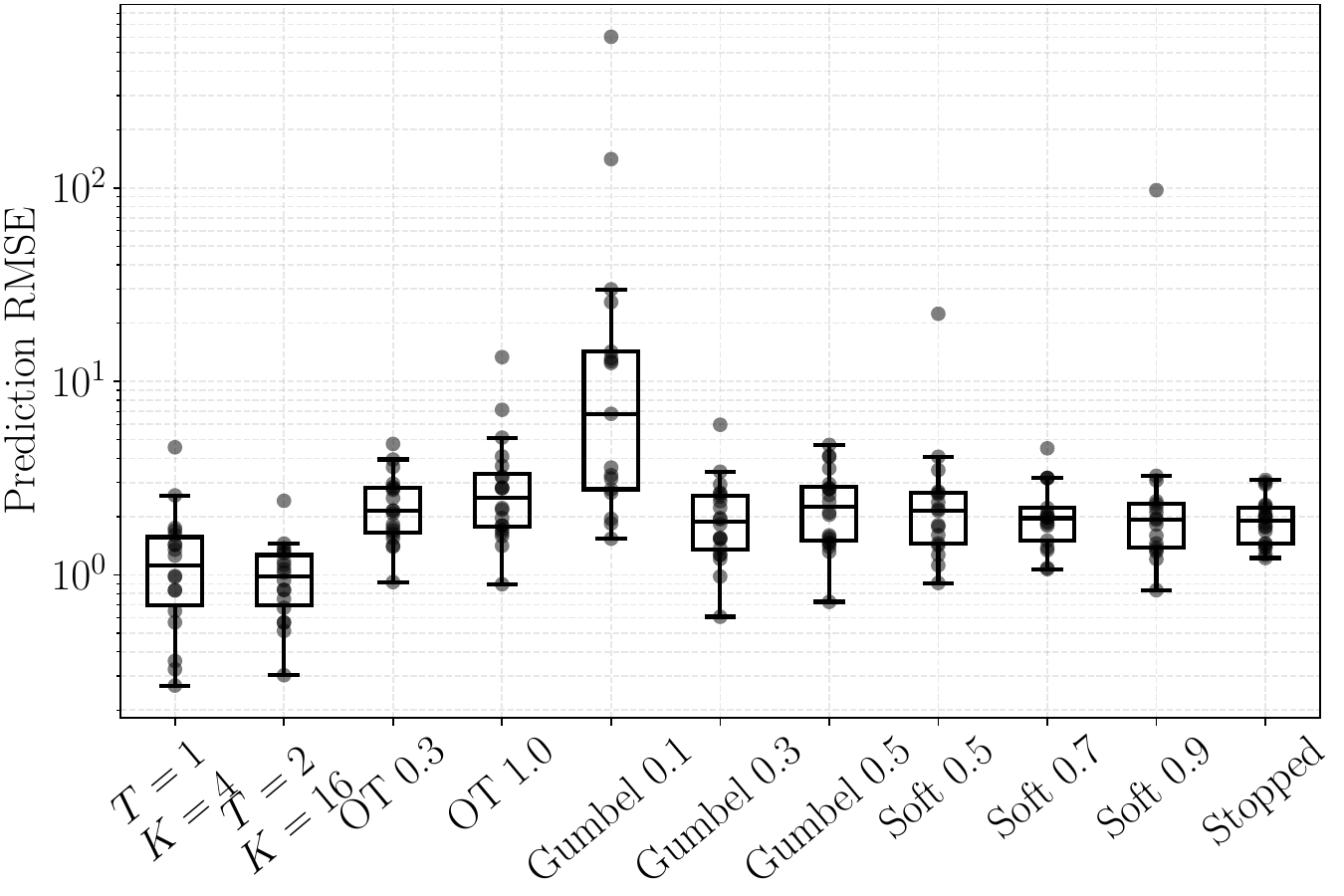}
    \caption{Root mean square errors of the learnt prey-predator models over independent runs (scatter points).}
    \label{fig:lokta-rmse}
\end{figure}

\begin{figure}[t!]
    \centering
    \includegraphics[width=.9\linewidth]{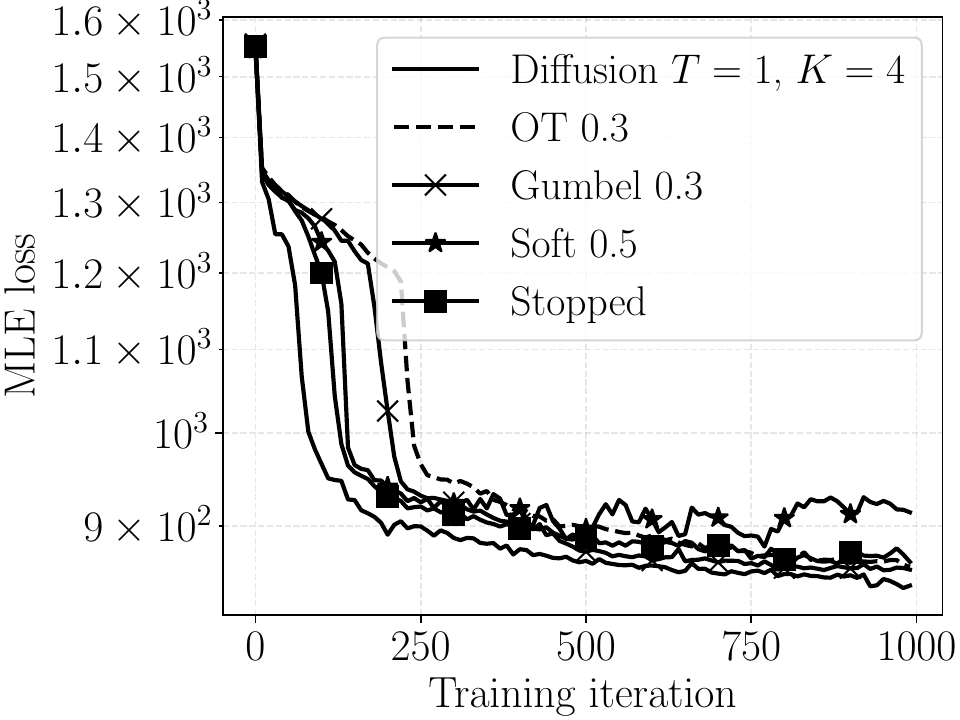}
    \caption{Loss traces (median over all runs) for training the prey-predator model. 
    The training enabled by the diffusion resampling achieves the lowest and stablest. }
    \label{fig:lokta-loss}
\end{figure}

Results are summarised in Figures~\ref{fig:lokta-rmse} and~\ref{fig:lokta-loss}. 
The first figure shows that the trained model using the diffusion resampling achieves the lowest prediction error and is significantly more stable than the baselines. 
The second figure aligns with the results from the first figure, showing that the training loss with diffusion resampling consistently lower bounds that of the other methods, and is more stable. 

\subsection{Vision-based pendulum dynamics tracking}
\label{sec:real-experiment}
Next, we consider the problem of learning the dynamics of a physical system from high-dimensional and structural image observations. 
We simulate pendulum dynamics described by a discrete-time SSM 
\begin{equation}
    \begin{split}
    Z^{(1)}_j &= Z^{(1)}_{j-1} + Z^{(2)}_{j-1} \, \Delta_j + \zeta_j^{(1)}, \\
    Z^{(2)}_j &= Z^{(2)}_{j-1} - \frac{g}{l} \sin\bigl(Z^{(1)}_{j-1} \bigr) \, \Delta_j + \zeta_j^{(2)}, \\
    Y_j \cond Z_j &\sim \mathrm{N}\bigl(r(Z_j), \sigma_{\text{obs}}^2 \, I_{32\times 32}\bigr),
    \end{split}
    \label{eq:pendulum_dynamics}
\end{equation}
where $Z_j \coloneqq \begin{bmatrix} Z^{(1)}_j & Z^{(2)}_j \end{bmatrix}$ represents the state defined by the angle and angular velocity, $\Delta_j$ is the discretisation step, and $\zeta_j \sim \mathcal{N}(0, \Delta_j \, \Lambda_{\zeta})$ is the process noise. 
The true observation function $r$ generates $32 \times 32$ snapshots of the pendulum. 
We generate a trajectory of observations $Y_{0:J}$ over $J=256$ steps. 
In contrast to Section~\ref{sec:lokta}, we learn both the underlying dynamics and the observation model, and we parametrise both the state transition, $Z_j = f_\theta(Z_{j-1}, \zeta_{j})$, and the decoder, $r \approx r_\phi$, by neural networks. 
Parameters $\theta$ and $\phi$ are learnt by minimising the negative log-likelihood estimated by the particle filter with resampling. 
For completeness, we run experiments in two regimes: the first uses adaptive resampling under a tempered observation likelihood, and the second is a more challenging setting with resampling at nearly every filtering step. 
The learnt models are evaluated using the average structural similarity index (SSIM) and peak
signal-to-noise ratio (PSNR) on their predicted image sequences. 

Our results demonstrate that diffusion resampling integrates effectively into this complex, high-dimensional SMC learning pipeline, enabling stable optimisation and achieving competitive performance relative to state-of-the-art differentiable resampling baselines.
Figure~\ref{fig:pendulum_grid} shows an accurate and high-fidelity mean image sequence reconstructed from a latent pendulum dynamics model and decoder learnt with diffusion resampling embedded in the optimisation process. 
The comparative results in Figure~\ref{fig:pendulum_SSIM_PSNR} show that, among the strongest configurations of each resampling class, diffusion resampling achieves mean SSIM and PSNR at least comparable to the baselines.  
See Appendix~\ref{app:pendulum} for further details and results. We also include additional large-scale experiments in Appendix~\ref{app:pbnn} and Appendix~\ref{app:weather}.
\begin{figure}[t!]
    \centering
    \includegraphics[width=.9\linewidth]{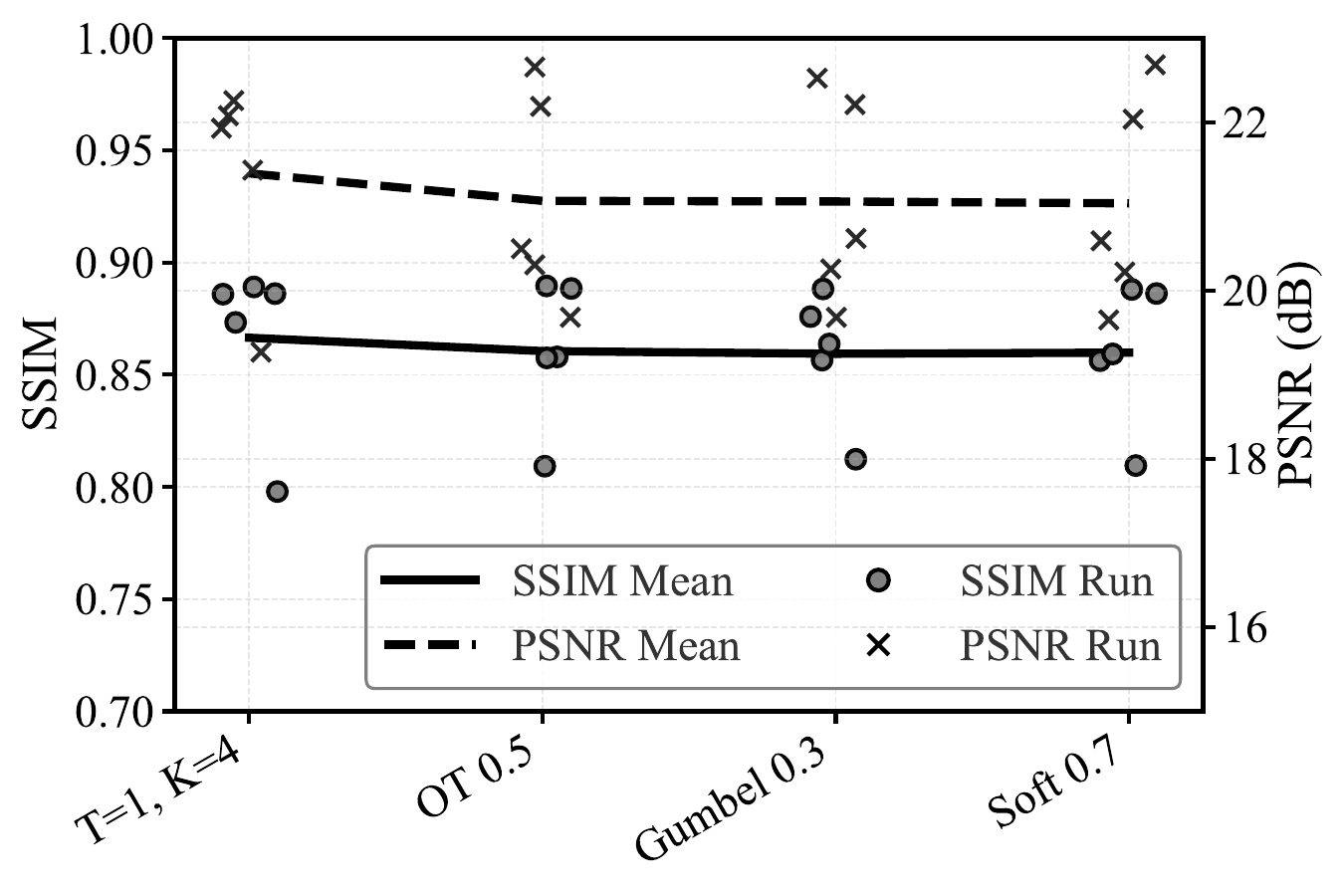}
    \caption{Mean prediction SSIM and PSNR for the best (by mean) model configuration of each resampler in the first experiment setting. 
    Individual runs are shown as scatter points.}
    \label{fig:pendulum_SSIM_PSNR}
\end{figure}

\begin{figure}[t!]
    \centering
    \includegraphics[width=.9\linewidth]{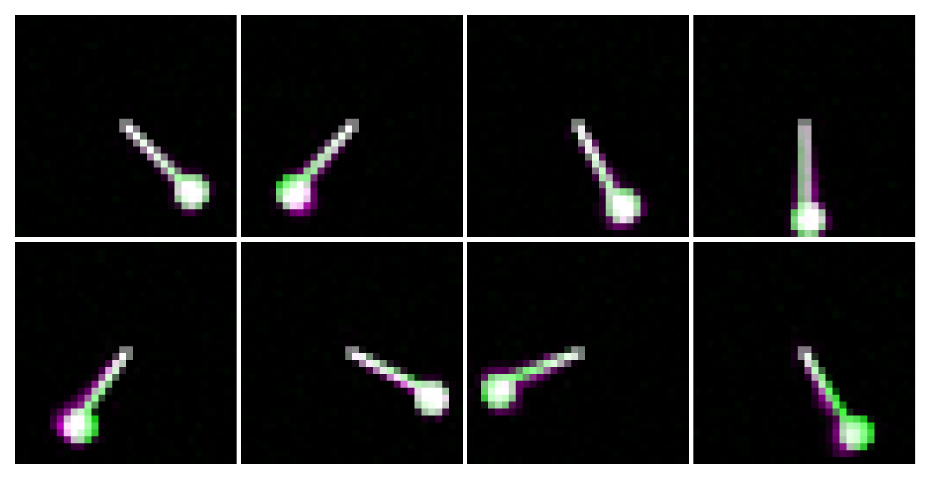}
    \caption{Qualitative comparison of the learnt pendulum dynamics. 
    The ground truth (green) is overlaid with model predictions (purple). White pixels indicate perfect alignment, while coloured regions highlight positional discrepancies (e.g., phase lag). 
    Snapshots are shown at eight evenly spaced time points over the 4 second simulation (read from left to right and top to bottom).
    }
    \label{fig:pendulum_grid}
\end{figure}

\section{Related work}
\label{sec:related-work}
In concurrent work, \citet{Gourevitch2026categorical} propose a differentiable reparameterisation of categorical distributions using stochastic interpolants. 
While our diffusion resampler similarly constitutes a differentiable relaxation of categorical sampling, the key departure is that their closed-form denoiser is derived under one-hot encoded samples $\lbrace X_i \rbrace_{i=1}^N$, whereas we focus on samples in $\R^d$. 
We further consider the case when the empirical measure approximates an underlying continuous distribution as $N\to\infty$, while \citet{Gourevitch2026categorical} mostly focus on discrete categorical distributions. 
There exists no exact reparametrisation producing the true expectation gradient \emph{under the categorical measure}, however, a consistent reparametrisation (e.g., ours) exists converging to the true expectation gradient \emph{under the underlying continuous measure}. 

Another work closely related to ours is that by~\citet{Wan2025diffpf}, who also leverages a diffusion model for differentiable resampling within the SMC framework. 
While empirically powerful, their approach relies on a \emph{trained} conditional
diffusion model, which introduces bias, breaks consistency guarantees, and adds substantial computational cost. 
In addition, their method introduces a further challenge in that the resampling gradient should also be propagated through the diffusion training. 
In contrast, our method focuses on a statistically grounded diffusion resampling scheme without training, and is more informative. 

\section{Conclusion}
\label{sec:conclusion}
In this paper we have proposed a new differentiable-by-construction, computationally efficient, and informative resampling method built on diffusion models. 
We have explicitly quantified a convergent error bound in the sample size and diffusion parameters. 
Our experiments verify that the proposed method outperforms the state-of-the-art differentiable resampling methods for filtering and parameter estimation of state-space models. 

\textbf{Limitations and future work}. 
We observe that propagating gradients through the diffusion resampling can be numerically sensitive to the SDE solver. 
This could potentially be mitigated by, e.g., reversible adjoint Heun's method~\citep{Kidger2021gradient} or related variants. 
We also note that the ensemble score may not be the only choice to achieve diffusion resampling.

\section*{Acknowledgements}
This work was partially supported by the Wallenberg AI, Autonomous Systems and Software Program (WASP) funded by the Knut and Alice Wallenberg (KAW) Foundation. 
Computations were enabled by the supercomputing resource Berzelius provided by National Supercomputer Centre at Link\"{o}ping University and the KAW foundation.
We also thank the Stiftelsen G.S Magnusons fond (MG2024-0035). 

Authors' contributions are as follows. 
JA verified the theoretical results, wrote Related work, and conducted the pendulum and weather forecasting experiments. 
ZZ came up with the idea and developed the method, wrote the initial draft, proved the theoretical results, and performed the other experiments. 
All authors contributed to and revised the final manuscript. 

\section*{Impact statement}
The work is concerned with a fundamental problem within statistics and machine learning. 
It does not directly lead to any ethical and societal concerns needed to be explicitly discussed here.

\bibliography{refs}
\bibliographystyle{icml2026}

\newpage
\appendix
\onecolumn

\section{Diffusion resampling with Gaussian reference}
\label{app:diffres-gaussian}
For easy reproducibility, we present the diffusion resampling specifically with the mean-reverting Gaussian reference $\refmeasure$ in the following algorithm. 
In practice, the algorithm is always implemented for log weights. 

\SetKwProg{Def}{def}{:}{}
\begin{algorithm2e}[]
    \SetAlgoLined
    \DontPrintSemicolon
    \KwInputs{Weighted samples $\lbrace (w_i, X_i) \rbrace_{i=1}^n$ and time grid $0=t_0 < t_1 < \cdots < t_K=T$. }
    \KwOutputs{Differentiable re-samples $\lbrace (\frac{1}{N}, X^{*}_i) \rbrace_{i=1}^n$}
    $\mu_N \gets \sum_{i=1}^N w_i \, X_i$\;
    $\Sigma_N \gets \sum_{i=1}^N w_i \, (X_i - \mu_N) \, (X_i - \mu_N)^\trans$\tcp*{Or other Gaussian approximation methods}
    \Def{$s(x, t)$}{%
        $\alpha_i(x, t) \gets w_i \, \mathrm{N}(x; m_t(X_i), V_t)$\tcp*{parallel for $i=1,2,\ldots, N$}
        $\alpha_i(x, t) \gets \alpha_i(x, t) \, / \, \sum_{j=1}^N \alpha_j(x, t)$ \;
        return 
        \begin{equation*}
            -\sum_{i=1}^N \alpha_i(x, t) \, V_t^{-1} \, \bigl(x - m_t(X_i) \bigr)
        \end{equation*}
    }
    \For(\tcp*[f]{parallel for $i=1,2,\ldots, N$}){$i=1,2,\ldots, N$}{%
        $U_{i, 0} \sim \mathrm{N}(\mu_N, \Sigma_N)$\;
        \For{$k=1, 2, \ldots, K$}{%
            $\Delta_k = t_k - t_{k-1}$\;
			Draw $\xi_k^i \sim \mathrm{N}(0, 2 \,  b^2 \, \Delta_k \, I_d)$\;
            $U_{i, t_{k}} = U_{i, t_{k-1}} + b^2 \bigl[\Sigma_N^{-1} \, (U_{i, t_{k-1}} - \mu_N) + 2 \, s(U_{i, t_{k-1}}, T - t_{k-1}) \bigr]\, \Delta_k + \xi_k^i$\;
            \tcp{Or any other SDE solver for Equation~\eqref{equ:rev}}
        }
        $X^*_i \gets U_{i, T}$
    }
    \caption{Diffusion resampling \texttt{diffres} with Gaussian reference}
    \label{app:alg-diffres-normal}
\end{algorithm2e}

\section{Proof of Proposition~\ref{proposition:main}}
\label{app:proof-main}
For any $t\in[0, T)$, define a residual $h_t \coloneqq U(t) - \widetilde{U}(t)$, and recall that they both are driven by the same Brownian motion. 
We have 
\begin{equation}
	\begin{split}
		\abs{h_t} &\leq \abs{U(0) - \widetilde{U}(0)} \\
		&\quad + b^2 \biggl( \int^t_0 \absbig{ \grad\log \refmeasure(\widetilde{U}(\tau)) - \grad\log\refmeasure(U(\tau))}\diff \tau + 2\int^t_0 \absbig{ \grad\log p_{T - \tau}(U(\tau)) - s_N(\widetilde{U}(\tau), T - \tau)} \diff \tau \biggr) \\
	\end{split}
\end{equation}
By It\^{o}'s formula we have 
\begin{equation}
	\abs{h_t}^2 = \abs{h_0}^2 + 2 \int^t_0 \innerp{h_\tau, h'_\tau} \diff \tau,
\end{equation}
where $h'_\tau$ stands for the drift of $h_\tau$. 
We obtain
\begin{equation*}
	\begin{split}
		\abs{h_t}^2 &= \abs{h_0}^2 + 2 \, b^2 \int^t_0 \innerpbig{h_\tau, \grad\log \refmeasure(\widetilde{U}(\tau)) - \grad\log\refmeasure(U(\tau)) + 2 \, \grad\log p_{T - \tau}(U(\tau)) - 2 \, s_N(\widetilde{U}(\tau), T - \tau)} \diff \tau \\
		&= \abs{h_0}^2 + 2 \, b^2 \int^t_0 \innerpbig{h_\tau, \grad\log \refmeasure(\widetilde{U}(\tau)) - \grad\log\refmeasure(U(\tau))} \diff \tau \\
		&\qquad\qquad+ 4 \, b^2 \int^t_0 \innerpbig{h_\tau, \grad\log p_{T - \tau}(U(\tau)) - \grad\log p_{T - \tau}(\widetilde{U}(\tau))} + \innerpbig{h_\tau, \grad\log p_{T - \tau}(\widetilde{U}(\tau)) - s_N(\widetilde{U}(\tau), T - \tau)} \diff \tau.
	\end{split}
\end{equation*}
Therefore, applying Assumptions~\ref{assumption:diffusion} and~\ref{assumption-score-approx}, and H\"{o}lder's inequality we get
\begin{equation*}
	\begin{split}
		\expecbig{\abs{h_t}^2} &\leq \expecbig{\abs{h_0}^2} + 2\,b^2 \,(C_\mathrm{ref} - 2\, C_p)\int^t_0 \expecbig{\abs{h_\tau}^2} \diff \tau + 4 \, b^2 \int^t_0 \expecbig{\abs{h_\tau}^2}^{\frac{1}{2}} \, \frac{C_e(T - \tau)}{N^r} \diff \tau,
	\end{split}
\end{equation*}
where recall that the weighted samples and the SDE system are mutually independent.
Applying (non-linear) Gr\"{o}nwall inequality~\citep[][Theorem 21]{Dragomir2003} we finally obtain
\begin{equation}
	\expecbig{\abs{h_t}^2} \leq \expecbig{\abs{h_0}^2} \, \expp^{b^2 \, (C_{\mathrm{ref}}-2 \, C_p) \, t} + 2 \, b^2 \, N^{-r} \int^t_0 C_e(T -\tau) \, \expp^{b^2 \, (C_{\mathrm{ref}}-2 \, C_p) \, (t - \tau)} \diff \tau.
\end{equation}
This directly implies an upper bound for the 2-Wasserstein distance between the measures of $U(t) \sim q_t$ and $\widetilde{U}(t) \sim \widetilde{q}_t$:
\begin{equation}
	\mathsf{W}_2^2(\widetilde{q}_t, q_t) \leq \expp^{b^2 \,(C_{\mathrm{ref}} - 2\, C_p) \, t} \mathsf{W}_2^2(p_T, \refmeasure) + 2 \, b^2 \, N^{-r} \overline{C}_e(t, T),
	\label{equ:appendix-bound}
\end{equation}
where $\overline{C}_e(t, T) \coloneqq \int^t_0 C_e(T -\tau) \exp[b^2 \, (C_{\mathrm{ref}}-2 \, C_p) \, (t - \tau)] \diff \tau$.

\section{Proof of Corollary~\ref{corollary:asymp}}
\label{app:proof-corollary}
Recall the forward (Langevin) equation
\begin{equation}
	\begin{split}
		\diff X(t) &= b^2\grad \log \refmeasure(X(t)) \diff t + \sqrt{2} \, b \diff W(t), \\
		X(0) &\sim \pi.
	\end{split}
\end{equation}
The error $\mathsf{W}_1(p_t, \refmeasure) \to 0$ geometrically fast as $t\to\infty$, more specifically, 
\begin{equation}
	\mathsf{W}_2^2(p_T, \refmeasure) \leq \expp^{-2 \,b^2 \, C_\mathrm{ref}^- \, T} \, \mathsf{W}_2^2(\pi, \refmeasure), 
	\label{equ:appendix-langevin-conv}
\end{equation}
where $C_{\mathrm{ref}}^-$ is the concave constant in Assumption~\ref{assumption:diffusion}. 
This is a classical result, see, for instance, \citet{Ambrosio2008} and~\citet{VonRenesse2005}. 
Now choose $N^{r-c} = \overline{C}_e(t, T)$ for any constant $c$ such that $0<c<r$, and substitute Equation~\eqref{equ:appendix-langevin-conv} into Equation~\eqref{equ:appendix-bound} we get 
\begin{equation}
	\mathsf{W}_2^2(\widetilde{q}_t, q_t) \leq 2 \, b^2 \, N^{-c} + \expp^{b^2 \,(C_{\mathrm{ref}} - 2\, C_p) \, t - 2\,b^2 \, C_\mathrm{ref}^- \, T} \, \mathsf{W}_2^2(\pi, \refmeasure). 
\end{equation}
Recall that $t\mapsto C_e(t)$ is positive non-increasing, therefore $t\mapsto \overline{C}_e(t, T)$ is increasing for any $T>0$. 
The aim here is to show there exists a sequence $t\mapsto T(t)$ such that $\lim_{t\to \infty}\mathsf{W}_1(\widetilde{q}_t, q_t) = 0$ \emph{without further imposing conditions on $C_e$}. 
Clearly, the limit holds when 1) $T(t) > t$ uniformly and 2) $t \mapsto N(t)$ is increasing. 
To show 2), we take derivative of $t \mapsto \overline{C}_e(t, T(t))$ obtaining
\begin{equation*}
	\overline{C}_e'(t, T(t)) = C_e(T(t) - t) + \int^t_0 T'(t) \, C_e'(T(t)-\tau) \, \expp^{z \, (t-\tau)} + C_e(T(t)-\tau) \, \expp^{z \, (t - \tau)} \, z \diff \tau,
\end{equation*}
where we shorthand $z = b^2 \, (C_\mathrm{ref} - 2 \, C_p)$ for simplicity. 
Using integration by parts on $\int^t_0 C_e'(T(t)-\tau) \, \expp^{z \, (t-\tau)} \diff \tau$ we get
\begin{equation}
	\int^t_0 C'_e(T(t) - \tau) \, \expp^{z \, (t-\tau)} \diff \tau = -C_e(T(t) - t) + C_e(T(t)) \, \expp^{z \, t} - \int^t_0 C_e(T(t)-\tau) \, \expp^{z \, (t - \tau)} \, z \diff \tau. 
\end{equation}
Therefore, 
\begin{equation}
	\overline{C}_e'(t, T(t)) = (1 - T'(t)) \, \biggl( C_e(T(t) - t) + \int^t_0 C_e(T(t)-\tau) \, \expp^{z \, (t - \tau)} \, z \diff \tau \biggr) + T'(t) \, C_e(T(t)) \, \expp^{z \, t}.
\end{equation}
Due to the positivity of $C_e$, a sufficient condition for the derivative above being positive is $T'(t) > 0$ and $T'(t) \leq 1$. 
The set of such sequence is not empty, and a trivial example is $T(t) = t + \delta$ for some parameter $\delta >0$, which satisfies 1) and is clearly independent of any of the system components. 
Moreover, the exponent $b^2 \,(C_{\mathrm{ref}} - 2\, C_p) \, t - 2 \, b^2 \, C_\mathrm{ref}^- \, T(t)$ is always negative under such an example. 

Note that the factor $N^{-c}$ is likely suboptimal since it can not exceed the importance sampling factor $r$. 
This is intuitive, because we are essentially asking $N$ to compensate the SDE accumulation error. 
It is possible to relax from this and choose $c \geq r$. 
This in turn means that we need more assumptions on $C_e$ to ensure that $\overline{C}_e(t, T(t))$ is \emph{non-increasing} in $t$, in contrast to what we aimed here for being increasing which is mild. 
However, such settings will make Assumption \ref{assumption-score-approx} more difficult to verify in reality. 
Since making mild and practically checkable assumptions on score approximation remains an open discussion~\citep[see, e.g.,][]{DeBortoli2022convergence, Chen2023sampling}, we opt for a suboptimal bound under assumptions that are simple and easy-to-verify (n.b., we have only assumed $C_e$ to be a positive non-increasing function without specifying any rate). 

The result essentially reveals how we can leverage information from $\refmeasure$ to achieve for lower resampling error. 
Commonly used resampling schemes assume that we only have access to the given samples $\lbrace (w_i, X_i) \rbrace_{i=1}^N$. 
On the contrary, diffusion resampling additionally assumes a sampler for $\refmeasure$ that approximates $\pi$ yielding more information of the target. 
The diffusion resampling effectively fuses the samples from $\lbrace (w_i, X_i) \rbrace_{i=1}^N$ and $\refmeasure$ to obtain the re-samples, with the constant $b$ indicating how much we rely on $\refmeasure$. 
With $b=0$, it means that we completely trust the samples from $\refmeasure$, and this is indeed optimal when $\refmeasure=\pi$. 
See Appendix~\ref{app:elaboration-doob} for more details.

\section{Elaboration of Remark~\ref{remark:doob}}
\label{app:elaboration-doob}
Recall the forward/noising process
\begin{equation}
    \diff X(t) = b^2\grad \log \refmeasure(X(t)) \diff t + \sqrt{2} \, b \diff W(t), \quad X(0) \sim \pi,
    \label{equ:app-fwd}
\end{equation}
which corresponds to the reversal
\begin{equation}
    \diff U(t) = b^2 \bigl[-\grad \log \refmeasure(U(t)) + 2\grad\log p_{T - t}(U(t)) \bigr]\diff t + \sqrt{2} \, b \diff W(t), \quad U(0) \sim p_T.
    \label{equ:app-rev-empirical}
\end{equation}
Also recall the associated $h$-function $h(x, t) = \sum_{i=1}^N w_i \, p_{t | 0}(x \cond X_i)$. 
Now if we instead initialise $X(0)$ at the empirical $\pi^N \coloneqq \sum_{i=1}^N w_i \, \delta_{X_i}$, the resulting process $t\mapsto X^N(t)$ becomes the correspondence of reversal
\begin{equation}
    \diff U^N(t) = b^2 \bigl[-\grad \log \refmeasure(U^N(t)) + 2 \grad \log h(U^N(t), T - t) \bigr]\diff t + \sqrt{2} \, b \diff W(t), \quad U^N(0) \sim p^N_T, 
\end{equation}
where $p_T^N(\cdot) = h(\cdot, T)$, and at time $t=T$ we recover $U^N(T) \sim \pi^N$. 
As such, we can see that this empirical forward-reversal pair is essentially a differentiable reparametrisation of multinomial resampling. 
However, we do not implement this reversal in practice, since sampling from the mixture $p_T^N$ too typically introduces discrete randomness. 
The actually implementable reversal is 
\begin{equation}
    \diff \widetilde{U}(t) = b^2 \bigl[-\grad \log \refmeasure(\widetilde{U}(t)) + 2 \grad \log h(\widetilde{U}(t), T - t) \bigr]\diff t + \sqrt{2} \, b \diff W(t),  \quad   \widetilde{U}(0) \sim \refmeasure. 
\end{equation}
At time $t=T$, we obtain $\widetilde{U}(T) \sim \sum_{i=1}^N \gamma_i \, \delta_{X_i}$, where $\gamma_i \neq w_i$ depending on $\refmeasure$. 
To arrive at the new weights, let us denote the transition distribution of $U^N$ by $\widetilde{q}^N_{t | s}$ which is the same for $\widetilde{U}$. 
Equations~\eqref{equ:app-fwd} and~\eqref{equ:app-rev-empirical} imply that $p_T^N(u_0) \, \widetilde{q}^N_{T|0}(u_T \cond u_0) = \pi^N(u_T)\, p_{T | 0}(u_0 \cond u_T)$. 
Therefore, the distribution of $\widetilde{U}(T)$ is 
\begin{equation}
    \begin{split}
        \widetilde{q}_T(u_T) = \int \refmeasure(u_0) \, \widetilde{q}^N_{T|0}(u_T \cond u_0) \diff u_0 &= \int \refmeasure(u_0) \, \widetilde{q}^N_{T|0}(u_T \cond u_0) \, \frac{p_T^N(u_0)}{p_T^N(u_0)} \diff u_0 \\
        &=\int \frac{\refmeasure(u_0) \, p_{T|0}(u_0 \cond u_T)}{p_T^N(u_0)} \, \pi^N(u_T) \diff u_0 = \sum_{i=1}^N \gamma_i(u_T) \, \delta_{X_i}(u_T), 
        \label{equ:resampling-dist}
    \end{split}
\end{equation}
where the weight $\gamma_i(u_T) = w_i \, \int \frac{\refmeasure(u_0) \, p_{T|0}(u_0 \cond u_T)}{p_T^N(u_0)} \diff u_0$ depends on the spatial location, and the integral is the Radon--Nikodym derivative between $\widetilde{q}_T$ and $\pi^N$. 
Clearly, we see that the diffusion resampling is essentially an importance sampling upon $\pi^N$ similar to soft resampling. 
However, the crucial difference is that the soft resampling's importance proposal is completely uninformative about the target while the diffusion resampling is (i.e., the proposal $\refmeasure$ and derivative $\gamma_i$ know information about the target).  
As a special case, if we choose $\refmeasure = p_T^N$ then $\gamma_i = w_i$, reducing to the target empirical measure. 

\section{Error analysis of the resampling mapping}
\label{app:error-resampling-only}
Previously in our main results (e.g., Corollary~\ref{corollary:asymp}), we have analysed the error between the resampled distribution and the underlying true continuous distribution. 
In this section, we forgo the continuous distribution, and purely analyse the resampling mapping, that is, the $L_2$ error between the input and output ensembles. 

Recall the input ensemble $\pi^N \coloneqq \sum_{i=1}^N w_i \, \delta_{X_i}$, the resampled ensemble $\pi^{*, N}$, and denote $\pi^N(\phi) \coloneqq \mathbb{E}_{\pi^N}[\phi] = \sum_{i=1}^N w_i \, \phi(X_i)$ for any test function $\phi$. 
Define the $\chi^2$ divergence by $\chi^2(p \, \Vert \, q) \coloneqq \int (p(x)-q(x) \, / \, q(x)^2) \, q(x) \diff x$. 
Based on Equation~\eqref{equ:resampling-dist} we have 
\begin{equation}
    \pi^{\star,N}(\phi)=\int\pi_{\mathrm{ref}}(u_0)\sum_{i=1}^N\phi(X_i) \, \frac{w_i \, p_{T|0}(u_0 \cond X_i)}{p_T^N(u_0)} \diff u_0=\int\pi_{\mathrm{ref}}(u_0) \, \tilde{\phi}_N(u_0) \diff u_0,
\end{equation}
where $\tilde{\phi}_N(u_0)\coloneqq\sum_{i=1}^N\phi(X_i) \, \frac{w_i \, p_{T|0}(u_0 \cond X_i)}{p_T^N(u_0)}$. 
Then $\pi^N(\phi)=\sum_{i=1}^N w_i \, \phi(X_i)\int p_{T|0}(u_0 \cond X_i) \diff u_0=\int p_T^N(u_0) \, \tilde{\phi}_N(u_0) \diff u_0$ and we have the residual $\pi^{\star,N}(\phi)-\pi^N(\phi)=\int(\pi_{\mathrm{ref}}(u_0)-p_T^N(u_0)) \, \tilde{\phi}_N(u_0) \diff u_0$. 
Hence
\begin{equation}
    \absbig{\pi^{\star,N}(\phi)-\pi^N(\phi)}^2=\int \absbigg{\tilde{\phi}_N(u_0) \, \bigl(\pi_{\mathrm{ref}}(u_0)-p_T^N(u_0) \bigr) \, \frac{\sqrt{p_T^N(u_0)}}{\sqrt{p_T^N(u_0)}}}\diff u_0\leq\pi^N(\phi^2) \, \chi^2(\pi_{\mathrm{ref}} \, \Vert \, p_T^N),
\end{equation}
where the identity $\tilde{\phi}_N(u_0)^2\leq\sum_{i=1}^N\phi(X_i)^2 \, \frac{w_i \, p_{T|0}(u_0 \cond X_i)}{p_T^N(u_0)}$ was applied.
Assume that the test function $\phi$ is uniformly bounded, i.e., $\pi^N(\phi^2) \leq c_\phi$, and also knowing that $p_T = \expec{p_T^N}$, we finally have
\begin{equation}
    \expecbig{\abs{\pi^{\star,N}(\phi)-\pi^N(\phi)}^2}\leq c_\phi \, \expec{\chi^2(p_T \, \Vert \, p_T^N)}+c_\phi \, \expecbigg{\int\frac{p_T(u)-\pi_{\mathrm{ref}}(u)}{p_T^N(u)} \diff u}.
\end{equation}

\section{Common experiment settings}
Unless otherwise stated, all experiments share the following same settings. 

Implementations are based on JAX~\citep{Jax2018github}. 
When applying a Wasserstein distance to quantify the quality of approximate samples, we use the earth-moving cost function denoted by $\mathsf{W}_1$, approximated by a sliced Wasserstein distance with 1,000 projections. 
We directly use the implementation by OTT-JAX~\citep{Cuturi2022OTT}. 
For neural network implementations, we use Flax~\citep{Flax2020github}. 

For all particle filtering, the resampling is triggered at every steps.
When applied to a state-space model, a bootstrap construction of the Feynman--Kac model is consistently used. 

Whenever dealing with probability density evaluations, such as importance sampling, diffusion resampling, and Gaussian mixture, we always implement in the log domain. 

Diffusion resampling uses the mean-reverting construction of the reference distribution $\refmeasure$ as in Algorithm~\ref{app:alg-diffres-normal}, and we set the diffusion coefficient $b^2 = \Sigma_N$.
In practice, the diffusion coefficient should be calibrated, as Corollary~\ref{corollary:asymp} implies that there is likely an optimal $b$. 
Here we choose $b^2 = \Sigma_N$ just to simplify comparison, and this choice is not necessarily optimal. 
All SDE solvers are applied on evenly spaced time grids $0=t_0 < t_1 < \cdots < t_K=T$. 

The numerical solver for SDEs impact the quality of diffusion resampling. 
In all experiments involving diffusion resampling, we test a combination of ways for solving the resampling SDE:
\begin{itemize}
    \item Four SDE integrators. 
    The commonly used Euler--Maruyama, and the two exponential integrators by~\citet{Jentzen2009} and~\citet{Lord2004}, see Section~\ref{sec:exp-integrator} for formulae, and Tweedie's formula~\citep[][essentially DDPM]{Ho2020}. 
    They are called by EM, Jentzen--Kloeden, Lord--Rougemont, and Tweedie, respectively in the later context.
    \item SDE and probability flow. 
    The resampling SDE in Equation~\ref{equ:rev} can be solved also with a probability flow ODE~\citep{Song2021scorebased}. 
    We test both SDE and ODE versions. 
    \item Diffusion time $T$ and the number of discretisation steps $K$. 
    We have mainly tested for $T = 1, 2, 3$ and $K=4, 8, 32$. 
\end{itemize}

The settings above result in at least 72 combinations, and it is not possible to report them all in the main body of the paper due to page limit. 
Therefore, in the main paper we only report the best combination, and we detail the rest of the combinations in appendix.

\paragraph{Gumbel-Softmax resampling} The seminal work by~\citet{Jang2017categorical} did not formulate how the Gumbel trick can be applied for resampling. 
We here make it explicit. 
The gist is to approximate the discrete indexing of multinomial resampling by a matrix-vector multiplication which is differentiable. 
Recall our weighted samples $\lbrace (w_i, X_i) \rbrace_{i=1}^N$, and define a matrix $\cu{X} \in\R^{N \times d}$ aggregating all the $N$ samples. 
For each $i$, draw $N$ independent samples $\lbrace u_{i, j} \sim \mathrm{Uniform}[0, 1] \rbrace_{j=1}^N$ and compute $g_{i, j} = -\log\log u_{i, j}$.
Then, the $i$-th re-sample is
\begin{equation*}
    X^*_i = \sum_{j=1}^N S_{i, j} \, X_j, 
\end{equation*}
where $S_i\in\R^{1 \times N}$ is a Softmax vector with elements $S_{i, j} = \overline{S}_{i, j} \, / \, \sum_{j=1}^N \overline{S}_{i, j}$, where $\overline{S}_{i, j} = \exp((\log w_j + g_{i, j}) \, / \, \tau)$. 
This is independently repeated for $i=1,2,\ldots, N$ to generate the re-samples $\lbrace X^*_i \rbrace_{i=1}^N$. 
Upon the limit $\tau \to 0$, this recovers multinomial resampling, but the variance of the gradient estimation will grow too. 
See also~\citet{Rosato2022}. 

\paragraph{Soft resampling} The approach by~\citet{Karkus2018} view the samples $\lbrace (w_i, X_i) \rbrace_{i=1}^N$ as a discrete distribution with probability given by their normalised weights.
Denote this discrete distribution by $\pi^{\mathrm{D}}$. 
Directly sampling from $\pi^{\mathrm{D}}$ using weighted random choice will result in undefined gradients. 
Soft resampling mitigates this issue by introducing another proposal (discrete) distribution $q^{\mathrm{D}}$ defined by $q^{\mathrm{D}}(X_i) \coloneqq \alpha \, w_i + (1 - \alpha) \, \frac{1}{N}$ for $i=1,2,\ldots, N$. 
Then, sampling from $\pi^{\mathrm{D}}$ can be achieved by an importance sampling as follows.
\begin{enumerate}
    \item Draw $I_i \sim \mathrm{Categorical}\bigl( w^q_1, w^q_2, \ldots, w^q_N \bigr)$, where $w^q_i \coloneqq q^{\mathrm{D}}(X_i) = \alpha \, w_i + (1 - \alpha) \, \frac{1}{N}$. 
    \item Indexing $X^*_i = X_{I_i}$. Until here, we have sampled from the proposal $q^{\mathrm{D}}$. 
    \item Weight $w^*_i = \pi^{\mathrm{D}}(X^*_i) \, / \, q^{\mathrm{D}}(X^*_i) = w_{I_i} \, / \, w^q_{I_i}$ and then normalise. 
    \item Return $ \lbrace (w_i^*, X_i^*) \rbrace_{i=1}^N$. 
\end{enumerate}
Clearly, we can see that when $\alpha=0$, the gradient is fully defined but the resampling and the gradient completely discard information from the weights $\lbrace w_i \rbrace_{i=1}^N$, resulting in high variance. 
The setting $\alpha=1$ recovers the standard multinomial resampling, but the gradient becomes undefined. 
The gradient produced by method is thus always biased.
Furthermore, we can also see that the soft resampling does not return uniform resampling weights unlike other methods, which to some extents, contradicts the purpose of resampling. 

We will compare to the soft resampling and Gumbel-Softmax resampling with different settings of their tuning parameters. 

\section{Gaussian mixture resampling}
\label{app:exp-gm}
Recall the model
\begin{equation}
    \begin{split}
        \phi(x) &= \sum_{i=1}^c \omega_i \, \mathrm{N}(x; m_i, v_i), \\
        p(y \cond x) &= \mathrm{N}(y; H \, x, \Xi),
    \end{split}
\end{equation}
where $x\in\R^d$, $y \in\R$, $d=8$, and the number of components $c=5$. 
The posterior distribution $\pi(x) \propto \phi(x) \, p(y \cond x)$ is also a Gaussian mixture $\pi(x) = \sum_{i=1}^c \Omega_i \, \mathrm{N}(x; \mathcal{M}_i, \mathcal{V}_i)$~\citep{Zhao2025B0SMC}, given by 
\begin{equation*}
    \begin{split}
        G_i &= H \, v_i \, H^\trans + \Xi, \\
        \overline{\Omega}_i &= \omega_i \, \mathrm{N}(y; H \, m_i, G_i), \\
        \Omega_i &= \overline{\Omega}_i \, / \, \sum_{j=1}^c \overline{\Omega}_j, \\
        \mathcal{M}_i &= m_i + v_i \, H^\trans  \, G^{-1}_i \, (y - H \, m_i), \\
        \mathcal{V}_i &= v_i - v_i \, H^\trans \, G^{-1}_i \, H \, v_i.
    \end{split}
\end{equation*}

At each individual experiment, we randomly generate the Gaussian mixture components. 
Specifically, we draw the mixture mean $m_i \sim \mathrm{Uniform}([-5, 5]^d)$ and the covariance $v_i = \overline{v}_i + I_d$, where $\overline{v}_i \sim \mathrm{Wishart}(I_d)$.
To reduce experiment variance, the weights $\omega_i = 1 \, / \, c$ are fixed to be even, the observation operator $H\in\R^{1 \times d}$ is an all-one vector, and $\Xi=1$.
Experiments are repeated 100 times independently.

\begin{remark}[Resampling variance]
    In literature, the resampling variance is usually defined as the variance of the resampling algorithm itself conditioned on the input samples, that is, $\varrBig{\frac{1}{N}\sum_{i=1}^N X_i^* \cond \lbrace (w_i, X_i) \rbrace_{i=1}^N}$. 
    This is useful to heuristically gauge the noise level of resampling. 
    However, this does not inform how well the re-samples approximate the true distribution $\pi$. 
    With a slight abuse of terminology, we define the resampling variance as the mean estimator error $\expecBig{\bigl(\frac{1}{N}\sum_{i=1}^N X_i^* - \mathcal{M} \bigr)^2}$, where $\mathcal{M}$ stands for the true mean of the Gaussian mixture posterior $\pi$.
    The expectation is approximated by the 100 independent Monte Carlo runs.
\end{remark}

Tables~\ref{tbl:app-gm-swd-1} and~\ref{tbl:app-gm-swd-2} show more detailed results compared to Table~\ref{tbl:gm-compare} in the main body. 
We see that with a fixed fine-enough discretisation, the SWD in general decreases as $T$ increases independent of the SDE integrator used. 
It also shows that the probability flow (ODE) version for solving the resampling SDE is better than SDE, especially when $K$ is small. 
As for the integrators, Jentzen--Kloeden appears to be the best, and it is especially more useful when the discretisation is coarse. 
The OT resampling performs roughly the same as with diffusion resampling with setting $\varepsilon=0.8$ and $T=3, K=32$. 

\begin{table}[]
\centering
\caption{Sliced 1-Wasserstein distance (SWD, scaled by $10^{-1}$) of the Gaussian mixture experiments. 
This table focuses only on the diffusion resampling. 
We see that the Jentzen--Kloeden integrator performs the best in average. }
\label{tbl:app-gm-swd-1}
\wrapbox{.99}{%
    \begin{tabular}{@{}llllllll@{}}
    \toprule
     \multirow{2}{*}{Method}  & \multicolumn{2}{l}{Euler--Maruyama} & \multicolumn{2}{l}{Lord--Rougemont} & \multicolumn{2}{l}{Jentzen--Kloeden} & Tweedie         \\ \cmidrule(l){2-8} 
                   & ODE              & SDE              & ODE              & SDE              & ODE               & SDE              & SDE             \\ \cmidrule(l){2-8} 
    $T=1$, $K=8$   & $1.92 \pm 0.41$  & $2.39 \pm 0.57$  & $1.79 \pm 0.27$  & $3.04 \pm 0.67$  & $1.64 \pm 0.35$   & $2.70 \pm 0.62$  & $2.10 \pm 0.51$ \\
    $T=1$, $K=32$  & $1.31 \pm 0.30$  & $1.03 \pm 0.26$  & $1.29 \pm 0.30$  & $1.03 \pm 0.26$  & $1.26 \pm 0.30$   & $1.03 \pm 0.26$  & $1.01 \pm 0.26$ \\
    $T=1$, $K=128$ & $1.23 \pm 0.31$  & $0.92 \pm 0.25$  & $1.23 \pm 0.31$  & $0.90 \pm 0.25$  & $1.22 \pm 0.31$   & $0.91 \pm 0.25$  & $0.92 \pm 0.25$ \\
    $T=2$, $K=8$   & $2.45 \pm 0.56$  & $4.15 \pm 0.91$  & $4.41 \pm 0.74$  & $7.38 \pm 1.37$  & $1.93 \pm 0.45$   & $5.56 \pm 1.04$  & $3.28 \pm 0.80$ \\
    $T=2$, $K=32$  & $1.02 \pm 0.22$  & $1.38 \pm 0.34$  & $1.22 \pm 0.23$  & $1.66 \pm 0.39$  & $0.93 \pm 0.22$   & $1.51 \pm 0.35$  & $1.27 \pm 0.32$ \\
    $T=2$, $K=128$ & $0.82 \pm 0.21$  & $0.84 \pm 0.24$  & $0.84 \pm 0.21$  & $0.85 \pm 0.24$  & $0.82 \pm 0.21$   & $0.85 \pm 0.24$  & $0.83 \pm 0.24$ \\
    $T=3$, $K=8$   & $3.22 \pm 0.73$  & $5.60 \pm 1.14$  & $10.78 \pm 2.17$ & $13.21 \pm 2.14$ & $2.53 \pm 0.60$   & $8.62 \pm 1.42$  & $4.12 \pm 0.97$ \\
    $T=3$, $K=32$  & $1.20 \pm 0.26$  & $1.89 \pm 0.46$  & $2.27 \pm 0.45$  & $2.52 \pm 0.57$  & $1.06 \pm 0.24$   & $2.16 \pm 0.49$  & $1.66 \pm 0.43$ \\
    $T=3$, $K=128$ & $0.82 \pm 0.21$  & $0.89 \pm 0.24$  & $0.93 \pm 0.21$  & $0.94 \pm 0.25$  & $0.80 \pm 0.21$   & $0.91 \pm 0.25$  & $0.87 \pm 0.24$ \\ \bottomrule
    \end{tabular}
}
\end{table}

\begin{table}[]
    \centering
    \caption{Sliced 1-Wasserstein distance (SWD, scaled by $10^{-1}$) and resampling variance (scaled by $10^{-2}$) of the Gaussian mixture experiments. 
    This table should be compared to Table~\ref{tbl:app-gm-swd-1}. }
    \label{tbl:app-gm-swd-2}
    \wrapbox{.99}{%
        \begin{tabular}{@{}lll@{}}
        \toprule
        Method                 & SWD             & Resampling variance \\ \midrule
        OT ($\varepsilon=0.3$) & $0.84 \pm 0.22$ & $3.42 \pm 3.26$     \\
        OT ($\varepsilon=0.6$) & $0.97 \pm 0.21$ & $3.41 \pm 3.29$     \\
        OT ($\varepsilon=0.8$) & $1.08 \pm 0.20$ & $3.42 \pm 3.30$     \\
        OT ($\varepsilon=0.9$) & $1.14 \pm 0.20$ & $3.42 \pm 3.29$     \\ \bottomrule
        \end{tabular}
        \hfill
        \begin{tabular}{@{}lll@{}}
        \toprule
        Method     & SWD              & Resampling variance \\ \midrule
        Gumbel 0.2 & $1.40 \pm 0.24$  & $3.92 \pm 3.74$     \\
        Gumbel 0.2 & $2.52 \pm 0.41$  & $3.90 \pm 3.76$     \\
        Gumbel 0.4 & $5.16 \pm 0.82$  & $3.83 \pm 3.76$     \\
        Gumbel 0.8 & $11.57 \pm 1.81$ & $3.59 \pm 3.54$     \\ \bottomrule
        \end{tabular}
        \hfill
        \begin{tabular}{@{}lll@{}}
        \toprule
        Method   & SWD             & Resampling variance \\ \midrule
        Soft 0.2 & $0.92 \pm 0.25$ & $4.71 \pm 3.60$     \\
        Soft 0.4 & $0.86 \pm 0.23$ & $4.11 \pm 3.98$     \\
        Soft 0.8 & $0.85 \pm 0.24$ & $4.12 \pm 4.09$     \\
        Soft 0.9 & $0.83 \pm 0.24$ & $3.75 \pm 3.77$     \\ \bottomrule
        \end{tabular}
    }
\end{table}

\section{Time comparison}
\label{app:time}
In this experiment we focus on the actual computational time and compare different methods when controlling them to have a similar level of estimation error. 
The main result is shown in Figure~\ref{fig:time}, supplemented by Tables~\ref{tbl:time-err-diffres} and~\ref{tbl:time-err-ot}. 

The results are averaged over 50 independent runs on an NVIDIA A100 80G GPU with a fixed dimension $d=8$.
We have also experimented on a CPU (AMD EPYC 9354 32-Core) but we did not observe any significant difference compared to that of GPU worth to report. 

For previous experiments, $T$ and $K$ are given, and the time interval is adapted by $\Delta = T \, / \,(K + 1)$. 
For this experiment, we fix a time interval $\Delta =0.1$ and change $K$, so that $T = \Delta \, (K + 1)$. 
We compute the resampling error which is a square root of the resampling variance. 

Tables~\ref{tbl:time-err-diffres} and~\ref{tbl:time-err-ot} show that the diffusion and OT methods have a similar estimation error especially for large sample size $N$. 
Moreover, the results empirically verify the theoretical connection between the OT regularisation $\varepsilon$ and the diffusion time $T$ stated in Section~\ref{sec:analysis}. 
That is, $K$ scales proportionally to $1 \, / \, \varepsilon$ in order for them to have the same level of statistical performance. 
However, as evidenced in Figure~\ref{fig:time}, the diffusion resampling is in average faster than that of OT.
When both methods are at their best estimation performance (i.e., $\varepsilon=0.1$ and $K=32$), diffusion resampling is still faster than OT. 

\begin{table}[t!]
    \caption{The resampling error of diffusion resampling in terms of the number of time steps $K$ and sample size $N$. 
    Related to the time experiment in Appendix~\ref{app:time}. }
    \label{tbl:time-err-diffres}
    \centering
    \begin{tabular}{@{}llllllll@{}}
    \toprule
     \multirow{2}{*}{Method}  & \multicolumn{7}{c}{Number of samples $N$}                                                                                                                                               \\ \cmidrule(l){2-8} 
           & \multicolumn{1}{c}{128} & \multicolumn{1}{c}{256} & \multicolumn{1}{c}{512} & \multicolumn{1}{c}{1024} & \multicolumn{1}{c}{2048} & \multicolumn{1}{c}{4096} & \multicolumn{1}{c}{8192} \\ \cmidrule(l){2-8} 
    $K=4$  & 0.29 $\pm$ 0.07         & 0.21 $\pm$ 0.04         & 0.15 $\pm$ 0.04         & 0.11 $\pm$ 0.03          & 0.08 $\pm$ 0.02          & 0.05 $\pm$ 0.01          & 0.04 $\pm$ 0.01          \\
    $K=8$  & 0.28 $\pm$ 0.06         & 0.21 $\pm$ 0.04         & 0.15 $\pm$ 0.04         & 0.11 $\pm$ 0.03          & 0.08 $\pm$ 0.03          & 0.05 $\pm$ 0.01          & 0.04 $\pm$ 0.01          \\
    $K=16$ & 0.27 $\pm$ 0.06         & 0.20 $\pm$ 0.04         & 0.15 $\pm$ 0.04         & 0.11 $\pm$ 0.03          & 0.08 $\pm$ 0.03          & 0.04 $\pm$ 0.01          & 0.04 $\pm$ 0.01          \\
    $K=32$ & 0.27 $\pm$ 0.06         & 0.20 $\pm$ 0.04         & 0.15 $\pm$ 0.04         & 0.11 $\pm$ 0.03          & 0.08 $\pm$ 0.03          & 0.04 $\pm$ 0.01          & 0.04 $\pm$ 0.01          \\ \bottomrule
    \end{tabular}
\end{table}

\begin{table}[t!]
    \caption{The resampling error of OT in terms of entropy regularisation $\varepsilon$ and sample size $N$. 
    Related to the time experiment in Appendix~\ref{app:time}.}
    \label{tbl:time-err-ot}
    \centering
    \begin{tabular}{@{}llllllll@{}}
    \toprule
    \multirow{2}{*}{Method}  & \multicolumn{7}{c}{Number of samples $N$}                                                                                                                                               \\ \cmidrule(l){2-8} 
                      & \multicolumn{1}{c}{128} & \multicolumn{1}{c}{256} & \multicolumn{1}{c}{512} & \multicolumn{1}{c}{1024} & \multicolumn{1}{c}{2048} & \multicolumn{1}{c}{4096} & \multicolumn{1}{c}{8192} \\ \cmidrule(l){2-8} 
    $\varepsilon=0.8$ & 0.27 $\pm$ 0.06         & 0.20 $\pm$ 0.03         & 0.15 $\pm$ 0.03         & 0.11 $\pm$ 0.02          & 0.08 $\pm$ 0.02          & 0.05 $\pm$ 0.01          & 0.04 $\pm$ 0.01          \\
    $\varepsilon=0.4$ & 0.27 $\pm$ 0.06         & 0.20 $\pm$ 0.03         & 0.15 $\pm$ 0.03         & 0.11 $\pm$ 0.02          & 0.08 $\pm$ 0.02          & 0.05 $\pm$ 0.01          & 0.04 $\pm$ 0.01          \\
    $\varepsilon=0.2$ & 0.27 $\pm$ 0.06         & 0.20 $\pm$ 0.03         & 0.15 $\pm$ 0.03         & 0.11 $\pm$ 0.02          & 0.08 $\pm$ 0.02          & 0.05 $\pm$ 0.01          & 0.04 $\pm$ 0.01          \\
    $\varepsilon=0.1$ & 0.27 $\pm$ 0.06         & 0.20 $\pm$ 0.03         & 0.15 $\pm$ 0.03         & 0.11 $\pm$ 0.02          & 0.08 $\pm$ 0.02          & 0.05 $\pm$ 0.01          & 0.04 $\pm$ 0.01          \\ \bottomrule
    \end{tabular}
\end{table}

\section{Linear Gaussian SSM}
\label{app:lgssm}

Recall the model
\begin{equation}
    \begin{split}
        Z_j \cond Z_{j-1} &\sim \mathrm{N}(z_j ; \theta_1 \, z_{j-1}, I_d), \quad Z_0 \sim \mathrm{N}(0, I_d), \\
        Y_j \cond Z_j &\sim \mathrm{N}(y_j ; \theta_2 \, z_j, 0.5 \, I_d),
    \end{split}
\end{equation}
where we set parameters $\theta_1=0.5$ and $\theta_2=1$. 
For each experiment, we generate a measurement sequence with $j=0,1,\ldots, 128$ steps, with $N=32$ particles. 
The loss function estimate error $\norm{L - \widehat{L}}^2_2 \coloneqq \int_\Theta (L(\theta_1, \theta_2) - \widehat{L}(\theta_1, \theta_2))^2 \diff \theta_1\diff\theta_2$, where $\widehat{L}$ is estimated by Trapezoidal quadrature at the Cartesian grids $\Theta = [\theta_1 - 0.1, \theta_1 + 0.1]\times [\theta_2 - 0.1, \theta_2 + 0.1]$. 
The filtering error is evaluated by Kullback--Leibler (KL) divergence between the true filtering distribution (which is a Gaussian computed exactly by a Kalman filter) and the particle filtering samples. 
Precisely, the error is defined by $\frac{1}{129}\sum_{j=0}^{128}\klbig{\mathrm{N}(m_j^f, V^f_j)}{\mathrm{N}(\hat{m}_j^f, \hat{V}^f_j)}$, where $\mathrm{N}(m_j^f, V^f_j)$ is the true filtering distribution at step $j$, and $\mathrm{N}(\hat{m}_j^f, \hat{V}^f_j)$ stands for the empirical approximation by the particle samples. 

The optimiser L-BFGS is implemented using JAXopt~\citep{Blondel2021} \texttt{scipyminimize} wrapper with default settings and initial parameters $\theta + 1$. 

We observe a numerical issue due to the gradient estimate by the resampling methods. 
The L-BFGS optimiser frequently diverges. 
This is expected, since L-BFGS is highly sensitive to the quality of gradient estimate~\citep{Xie2020}; see also~\citet{Zhao2023MFS} for an empirical validation. 
As such, when we report the results we have defined a convergent run by 1) positive convergence flag returned by the L-BFGS optimiser and 2) the error $\norm{\theta - \widehat{\theta}}_2 < 1.9 $ is not absurd. 
It turns out that the effect is especially pronounced for the soft and Gumbel-Softmax resampling, where at 100 runs, only nearly 10\% return a successful flag of the optimiser, and all their estimated parameters diverge to a meaningless position. 
This was not problematic for the diffusion and OT samplers, where they have approximately 80\% success rate in average. 

We also observe a numerical issue of OT implementation. 
By default, for instance, in OTT-JAX, gradient propagation through the Sinkhorn solver uses implicit differentiation by solving a linear system. 
This makes gradient computation of the OT resampling efficient. 
However, the linear system often becomes ill-conditioned for most runs, and making OT parameter estimation diverges largely. 
We thus had to disable this feature by unrolling the gradients. 
To fairly compare SDEs and OT, we have also unrolled gradient propagation through SDEs without using any implicit/adjoint methods. 

The results are detailed in Tables~\ref{tbl:app-lgssm-loss-diffusion} to~\ref{tbl:app-lgssm-loss-others}.
We clearly see that the diffusion resampling outperforms other methods across all the three metrics, in particular for the filtering and parameter estimation errors. 
The superiority of the log-likelihood function estimation is not significant due to that the loss function's magnitude does not change much in $\Theta$, see also Figure~\ref{fig:lgssm-loss}. 

Let us focus on the diffusion resampling in Tables~\ref{tbl:app-lgssm-loss-diffusion} and~\ref{tbl:app-lgssm-params-diffusion}. 
In Tables~\ref{tbl:app-lgssm-loss-diffusion} we find that the loss function estimation error seems to be invariant in the time $T$ and steps $K$ when using an ODE solver (except for Lord--Rougemont), giving a relatively large error $2.72$. 
With a fixed $T$, it is not clear if increasing $K$ will improve the log-likelihood estimate. 
This also resonates with Table~\ref{tbl:app-lgssm-params-diffusion} that increasing $K$ can possibly lead to worst even exploding parameter estimation. 
This may be caused by numerical errors when back-propagating gradients through SDE solvers which can be addressed by, for instance, \citet{Kidger2021gradient}.
On the other hand, with $K$ fixed, increasing $T$ improves the result. 
Like in the previous experiments, the Jentzen--Kloeden exponential integrator consistently achieves the best result, albeit marginally, while Lord--Rougemont is the most sensitive to $K$. 

Table~\ref{tbl:app-lgssm-kl-diffusion} shows the performance of the filtering, and this does not compute any gradients.  
We see that improving $T$ and $K$ improves the filtering estimation in general. 
This empirically verifies the results from Tables~\ref{tbl:app-lgssm-loss-diffusion} and~\ref{tbl:app-lgssm-params-diffusion} that efficient and accurate back-propagation through SDE solvers may be necessary. 

Figure~\ref{fig:lgssm-loss} shows loss function landscapes estimated by the particle filtering ($N=32$ particles) with different resampling schemes. 
The calibration parameters for Gumbel-Softmax and soft resampling are 0.2 and 0.1, respectively, and for diffusion, we use Jentzen--Kloden SDE integrator with $T=2$ and $K=4$, and for OT $\varepsilon=0.3$. 
We see in this figure that multinomial gives the most noisy estimate, hard for searching the optimum. 
This is also true for soft resampling, and even at a low parameter $0.1$ which already means very uninformative resampling, the loss function is still rather noisy. 
On the other hand, the loss function estimates by diffusion, OT, and Gumbel-Softmax look smooth.

\begin{table}[]
\centering
\caption{The errors of loss function estimate $\norm{L - \widehat{L}}_2$ associated with the linear Gaussian state-space model (LGSSM) experiment. 
This table focuses only on the diffusion resampling and should be compared to Table~\ref{tbl:app-lgssm-loss-others} }
\label{tbl:app-lgssm-loss-diffusion}
\wrapbox{.99}{%
    \begin{tabular}{@{}llllllll@{}}
    \toprule
     \multirow{2}{*}{Method} & \multicolumn{2}{l}{Euler--Maruyama} & \multicolumn{2}{l}{Lord--Rougemont} & \multicolumn{2}{l}{Jentzen--Kloeden} & Tweedie         \\ \cmidrule(l){2-8} 
                  & ODE              & SDE              & ODE              & SDE              & ODE               & SDE              & SDE             \\ \cmidrule(l){2-8} 
    $T=1$, $K=4$  & $2.72 \pm 2.12$  & $2.66 \pm 2.11$  & $2.90 \pm 2.24$  & $2.63 \pm 2.10$  & $2.72 \pm 2.12$   & $2.61 \pm 2.08$  & $2.69 \pm 2.13$ \\
    $T=1$, $K=8$  & $2.72 \pm 2.13$  & $2.70 \pm 2.00$  & $2.77 \pm 2.17$  & $2.69 \pm 1.98$  & $2.72 \pm 2.13$   & $2.68 \pm 1.98$  & $2.71 \pm 2.02$ \\
    $T=1$, $K=16$ & $2.72 \pm 2.13$  & $2.77 \pm 2.23$  & $2.74 \pm 2.15$  & $2.79 \pm 2.23$  & $2.72 \pm 2.13$   & $2.77 \pm 2.23$  & $2.77 \pm 2.23$ \\
    $T=1$, $K=32$ & $2.72 \pm 2.14$  & $2.67 \pm 2.12$  & $2.73 \pm 2.15$  & $2.66 \pm 2.11$  & $2.72 \pm 2.14$   & $2.66 \pm 2.11$  & $2.67 \pm 2.12$ \\
    $T=2$, $K=4$  & $2.72 \pm 2.11$  & $2.58 \pm 2.06$  & $5.68 \pm 4.25$  & $2.55 \pm 2.05$  & $2.72 \pm 2.11$   & $2.51 \pm 1.95$  & $2.64 \pm 2.14$ \\
    $T=2$, $K=8$  & $2.72 \pm 2.12$  & $2.66 \pm 1.93$  & $3.29 \pm 2.51$  & $2.69 \pm 1.92$  & $2.72 \pm 2.12$   & $2.62 \pm 1.89$  & $2.70 \pm 1.97$ \\
    $T=2$, $K=16$ & $2.72 \pm 2.13$  & $2.79 \pm 2.31$  & $2.87 \pm 2.23$  & $2.81 \pm 2.30$  & $2.72 \pm 2.13$   & $2.78 \pm 2.32$  & $2.80 \pm 2.30$ \\
    $T=2$, $K=32$ & $2.72 \pm 2.13$  & $2.62 \pm 2.10$  & $2.77 \pm 2.18$  & $2.62 \pm 2.10$  & $2.72 \pm 2.13$   & $2.61 \pm 2.09$  & $2.62 \pm 2.11$ \\
    $T=3$, $K=4$  & $2.72 \pm 2.11$  & $2.51 \pm 2.01$  & $14.89 \pm 8.90$ & $33.38 \pm 5.45$ & $2.72 \pm 2.11$   & $2.57 \pm 1.67$  & $2.63 \pm 2.13$ \\
    $T=3$, $K=8$  & $2.72 \pm 2.11$  & $2.61 \pm 1.92$  & $5.90 \pm 4.39$  & $2.70 \pm 2.04$  & $2.72 \pm 2.11$   & $2.55 \pm 1.89$  & $2.67 \pm 1.96$ \\
    $T=3$, $K=16$ & $2.72 \pm 2.12$  & $2.78 \pm 2.33$  & $3.37 \pm 2.56$  & $2.82 \pm 2.33$  & $2.72 \pm 2.12$   & $2.76 \pm 2.35$  & $2.80 \pm 2.32$ \\
    $T=3$, $K=32$ & $2.72 \pm 2.13$  & $2.58 \pm 2.08$  & $2.89 \pm 2.23$  & $2.56 \pm 2.09$  & $2.72 \pm 2.13$   & $2.56 \pm 2.07$  & $2.59 \pm 2.09$ \\ \bottomrule
    \end{tabular}
}
\end{table}

\begin{table}[]
\centering
\caption{The filtering error in terms of KL divergence (scaled by $10^{-1}$) associated with the linear Gaussian state-space model (LGSSM) experiment. 
This table focuses only on the diffusion resampling and should be compared to Table~\ref{tbl:app-lgssm-loss-others} }
\label{tbl:app-lgssm-kl-diffusion}
\wrapbox{.99}{%
    \begin{tabular}{@{}llllllll@{}}
    \toprule
    \multirow{2}{*}{Method} & \multicolumn{2}{l}{Euler--Maruyama} & \multicolumn{2}{l}{Lord--Rougemont} & \multicolumn{2}{l}{Jentzen--Kloeden} & Tweedie         \\ \cmidrule(l){2-8} 
                  & ODE              & SDE              & ODE              & SDE              & ODE               & SDE              & SDE             \\ \cmidrule(l){2-8} 
    $T=1$, $K=4$  & $5.03 \pm 5.75$  & $5.05 \pm 7.00$  & $5.39 \pm 6.57$   & $5.03 \pm 7.59$  & $5.02 \pm 5.73$   & $4.94 \pm 6.92$  & $5.28 \pm 8.24$ \\
    $T=1$, $K=8$  & $4.99 \pm 5.67$  & $5.05 \pm 7.69$  & $5.10 \pm 5.86$   & $4.90 \pm 6.95$  & $4.99 \pm 5.65$   & $4.97 \pm 7.36$  & $5.10 \pm 7.98$ \\
    $T=1$, $K=16$ & $4.98 \pm 5.68$  & $5.11 \pm 7.90$  & $4.99 \pm 5.64$   & $5.50 \pm 11.01$ & $4.97 \pm 5.67$   & $5.17 \pm 8.46$  & $5.03 \pm 7.27$ \\
    $T=1$, $K=32$ & $5.01 \pm 5.87$  & $4.84 \pm 5.24$  & $5.00 \pm 5.83$   & $4.82 \pm 5.26$  & $5.00 \pm 5.87$   & $4.80 \pm 5.13$  & $4.85 \pm 5.25$ \\
    $T=2$, $K=4$  & $5.07 \pm 5.84$  & $4.71 \pm 5.76$  & $7.48 \pm 9.61$   & $5.06 \pm 12.06$ & $5.06 \pm 5.82$   & $4.38 \pm 5.47$  & $5.07 \pm 6.81$ \\
    $T=2$, $K=8$  & $5.03 \pm 5.75$  & $4.60 \pm 5.20$  & $5.87 \pm 7.56$   & $4.35 \pm 4.49$  & $5.02 \pm 5.72$   & $4.40 \pm 4.78$  & $4.82 \pm 6.01$ \\
    $T=2$, $K=16$ & $4.99 \pm 5.66$  & $5.24 \pm 8.62$  & $5.27 \pm 6.18$   & $5.18 \pm 8.12$  & $4.98 \pm 5.65$   & $5.17 \pm 8.38$  & $5.32 \pm 8.97$ \\
    $T=2$, $K=32$ & $4.98 \pm 5.68$  & $4.86 \pm 5.48$  & $5.01 \pm 5.62$   & $4.82 \pm 5.32$  & $4.97 \pm 5.67$   & $4.77 \pm 5.28$  & $4.89 \pm 5.56$ \\
    $T=3$, $K=4$  & $5.09 \pm 5.89$  & $4.39 \pm 5.08$  & $15.80 \pm 19.50$ & $3.59 \pm 5.97$  & $5.08 \pm 5.87$   & $4.13 \pm 7.93$  & $4.85 \pm 5.71$ \\
    $T=3$, $K=8$  & $5.06 \pm 5.81$  & $4.51 \pm 4.92$  & $7.68 \pm 9.93$   & $4.16 \pm 4.35$  & $5.05 \pm 5.78$   & $4.26 \pm 4.49$  & $4.76 \pm 5.57$ \\
    $T=3$, $K=16$ & $5.01 \pm 5.71$  & $5.17 \pm 7.26$  & $5.96 \pm 7.74$   & $5.29 \pm 7.74$  & $5.00 \pm 5.69$   & $5.17 \pm 7.50$  & $5.18 \pm 7.31$ \\
    $T=3$, $K=32$ & $4.98 \pm 5.65$  & $4.78 \pm 5.41$  & $5.26 \pm 6.13$   & $4.74 \pm 5.39$  & $4.97 \pm 5.64$   & $4.76 \pm 5.44$  & $4.81 \pm 5.48$ \\ \bottomrule
    \end{tabular}
}
\end{table}

\begin{table}[]
\centering
\caption{The parameter estimation error (scaled by $10^{-1}$) associated with the linear Gaussian state-space model (LGSSM) experiment. 
This table focuses only on the diffusion resampling and should be compared to Table~\ref{tbl:app-lgssm-loss-others}. }
\label{tbl:app-lgssm-params-diffusion}
\wrapbox{.99}{%
    \begin{tabular}{@{}llllllll@{}}
    \toprule
    \multirow{2}{*}{Method} & \multicolumn{2}{l}{Euler--Maruyama} & \multicolumn{2}{l}{Lord--Rougemont} & \multicolumn{2}{l}{Jentzen--Kloeden} & Tweedie         \\ \cmidrule(l){2-8} 
                  & ODE              & SDE              & ODE              & SDE              & ODE               & SDE              & SDE             \\ \cmidrule(l){2-8} 
    $T=1$, $K=4$  & $1.32 \pm 0.70$  & $1.45 \pm 1.90$  & $1.72 \pm 0.98$  & $1.65 \pm 1.61$  & $1.36 \pm 0.74$  & $1.28 \pm 0.70$   & $1.46 \pm 1.94$  \\
    $T=1$, $K=8$  & $1.48 \pm 0.93$  & $1.70 \pm 1.86$  & $1.56 \pm 0.94$  & $2.25 \pm 3.09$  & $1.44 \pm 0.82$  & $2.14 \pm 3.35$   & $1.75 \pm 2.51$  \\
    $T=1$, $K=16$ & $1.76 \pm 2.11$  & $11.30 \pm 5.22$ & $1.59 \pm 1.55$  & $11.94 \pm 4.90$ & $1.40 \pm 0.74$  & $12.07 \pm 4.68$  & $11.93 \pm 4.94$ \\
    $T=1$, $K=32$ & $4.43 \pm 5.39$  & $14.14 \pm 0.00$ & $4.05 \pm 5.33$  & $14.14 \pm 0.00$ & $3.23 \pm 4.38$  & $14.14 \pm 0.00$  & $14.14 \pm 0.00$ \\
    $T=2$, $K=4$  & $1.34 \pm 0.70$  & $1.25 \pm 0.70$  & $6.07 \pm 2.57$  & $3.19 \pm 3.83$  & $1.33 \pm 0.69$  & $1.49 \pm 1.55$   & $1.26 \pm 0.67$  \\
    $T=2$, $K=8$  & $1.36 \pm 0.76$  & $1.30 \pm 0.68$  & $2.29 \pm 1.22$  & $1.68 \pm 1.47$  & $1.34 \pm 0.70$  & $1.29 \pm 0.78$   & $1.23 \pm 0.62$  \\
    $T=2$, $K=16$ & $1.39 \pm 0.73$  & $1.84 \pm 2.52$  & $1.80 \pm 1.94$  & $1.72 \pm 2.33$  & $1.48 \pm 0.93$  & $1.39 \pm 0.89$   & $1.63 \pm 1.77$  \\
    $T=2$, $K=32$ & $1.59 \pm 1.69$  & $10.57 \pm 5.68$ & $1.68 \pm 2.05$  & $12.40 \pm 4.41$ & $1.62 \pm 1.67$  & $11.49 \pm 5.06$  & $13.31 \pm 3.06$ \\
    $T=3$, $K=4$  & $1.34 \pm 0.70$  & $1.27 \pm 0.74$  & $8.74 \pm 3.94$  & $14.14 \pm 0.00$ & $1.39 \pm 0.81$  & $1.79 \pm 1.09$   & $1.23 \pm 0.67$  \\
    $T=3$, $K=8$  & $1.32 \pm 0.70$  & $1.25 \pm 0.62$  & $5.96 \pm 2.64$  & $2.26 \pm 2.61$  & $1.32 \pm 0.70$  & $1.28 \pm 0.75$   & $1.28 \pm 0.61$  \\
    $T=3$, $K=16$ & $1.40 \pm 0.72$  & $1.48 \pm 1.56$  & $2.48 \pm 1.31$  & $1.41 \pm 0.89$  & $1.39 \pm 0.72$  & $1.26 \pm 0.74$   & $1.53 \pm 2.01$  \\
    $T=3$, $K=32$ & $1.39 \pm 0.73$  & $1.95 \pm 2.88$  & $1.92 \pm 2.07$  & $3.20 \pm 4.33$  & $1.41 \pm 0.76$  & $3.26 \pm 4.43$   & $2.81 \pm 4.32$  \\ \bottomrule
    \end{tabular}
}
\end{table}

\begin{table}[]
    \centering
    \caption{The errors of loss function, filtering in terms of KL divergence (scaled by $10^{-1}$), and parameter estimation (scaled by $10^{-1}$) associated with the linear Gaussian state-space model (LGSSM) experiment. 
    This table focuses only on methods other than the diffusion resampling, and should be compared to Tables~\ref{tbl:app-lgssm-loss-diffusion}, \ref{tbl:app-lgssm-kl-diffusion}, and \ref{tbl:app-lgssm-params-diffusion}. }
    \label{tbl:app-lgssm-loss-others}
    \wrapbox{.5}{%
        \begin{tabular}{@{}llll@{}}
        \toprule
        Method                 & $\norm{L - \widehat{L}}_2$ & Filtering KL         & $\norm{\theta - \widehat{\theta}}_2$ \\ \midrule
        OT ($\varepsilon=0.4$) & $2.64 \pm 2.13$            & $5.07 \pm 6.21$      & $1.53 \pm 1.16$                      \\
        OT ($\varepsilon=0.8$) & $2.68 \pm 2.16$            & $5.07 \pm 5.70$      & $1.58 \pm 1.22$                      \\
        OT ($\varepsilon=1.6$) & $2.76 \pm 2.20$            & $5.11 \pm 5.17$      & $1.49 \pm 0.97$                      \\
        Gumbel (0.1)           & $2.79 \pm 2.14$            & $4.83 \pm 5.76$      & NaN                     \\
        Gumbel (0.3)           & $2.75 \pm 2.17$            & $4.89 \pm 5.42$      & NaN                     \\
        Gumbel (0.5)           & $2.73 \pm 2.22$            & $5.10 \pm 5.93$      & NaN                     \\
        Soft (0.5)             & $3.63 \pm 2.10$           & $7.75 \pm 11.19$ & NaN                     \\
        Soft (0.7)             & $3.08 \pm 1.88$            & $5.34 \pm 7.82$    & NaN                     \\
        Soft (0.9)             & $2.85 \pm 1.80$            & $4.66 \pm 5.68$     & NaN                     \\
        Multinomial            & $2.80 \pm 1.84$            & $5.49 \pm 6.87$      & NaN                     \\ \bottomrule
        \end{tabular}
    }
\end{table}

\section{Prey-predator model}
\label{app:prey-predator}
Recall our prey-predator (or also called Lokta--Volterra) model
\begin{equation}
    \begin{split}
        \diff C(t) &= C(t) \, (\alpha - \beta \, R(t)) \diff t + \sigma \, C(t) \diff W_{1}(t),\\
        \diff R(t) &= R(t) \, (\zeta \, C(t) - \gamma) \diff t + \sigma \, R(t) \diff W_{2}(t), \\
        Y_j &\sim \mathrm{Poisson}\bigl(\lambda(C(t_j), R(t_j))\bigr),
    \end{split}
\end{equation}
where we set $\alpha=\gamma=6$, $\beta=2$, $\zeta=4$, and $\sigma=0.15$. 
The observation rate function $\lambda\colon \R\times\R\to\R^2_{>0}$ is defined by
\begin{equation}
    \lambda(c, r) \coloneqq \frac{5}{1 + \exp\Biggl(\begin{bmatrix} -5\, c \\ - c \, r\end{bmatrix} + 4 \Biggr) }.
\end{equation}
We simulate the model with Milstein's method and generate data at $t\in[0, 3]$ with 256 dicretisation steps. 
The discretisation results in an SSM $Z_{j+1} = f(Z_j, \epsilon_j)$, where $Z_j\in\R^2$ encodes $C(t_j)$ and $R(t_j)$, and $\epsilon_j\in\R^2$ stands for Brownian motion increment. 
We train a neural network to approximate the SDE dynamics by minimising the negative log-likelihood, estimated by a particle filter with resampling. 
The neural network is simply a three-layers fully connected network with residual connection, see Figure~\ref{fig:app-lokta-nn}, mimicking a discrete SDE solver. 
No prior knowledge of the dynamics is incorporated into the neural network. 
We train the neural network at a fixed 1,000 iterations by the Adam optimiser with learning rate 0.005.
At testing, we make 100 independent predictions from the learnt model and compute the root mean square error (RMSE) with respect to a reference trajectory generated using the true dynamics. 
We use 64 particles, and repeat the experiments 20 times. 

\begin{figure*}
    \begin{minipage}{.45\linewidth}
        \centering
        \wrapbox{.99}{%
            \tikzset{every picture/.style={line width=0.75pt}} 

\begin{tikzpicture}[x=0.75pt,y=0.75pt,yscale=-1,xscale=1]

\draw [line width=1.5]    (140,22) -- (140,50.8) ;
\draw [shift={(140,53.8)}, rotate = 270] [color={rgb, 255:red, 0; green, 0; blue, 0 }  ][line width=1.5]    (14.21,-4.28) .. controls (9.04,-1.82) and (4.3,-0.39) .. (0,0) .. controls (4.3,0.39) and (9.04,1.82) .. (14.21,4.28)   ;
\draw [line width=1.5]    (140,70) -- (140,98.8) ;
\draw [shift={(140,101.8)}, rotate = 270] [color={rgb, 255:red, 0; green, 0; blue, 0 }  ][line width=1.5]    (14.21,-4.28) .. controls (9.04,-1.82) and (4.3,-0.39) .. (0,0) .. controls (4.3,0.39) and (9.04,1.82) .. (14.21,4.28)   ;
\draw [line width=1.5]    (140,177.2) -- (140,206) ;
\draw [shift={(140,209)}, rotate = 270] [color={rgb, 255:red, 0; green, 0; blue, 0 }  ][line width=1.5]    (14.21,-4.28) .. controls (9.04,-1.82) and (4.3,-0.39) .. (0,0) .. controls (4.3,0.39) and (9.04,1.82) .. (14.21,4.28)   ;
\draw [line width=1.5]    (140,32) .. controls (-31.64,14.59) and (-36.46,180.33) .. (117.66,219.42) ;
\draw [shift={(120,220)}, rotate = 193.65] [color={rgb, 255:red, 0; green, 0; blue, 0 }  ][line width=1.5]    (14.21,-4.28) .. controls (9.04,-1.82) and (4.3,-0.39) .. (0,0) .. controls (4.3,0.39) and (9.04,1.82) .. (14.21,4.28)   ;
\draw [line width=1.5]    (140,123.2) -- (140,152) ;
\draw [shift={(140,155)}, rotate = 270] [color={rgb, 255:red, 0; green, 0; blue, 0 }  ][line width=1.5]    (14.21,-4.28) .. controls (9.04,-1.82) and (4.3,-0.39) .. (0,0) .. controls (4.3,0.39) and (9.04,1.82) .. (14.21,4.28)   ;

\draw (115,2.4) node [anchor=north west][inner sep=0.75pt]    {$Z_{j} ,\ \epsilon _{j}$};
\draw (38,54) node [anchor=north west][inner sep=0.75pt]   [align=left] {MLP (input 4, output 32), Swish};
\draw (51,161) node [anchor=north west][inner sep=0.75pt]   [align=left] {MLP (input 32, output 2)};
\draw (147,181.4) node [anchor=north west][inner sep=0.75pt]    {$\times \Delta $};
\draw (134,211.4) node [anchor=north west][inner sep=0.75pt]    {$+$};
\draw (117,231) node [anchor=north west][inner sep=0.75pt]   [align=left] {Output};
\draw (31,101) node [anchor=north west][inner sep=0.75pt]   [align=left] {MLP (input 32, output 32), Swish};
\draw (61,2) node [anchor=north west][inner sep=0.75pt]   [align=left] {Concat};

\end{tikzpicture}
        }
        \caption{The neural network used for learning the prey-predator model. 
        At the input, $Z_j$ and $\epsilon$ are concatenated to a four-dimensional vector. }
        \label{fig:app-lokta-nn}
    \end{minipage}
    \hfill
    \begin{minipage}{.41\linewidth}
        \captionof{table}{Prediction error (RMSE) and the number of successful runs (out of 20) for the prey-predator model experiment. }
        \label{tbl:app-lokta-others}
        \wrapbox{.99}{%
            \begin{tabular}{@{}lll@{}}
            \toprule
            Method                 & RMSE            & Number of success \\ \midrule
            OT ($\varepsilon=0.3$) & $2.35 \pm 0.95$ & 19                \\
            OT ($\varepsilon=0.5$) & $2.72 \pm 1.60$ & 20                \\
            OT ($\varepsilon=1.0$) & $3.27 \pm 2.71$ & 20                \\
            OT ($\varepsilon=1.5$) & $3.52 \pm 3.65$ & 20                \\
            Gumbel (0.1)           & $9.05 \pm 8.65$ & 15                \\
            Gumbel (0.3)           & $2.09 \pm 1.13$ & 20                \\
            Gumbel (0.5)           & $2.40 \pm 1.05$ & 20                \\
            Soft (0.5)             & $3.29 \pm 4.83$ & 17                \\
            Soft (0.7)             & $2.13 \pm 0.87$ & 17                \\
            Soft (0.9)             & $1.89 \pm 0.64$ & 16                \\
            Stopped                & $1.96 \pm 0.55$ & 20                \\ \bottomrule
            \end{tabular}
        }
    \end{minipage}
\end{figure*}

\begin{table}[]
    \centering
    \caption{Prediction RMSE (first table) of diffusion resampling for the prey-predator model, and the associated number of successful runs (second table) out of 20.  
    By comparing to Table~\ref{tbl:app-lokta-others} we see that all the entries here are largely better than the other resampling methods. 
    However, we also see that exponential integrators do not always provide stable gradient back-propagation, although they provide accurate forward estimation. }
    \label{tbl:app-lokta-diffusion}
    \wrapbox{.9}{%
        \begin{tabular}{@{}llllllll@{}}
        \toprule
        \multirow{2}{*}{Method} & \multicolumn{2}{l}{Euler--Maruyama} & \multicolumn{2}{l}{Lord--Rougemont}   & \multicolumn{2}{l}{Jentzen--Kloeden} & Tweedie         \\ \cmidrule(l){2-8} 
                                & ODE              & SDE              & ODE               & SDE               & ODE               & SDE              & SDE             \\ \cmidrule(l){2-8} 
        $T=1$, $K=4$            & $1.29 \pm 0.79$  & $1.30 \pm 0.98$  & NaN               & $17.01 \pm 3.16$  & $1.31 \pm 0.79$   & $2.22 \pm 0.72$  & $1.17 \pm 0.75$ \\
        $T=1$, $K=8$            & $1.31 \pm 0.71$  & $2.38 \pm 6.56$  & $16.62 \pm 6.42$  & $10.21 \pm 7.42$  & $1.29 \pm 0.70$   & $1.65 \pm 0.90$  & $1.23 \pm 0.71$ \\
        $T=1$, $K=16$           & $1.40 \pm 0.89$  & $1.95 \pm 1.10$  & $10.79 \pm 10.42$ & $5.78 \pm 6.15$   & $1.38 \pm 0.88$   & $5.20 \pm 4.62$  & $2.24 \pm 1.02$ \\
        $T=2$, $K=4$            & $1.21 \pm 0.81$  & $1.85 \pm 1.22$  & NaN               & NaN               & $1.21 \pm 0.76$   & NaN              & $1.18 \pm 0.76$ \\
        $T=2$, $K=8$            & $1.26 \pm 0.70$  & $1.61 \pm 0.79$  & NaN               & $13.88 \pm 10.80$ & $1.24 \pm 0.72$   & $2.05 \pm 1.19$  & $1.17 \pm 0.68$ \\
        $T=2$, $K=16$           & $1.29 \pm 0.70$  & $1.01 \pm 0.47$  & $9.52 \pm 1.97$   & $15.38 \pm 13.57$ & $1.30 \pm 0.75$   & $1.72 \pm 1.04$  & $1.29 \pm 0.87$ \\ \bottomrule
        \end{tabular}
    }
    \par\vspace*{1em}
    \wrapbox{.7}{%
        \begin{tabular}{@{}llllllll@{}}
        \toprule
        \multirow{2}{*}{Method} & \multicolumn{2}{l}{Euler--Maruyama} & \multicolumn{2}{l}{Lord--Rougemont} & \multicolumn{2}{l}{Jentzen--Kloeden} & Tweedie \\ \cmidrule(l){2-8} 
                                & ODE              & SDE              & ODE              & SDE              & ODE               & SDE              & SDE     \\ \cmidrule(l){2-8} 
        $T=1$, $K=4$            & 20               & 18               & 0                & 4                & 20                & 9                & 20      \\
        $T=1$, $K=8$            & 20               & 20               & 2                & 16               & 20                & 18               & 20      \\
        $T=1$, $K=16$           & 20               & 18               & 15               & 19               & 20                & 18               & 17      \\
        $T=2$, $K=4$            & 20               & 12               & 0                & 0                & 20                & 0                & 20      \\
        $T=2$, $K=8$            & 20               & 16               & 0                & 8                & 20                & 10               & 20      \\
        $T=2$, $K=16$           & 20               & 18               & 3                & 15               & 20                & 15               & 20      \\ \bottomrule
        \end{tabular}
    }
\end{table}

Results are shown in Tables~\ref{tbl:app-lokta-others} and~\ref{tbl:app-lokta-diffusion}. 
Since learning this model is challenging, most methods have experienced divergent runs. 
Notably, the diffusion resampling exhibits significantly less divergence, nearly a factor of two lower than the other methods. 
Moreover, the prediction error with diffusion resampling is substantially superior than the other methods; approximately two to five times better.
This table is also consistent with that of the LGSSM experiment that the exponential integrators may not propagate SDE gradients effectively, except for Tweedie. 

\section{Vision-based pendulum dynamics tracking}
\label{app:pendulum}

\paragraph{Experiment settings}
We assume that the pendulum dynamics evolve according to Equation~\eqref{eq:pendulum_dynamics}, with parameters $g=9.81$ and $l=0.4$. 
Using the known dynamics and the observation model specified in Equation~\eqref{eq:pendulum_dynamics} with observation noise $\sigma_{\mathrm{obs}} = 0.01$, we generate a single sequence of observational data over $t \in [0,4]$ using 256 discretisation steps (which is $j=0,1,\ldots, 256$ in discrete time) and starting from the initial state $Z_0 = \begin{bmatrix} \pi \, / \, 4 & 0 \end{bmatrix}$. 
During training, we jointly optimise the two neural networks, $f_\theta$ and $r_\phi$, parameterising the transition function and the decoder, respectively, by minimising the negative log-likelihood (NLL) estimated by the particle filter.  

For completeness, we consider two experiment settings: in the first setting we employ effective sample size (ESS)-based adaptive resampling under a tempered observation likelihood. 
Although the tempering modifies the original model, it is a common practice in deep learning for training complex models~\citep[see, e.g.,][]{Ng2025}. 
Here, the process noise covariance is fixed by $\Lambda_{\zeta} = \sigma_{\zeta}^2 \, I_2 $, where $\sigma_{\zeta}^2=0.01$. We use $\sigma_{\mathrm{train}} = 10 \, \sigma_{\mathrm{obs}}$ and scale the observation log-likelihood. In particular, we compute the log-likelihood as the mean over pixel dimensions (rather than the sum) and rescale the result by a factor $\kappa^{-1}$ (where $\kappa = 10^4$). 
We use a resampling threshold such that we resample if the ESS is less than $N=32$. We train the neural networks for 3000 iterations. The learnable parameters are updated using the Adam optimiser (learning rate $10^{-4}$) with global gradient clipping (maximum norm 1.0) and $\beta_2=0.999$ (other Adam defaults as in Optax). We use $N = 32$ particles to approximate the negative log-likelihood during training.

In the second setting, we consider a more challenging regime in which we force resampling at every filtering step after $j > 10$. Here, we use process noise with $\Lambda_{\zeta} = \sigma_{\zeta}^2 \, \diag{0,1}$ with $\sigma_{\zeta}^2=0.1$, and make an additive noise assumption in the learned transition model. We use $\sigma_{\mathrm{train}} = 5 \, \sigma_{\mathrm{obs}}$ and no other scaling of the log-likelihood compared to the true observation model. We train the neural networks for 1500 iterations. The learnable parameters are updated using the Adam optimiser (learning rate $2\times 10^{-4}$) with global gradient clipping (maximum norm 1.0) and $\beta_2=0.99$ (other Adam defaults as in Optax). Here, we use $N=16$ particles. The optimiser omits at most 10 consecutive non-finite parameter updates by leaving the current parameter values unchanged.

\paragraph{Results and evaluation}
We evaluate the learnt models by comparing the average SSIM and PSNR of image sequences obtained by unrolling the dynamics from $Z_0$ using the learnt transition function $f_\theta$, and then generating images by passing the resulting state trajectory through the learnt decoder $r_\phi$. All results corresponding to the first experiment setting are averaged over $5$ independent runs, and all results corresponding to the second experiment setting are averaged over $9$ independent runs.

For the first setting, results for different configurations of the diffusion resampler are presented in Table~\ref{tbl:app-pendulum-diffusion}.
Results for different configurations of the baselines are presented in Table~\ref{tbl:app-pendulum-others}. For the second, more challenging setting, in which resampling is invoked at nearly every filtering step, the corresponding results are presented in Tables~\ref{tbl:app-pendulum-diffusion2} and~\ref{tbl:app-pendulum-others2}, respectively.
Figure~\ref{fig:pendulum_overview1} shows the median loss evolution during training for the best model configuration (in terms of mean SSIM and PSNR) from each resampling class in each experiment setting. Figure~\ref{fig:ssim_psnr_setting2} shows mean prediction SSIM/PSNR for the second setting; see Figure~\ref{fig:pendulum_SSIM_PSNR} for a comparison to the corresponding results in the first setting. 

The difference in SSIM/PSNR between the two experiments indicate that the second setting is more challenging for all resampling methods. In the first setting, 
performance is competitive, and diffusion resampling provides a stable end-to-end optimisation and effective integration into the high-dimensional learning pipeline. Notably, resampling is not always necessary for convergence in end-to-end training using particle filters, and prior work has reported cases where omitting resampling can be more effective ~\citep[see, e.g.,][]{Karkus2018}. In our experiments, we consider end-to-end training \emph{with} resampling, and compare resamplers while holding the rest of the system fixed. Resampling itself can interact with the optimisation process in non-trivial ways, and we thus primarily evaluate the comparison by end-to-end performance. 
However, tempering the observation likelihood in the first setting reduces weight concentration: the weights become more uniform and ESS remains high. 
With this ESS-based adaptive resampling, the resampling can therefore be triggered infrequently. 
As such, to further stress test differentiable resampling in this context, we consider the second setting with more concentrated weights and resampling at nearly every filtering step.

As we have seen, even in this challenging setting, diffusion resampling remains competitive among the baselines. Figure~\ref{fig:pendulum_overview_grid4} shows a qualitative comparison of a top performing model in each resampling class in this setting. Here, the diffusion resampler is among the three top-performing baselines. While soft resampling occasionally shows better performance in terms of final metrics compared to diffusion resampling, see Figure~\ref{fig:ssim_psnr_setting2}, we note that it also exhibits significantly worse weight degeneracy compared to diffusion resampling and the other baselines. In addition, it is by construction not fully differentiable (see Table \ref{tbl:app-compare-all}). We also note that while OT produces the top performing model here, it exhibits high variability and is overall unstable in this setting, with only a few converging training runs.

\paragraph{Identifiability of the latent space}
Because both the transition model $f_\theta$ and the decoder $r_\phi$ are learnt only through the image observation likelihood, the latent state is typically identifiable only up to (possibly nonlinear) invertible transformations, and there is no unique correspondence between the learnt states and the true physical coordinates without special constructions~\citep[see, e.g.,][for a similar issue]{greydanus2019hamiltonian}.
For this reason, we do not directly evaluate the pendulum dynamics in state space, but instead assess model quality in the observation space via image reconstruction metrics. A similar evaluation strategy is common in recent work on latent SDEs for high-dimensional sequence data. 
For example, both~\citet{Bartosh2025sde} and \citet{Course2023amortized} primarily report qualitative comparisons of generated image sequences in comparable settings, rather than quantitative metrics in the latent state space, reflecting the view that high-quality reconstruction is indicative of a meaningful latent representation. 
Nevertheless, we find that sometimes the learnt dynamics closely match the true pendulum dynamics up to a linear transformation of the latent coordinates. Figure~\ref{fig:pendulum_overview2}
illustrates two such examples with the latent dynamics learnt using diffusion resampling, where a simple linear transformation of the latent trajectory yields a close match to the ground-truth pendulum state (coefficient of determination $R^2 \approx 0.98$ (left) and $R^2 \approx 0.92$ (right), respectively). 
This indicates that the models have learnt meaningful latent representations of the underlying pendulum dynamics, albeit in a transformed coordinate system.

\begin{table}
  \caption{Prediction quality in terms of structural similarity index (SSIM) and PSNR (higher the better) for the pendulum dynamics
  tracking experiment, with respect to models trained using diffusion resampling in the first experiment setting. 
  Each result is presented in terms of mean over 5 independent runs, together with the corresponding standard deviation, recorded after 3000 training iterations. 
  For two configurations only 1 out of 5 (*) and 3 out of 5 (**) runs, respectively, were non-divergent.}
  \label{tbl:app-pendulum-diffusion}
  \centering
  \wrapbox{.99}{%
    \begin{tabular}{@{}lllllllllll@{}}
      \toprule
      \multirow{3}{*}{Method}
        & \multicolumn{4}{l}{Euler--Maruyama}
        & \multicolumn{4}{l}{Jentzen--Kloeden}
        & \multicolumn{2}{l}{Tweedie} \\ \cmidrule(l){2-11}
        & \multicolumn{2}{l}{SDE}
        & \multicolumn{2}{l}{ODE}
        & \multicolumn{2}{l}{SDE}
        & \multicolumn{2}{l}{ODE}
        & \multicolumn{2}{l}{SDE} \\ \cmidrule(l){2-11}
        & SSIM & PSNR
        & SSIM & PSNR
        & SSIM & PSNR
        & SSIM & PSNR
        & SSIM & PSNR \\ \cmidrule(l){2-11}
      $T=1$, $K=4$
        & $0.866 \pm 0.039$ & $21.4 \pm 1.23$
        & $0.859 \pm 0.028$ & $21.1 \pm 1.31$
        & $0.508^* \pm \text{--}$ & $16.7^* \pm \text{--}$
        & $0.858 \pm 0.031$ & $21.0 \pm 1.20$
        & $0.823 \pm 0.077$ & $20.5 \pm 1.41$ \\
      $T=1$, $K=8$
        & $0.861 \pm 0.035$ & $21.1 \pm 1.10$
        & $0.859 \pm 0.031$ & $21.0 \pm 1.25$
        & $0.763^{**} \pm 0.103$ & $19.3^{**} \pm 0.84$
        & $0.865 \pm 0.036$ & $21.5 \pm 1.93$
        & $0.817 \pm 0.091$ & $20.4 \pm 1.47$ \\
      $T=1$, $K=16$
        & $0.762 \pm 0.128$ & $19.3 \pm 1.32$
        & $0.826 \pm 0.082$ & $20.6 \pm 1.55$
        & $0.769 \pm 0.086$ & $19.6 \pm 1.02$
        & $0.855 \pm 0.033$ & $20.8 \pm 1.35$
        & $0.736 \pm 0.134$ & $19.9 \pm 1.66$ \\ \bottomrule
    \end{tabular}
  } 
\end{table}

\begin{table}
  \caption{Prediction quality (SSIM, PSNR) for the pendulum dynamics
  tracking experiment, with respect to models trained using OT, Gumbel and Soft
  resampling. Results correspond to the first experiment setting.}
  \label{tbl:app-pendulum-others}
  \centering
  \wrapbox{.35}{%
    \begin{tabular}{@{}lll@{}}
      \toprule
      Method                 & SSIM  & PSNR  \\ \midrule
      OT ($\varepsilon=0.5$) & $0.861 \pm 0.033$ & $21.1 \pm 1.29$ \\
      OT ($\varepsilon=1.0$) & $0.813 \pm 0.109$ & $20.6 \pm 1.91$ \\
      OT ($\varepsilon=1.5$) & $0.849 \pm 0.031$ & $20.8 \pm 1.16$ \\
      Gumbel (0.1)           & $0.846 \pm 0.034$ & $20.6 \pm 1.22$ \\
      Gumbel (0.3)           & $0.859 \pm 0.029$ & $21.1 \pm 1.25$ \\
      Soft (0.7)             & $0.860 \pm 0.031$ & $21.0 \pm 1.27$ \\
      Soft (0.9)             & $0.814 \pm 0.105$ & $20.5 \pm 1.78$ \\ \bottomrule
    \end{tabular}
  }
\end{table}

\begin{table}
  \caption{Prediction quality in terms of structural similarity index (SSIM) and PSNR (higher the better) for the pendulum dynamics
  tracking experiment, with respect to models trained using diffusion resampling in the second experiment setting. 
  Each result is presented in terms of mean over 9 independent runs, together with the corresponding standard deviation, recorded after 1500 training iterations. 
  For four configurations only 8 out of 9 (*) runs, respectively, were non-divergent.}
  \label{tbl:app-pendulum-diffusion2}
  \centering
  \wrapbox{.99}{%
    \begin{tabular}{@{}lllllllllll@{}}
      \toprule
      \multirow{3}{*}{Method}
        & \multicolumn{4}{l}{Euler--Maruyama}
        & \multicolumn{4}{l}{Jentzen--Kloeden}
        & \multicolumn{2}{l}{Tweedie} \\ \cmidrule(l){2-11}
        & \multicolumn{2}{l}{SDE}
        & \multicolumn{2}{l}{ODE}
        & \multicolumn{2}{l}{SDE}
        & \multicolumn{2}{l}{ODE}
        & \multicolumn{2}{l}{SDE} \\ \cmidrule(l){2-11}
        & SSIM & PSNR
        & SSIM & PSNR
        & SSIM & PSNR
        & SSIM & PSNR
        & SSIM & PSNR \\ \cmidrule(l){2-11}
      $T=1$, $K=4$
        & $0.671^* \pm 0.082$ & $17.2^* \pm 0.666$ 
        & $0.607 \pm 0.091$ & $17.2 \pm 0.834$ 
        & $0.517^* \pm 0.134$ & $16.9^* \pm 0.528$ 
        & $0.507 \pm 0.048$ & $16.5 \pm 0.443$ 
        & $0.549^* \pm 0.081$ & $16.7^* \pm 0.666$ \\ 
      $T=1$, $K=8$
        & $0.576^* \pm 0.080$ & $16.6^* \pm 0.542$ 
        & $0.511 \pm 0.070$ & $16.7 \pm 0.554$ 
        & $0.599 \pm 0.101$ & $17.1 \pm 0.579$ 
        & $0.481 \pm 0.074$ & $16.5 \pm 0.253$ 
        & $0.503 \pm 0.067$ & $16.5 \pm 0.32$ \\ 
      $T=1$, $K=16$
        & $0.576 \pm 0.082$ & $16.6 \pm 1.19$ 
        & $0.280 \pm 0.112$ & $16.1 \pm 0.235$ 
        & $0.557 \pm 0.101$ & $16.9 \pm 0.744$ 
        & $0.239 \pm 0.082$ & $16.0 \pm 0.202$ 
        & $0.227 \pm 0.074$ & $16.1 \pm 0.197$ \\ \bottomrule 
    \end{tabular}
  } 
\end{table}

\begin{table}
  \caption{Prediction quality (SSIM, PSNR) for the pendulum dynamics
  tracking experiment, with respect to models trained using OT, Gumbel and Soft
  resampling. Results correspond to the second experiment setting. For OT configurations, only 3 (**) and 4(*) runs were non-divergent.}
  \label{tbl:app-pendulum-others2}
  \centering
  \wrapbox{.35}{%
    \begin{tabular}{@{}lll@{}}
      \toprule
      Method                 & SSIM  & PSNR  \\ \midrule
          OT ($\varepsilon=0.5$) & $0.662^{*}\pm 0.141$ & $17.7^{*} \pm 1.782$ \\
      OT ($\varepsilon=1.0$) & $0.575^{**}\pm 0.303$ & $17.5^{**} \pm 1.924$ \\
      OT ($\varepsilon=1.5$) & $0.708^{**}\pm 0.159$ & $18.2^{**} \pm 2.73$ \\
      Gumbel (0.1)           & $0.509 \pm 0.060$ & $16.7 \pm 0.400$ \\
      Gumbel (0.3)           & $0.593 \pm 0.083$ & $16.8 \pm 0.508$ \\
      Soft (0.7)             & $0.720 \pm 0.107$ & $17.7 \pm 0.707$ \\
      Soft (0.9)             & $0.704 \pm 0.114$ & $17.3 \pm 0.872$ \\ \bottomrule
    \end{tabular}
  }
\end{table}

\paragraph{Network architectures}
Both the neural network modelling the latent dynamics, $f_\theta$, and the neural network modelling the decoder, $r_\phi$, operate on a feature vector $\mathbf{h}(Z_j) = \begin{bmatrix} \sin(Z_j^{(1)}) & \cos(Z_j^{(1)}) & Z_j^{(2)} \, / \, 10 \end{bmatrix}$, which embeds the angle and scales the velocity. 
While the system is initialised with physical coordinates $Z_0$, both networks are unconstrained and unregularised for all $t>0$. 
As discussed above, this means that the system is free to learn a latent manifold sufficient for image reconstruction, even if the latent representation is a non-linear transformation of the true physical coordinate system. 
We employ a SIREN \citep{Sitzmann2019siren} architecture to model the latent dynamics. 
The network takes the concatenated input $\begin{bmatrix} \mathbf{h}(Z_j) & \zeta_j \end{bmatrix} \in \R^5$ and outputs the state derivative. 
The update rule in the first setting is $Z_{j+1} = Z_j + \mathrm{Net}(\begin{bmatrix} \mathbf{h}(Z_j) & \zeta_j \end{bmatrix}) \, \Delta_j$. In the second setting, where we invoke resampling at nearly all filtering steps, we make an additional additive noise assumption, omit concatenating the input and let $Z_{j+1} = Z_j + \mathrm{Net}(\mathbf{h}(Z_j))\,\Delta_j + \zeta_j$. Without this assumption, the neural network is free to learn to neglect the noise, and in the first setting, we indeed found that the learnt dynamics network often had this issue. This can happen in end-to-end learning since minimising the optimisation objective does not necessarily require the network to optimise for a good filtering distribution, and reducing the randomness might, for instance, help stabilise optimisation.

In both settings, the network consists of 3 hidden layers of 256 units with sine activations ($\omega_0=8.0$ in the first layer). The decoder maps $\mathbf{h}(Z_j)$ to observations and is identical in both settings. It first processes the input via two linear layers (16 and 256 units), reshaping the output into a spatial feature map of size $4 \times 4 \times 16$. 
This is followed by three transposed convolution layers (kernel size $3 \times 3$, stride 2) with output channels of 32, 16, and 1, respectively, to upsample to the final $32 \times 32$ image. 
ReLU activations are used for all hidden layers, while the output layer is linear.

\begin{figure}[t!]
    \centering
    \begin{minipage}{0.49\linewidth}
        \centering
        \includegraphics[width=\linewidth]{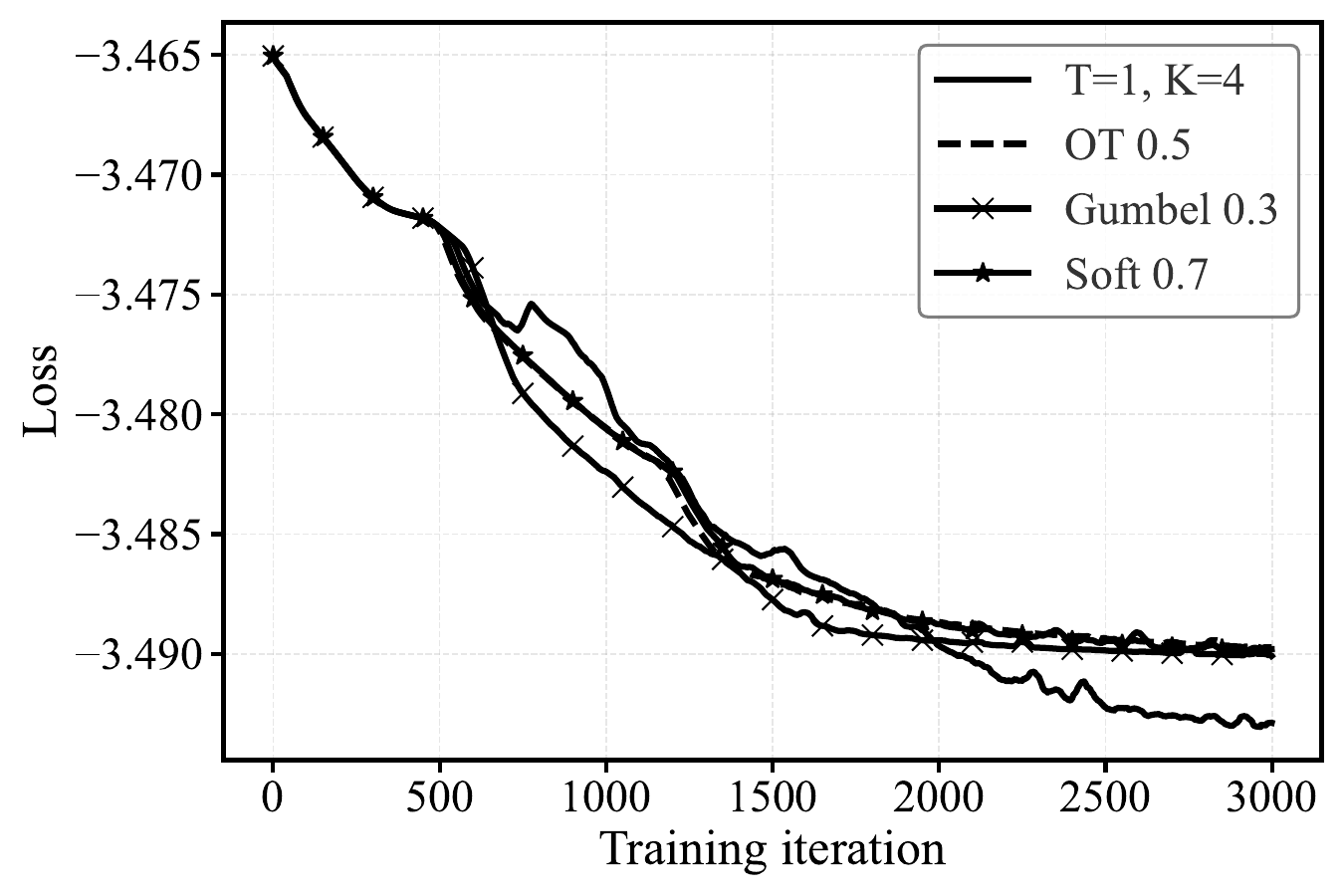}
        \label{fig:pendulum_top_five_loss_1}
    \end{minipage}
    \hfill
    \begin{minipage}{0.49\linewidth}
        \centering
        \includegraphics[width=\linewidth]{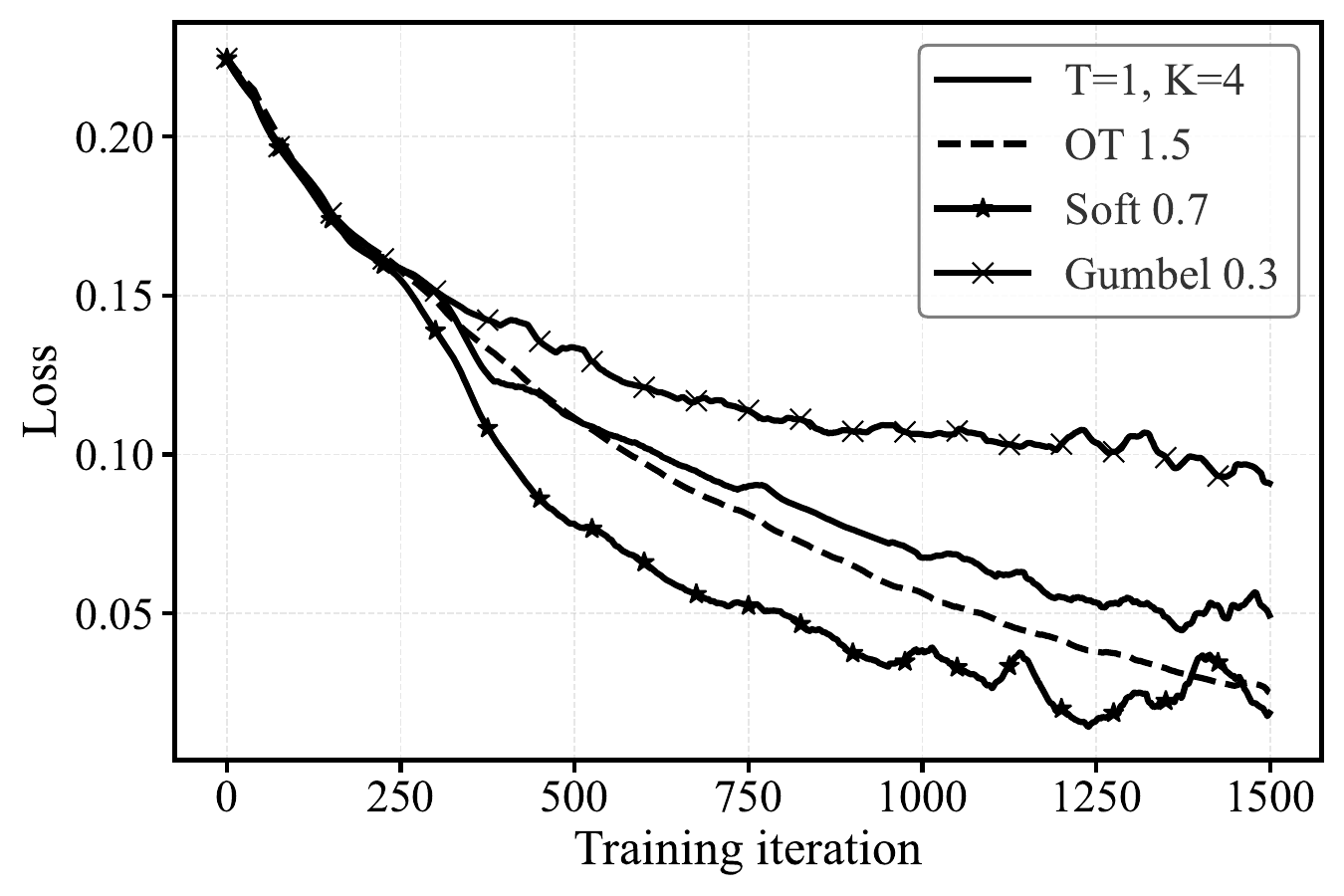}
        \label{fig:pendulum_top_five_loss_2}
    \end{minipage}
    \caption{(Left) Median loss during the training process for the top performing configuration in each resampling class in the first setting. We note that diffusion resampling achieves the lowest loss. 
    (Right) Median loss during the training process for the top performing configuration in each resampling class in the more challenging setting. We note that for OT, the reported median loss is based on the few successful runs and is therefore not statistically meaningful. While Soft achieves the lowest median loss (see the text for discussion), diffusion resampling remains competitive and yields low, stable loss, whereas Gumbel performs markedly worse than all baselines.}
    \label{fig:pendulum_overview1}
\end{figure}

\begin{figure}[t!]
    \centering
    \includegraphics[width=.5\linewidth]{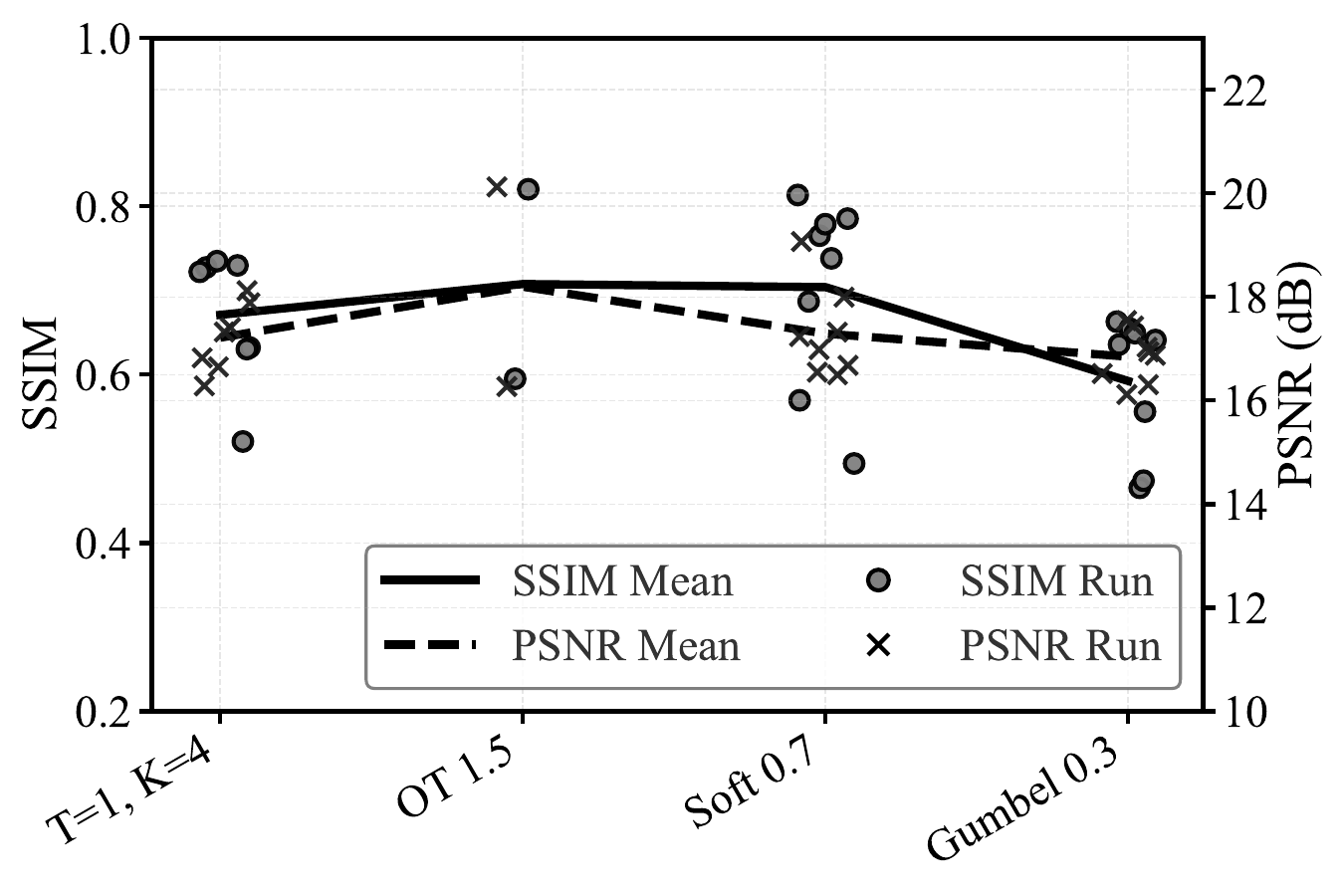}
    \caption{Mean prediction SSIM and PSNR for the best (by mean) model configuration of each resampler in the more challenging setting. 
    Individual runs are shown as scatter points. 
    We find that although Soft can occasionally give better results than diffusion resampling, its results exhibit large variability. 
    We also observe that Gumbel performs worse than diffusion resampling. We further observe that OT gets top SSIM and PSNR here, but note that it exhibits large variability and is not stable in this task (most OT runs diverged during training).
    }
    \label{fig:ssim_psnr_setting2}
\end{figure}

\begin{figure}[t!]
    \centering
    \begin{minipage}{0.49\linewidth}
        \centering
        \includegraphics[width=\linewidth]{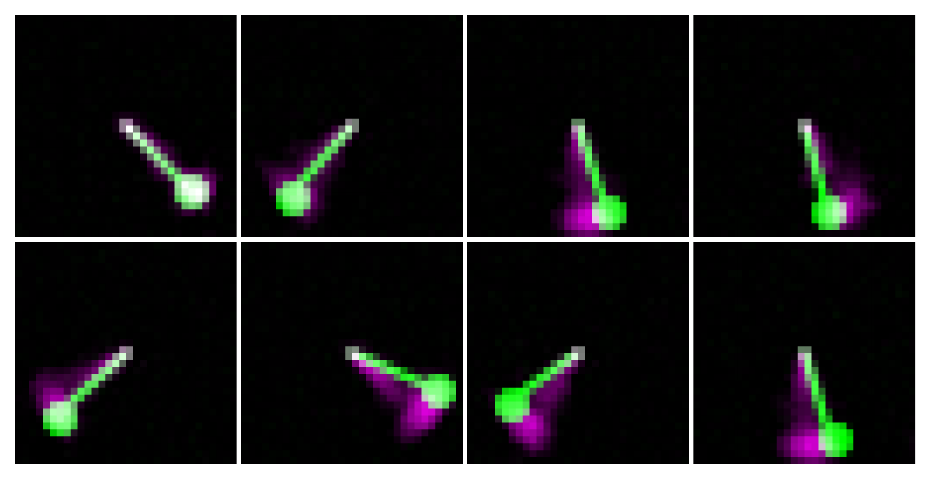}\\[-0.3em]
        \small (a)
    \end{minipage}
    \hfill
    \begin{minipage}{0.49\linewidth}
        \centering
        \includegraphics[width=\linewidth]{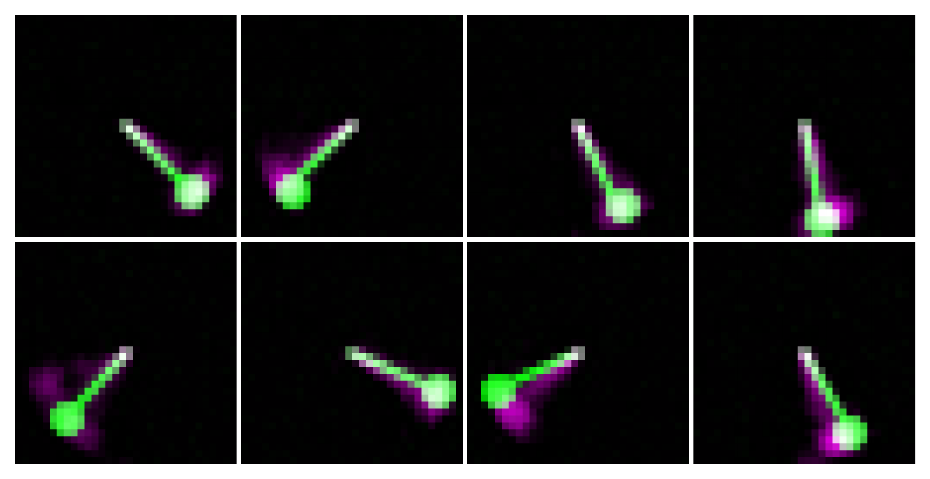}\\[-0.3em]
        \small (b)
    \end{minipage}

    \vspace{0.6em}

    \begin{minipage}{0.49\linewidth}
        \centering
        \includegraphics[width=\linewidth]{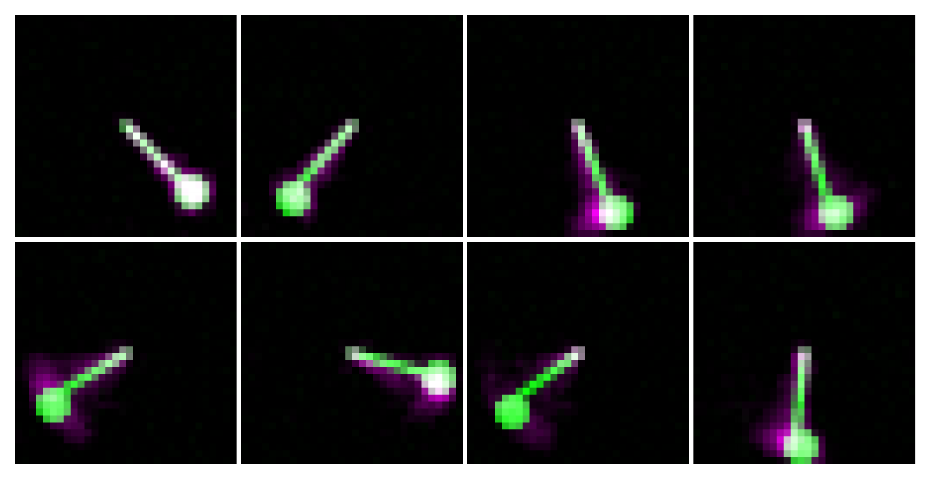}\\[-0.3em]
        \small (c)
    \end{minipage}
    \hfill
    \begin{minipage}{0.49\linewidth}
        \centering
        \includegraphics[width=\linewidth]{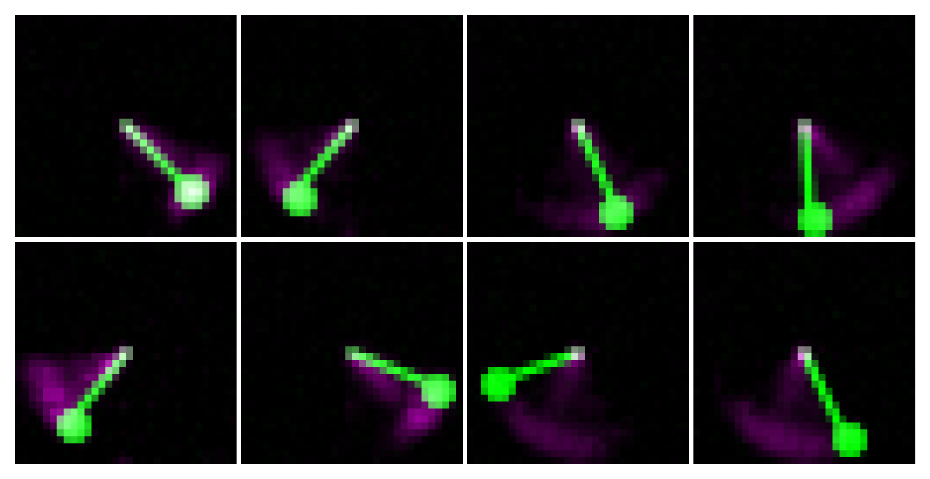}\\[-0.3em]
        \small (d)
    \end{minipage}

    \caption{Qualitative comparisons of the learnt pendulum dynamics trained end-to-end using a differentiable resampler in the SMC training loop. The results correspond to the second, more challenging experiment setting using resampling at nearly all filtering steps. The ground truth (green) is overlaid with model predictions (purple). White pixels indicate perfect alignment, while coloured regions highlight positional discrepancies (e.g., phase lag). Snapshots are shown at eight evenly spaced time points over the 4 second simulation (read from left to right and top to bottom). Each panel shows a top result in terms of mean SSIM for each differentiable resampling method in our comparison: a) Diffusion (mean SSIM/PSNR $0.761/17.0$, b) Soft ($0.816/19.2$), c) OT ($0.820/20.1$), d) Gumbel ($0.663/16.9$).}
    \label{fig:pendulum_overview_grid4}
\end{figure}

\begin{figure}[t!]
    \centering
    \begin{minipage}{0.49\linewidth}
        \centering
        \includegraphics[width=\linewidth]{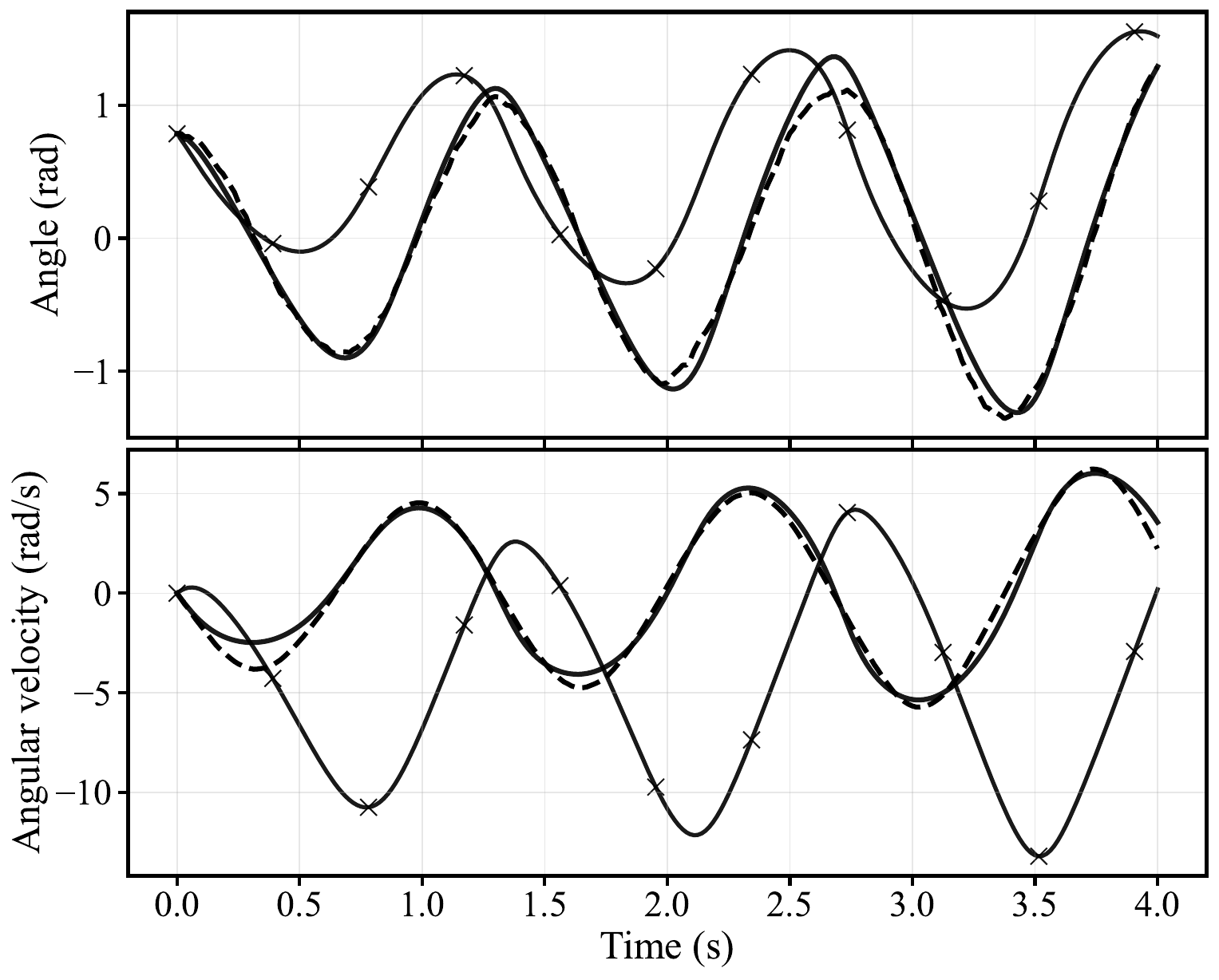}
        \label{fig:pendulum_top_five_loss}
    \end{minipage}
    \hfill
    \begin{minipage}{0.49\linewidth}
        \centering
        \includegraphics[width=\linewidth]{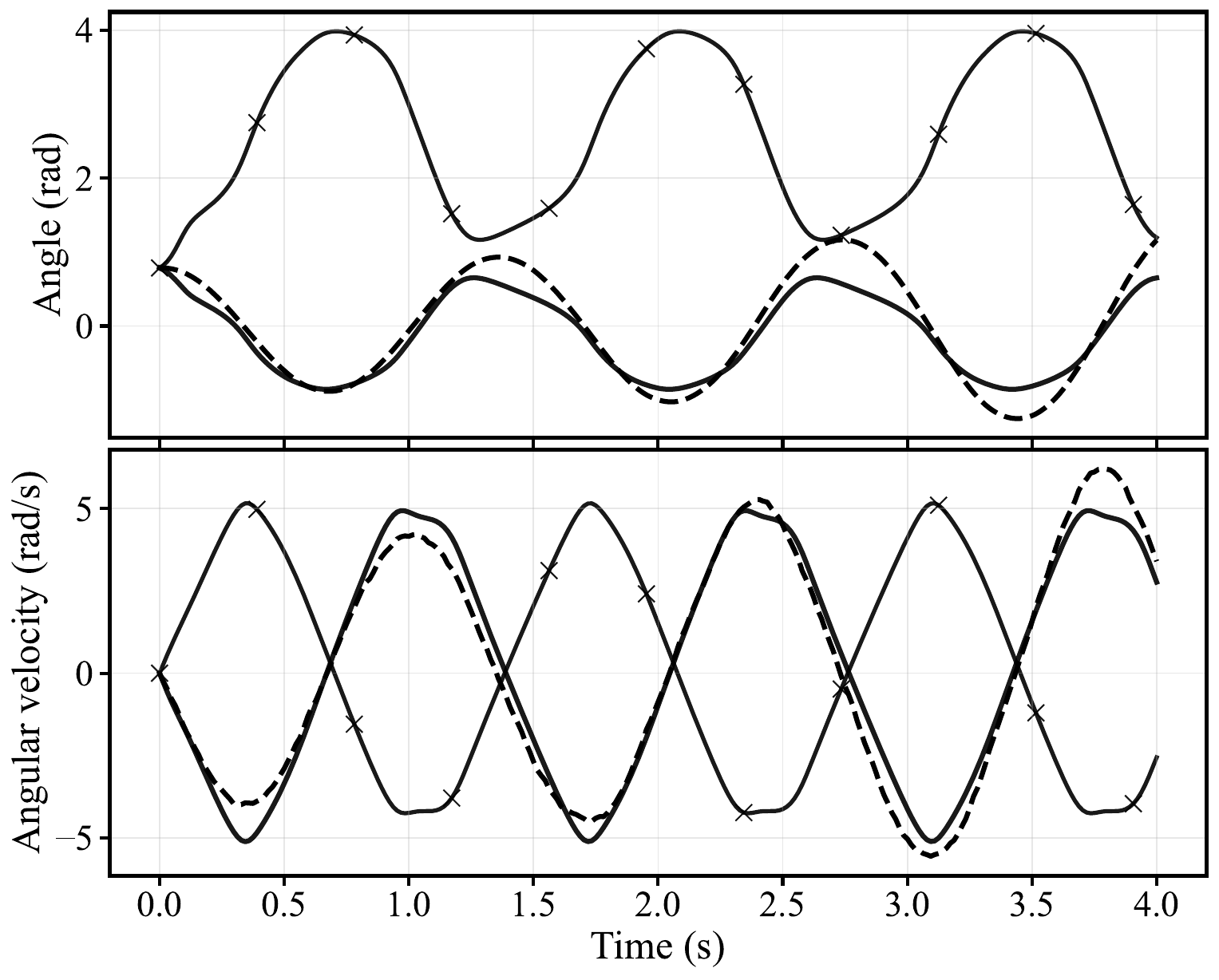}
        \label{fig:pendulum_dynamics}
    \end{minipage}
    \caption{(Left) Ground-truth pendulum state trajectories (dashed) and predictions in the original neural (solid with crosses) and transformed (solid) coordinate systems, for a top-performing model trained with diffusion resampling ($T=1$, $K=4$) in the first setting. 
    The dynamics model is trained until convergence, and correspond to the image predictions presented in Figure \ref{fig:pendulum_grid}.
    (Right) Same as the previous image, but for a model optimised in the more challenging setting. The corresponding image predictions are presented in Figure~\ref{fig:pendulum_overview_grid4}a.}
    \label{fig:pendulum_overview2}
\end{figure}

\section{Bayesian neural network training}
\label{app:pbnn}
We apply SMC for a larger-scale problem: training a (partial) Bayesian neural network~\citep[pBNN]{ZhaoPBNN2024, Sharma2023}. 
There are two steps moving from training a classical BNN to training a pBNN with SMC sampler. 
First, pBNNs only have priors on a (small) subset of the neural network parameters. 
This results in a latent-variable model to learn from the data, where computing the posterior distribution is easier than that of the full BNN. 
Let us denote the pBNN by $x \mapsto f_{\theta, z}(x)$, where $\theta$ and $z$ stand for deterministic and stochastic parameters, respectively.
Second, the prior construction is no longer static or explicit (e.g., $z$ being a Gaussian). 
Instead, $z$ is assumed to follow a Markovian dynamics motivated by~\citet{ZhaoPBNN2024, Chang2022bnn, Freitas2000} so that data are modelled as independent conditionally on a latent process. 
Specifically, we apply the pBNN setting and construct an SSM
\begin{equation}
    \begin{split}
        p(z_j \cond z_{j-1}) &= \mathrm{N}(z_j; \rho \, z_{j-1}, 1 - \rho^2), \\
        p_\theta(D_j \cond z_j) &= \mathrm{Softmax}(f_{\theta, z_j}; D_j),
    \end{split}
\end{equation}
for CIFAR10 classification, where $D_j$ is a batch data at surrogate time $j$ (training step), and we choose $\rho=0.99$. 

We choose a ResNet18 neural network, and set the first two convolution layers be stochastic (i.e., the neural network parameters in this part is the latent variable $z$). 
This results in latent dimension $d_z = 1,856$ and parameter dimension $d_\theta = 11,172,106$. 
We choose the number of particles to be 8, resampling threshold 0.5, and apply an Adam optimiser train for 200 epochs.
The diffusion resampling parameter is $T = 1$ and $K=4$ using the Jentzen--Kloeden SDE integrator. 
The results of the classification evaluated using accuracy and F1 score is shown in Table~\ref{tbl:cifar10-results}.

\begin{table}[t!]
    \centering
    \caption{CIFAR10 classification with different resampling methods. Results are averaged over 5 independent trainings.}
    \label{tbl:cifar10-results}
    \wrapbox{.7}{%
        \begin{tabular}{@{}llllll@{}}
        \toprule
                 & Diffusion                 & OT ($\varepsilon=1$) & Gumbel (0.3)     & Soft (0.7)       & Stopped          \\ \midrule
        Accuracy & $\mathbf{87.81} \pm 0.64$ & $83.28 \pm 0.63$   & $86.92 \pm 0.58$ & $86.57 \pm 0.58$ & $86.28 \pm 0.57$ \\
        F1 score & $\mathbf{87.76} \pm 0.63$ & $83.18 \pm 0.70$   & $86.89 \pm 0.57$ & $86.54 \pm 0.57$ & $86.21 \pm 0.56$ \\ \bottomrule
        \end{tabular}
    }
\end{table}

\section{Weather forecast}
\label{app:weather}
\paragraph{Experiment settings}
As another large-scale experiment, we consider a high-dimensional weather forecasting task to further demonstrate that the diffusion resampler scales to high-dimensional state space models. Given partial noisy observations, we track the evolution of the 850 hPa atmospheric temperature field over time. We use SMC to enable learning the parameters $\theta$ of an SSM with atmosphere dynamics model $p_\theta(z_k \cond z_{k-1})$ and observation model $p(y_k \cond z_k)$. Here, the latent state $z_k \in \R^{2048}$ represents the global temperature field at $5.625^\circ$ resolution, with the resulting latent dimension $d=2048$ ($32 \times 64$ spatial grid). 
The partial observation $y_k$ consists of $80$\% of pixels selected uniformly at random at each time step, corrupted by Gaussian noise with standard deviation $\sigma_y = 0.01$. The observation model is $p(y_k \cond z_k) = \mathcal{N}(y_k \cond \mathcal{M}_kz_k, \sigma_y^2 I)$, where $\mathcal{M}_k$ denotes the random mask operator for the observed pixels at time step $k$.
A UNet parametrises the transition density $p_\theta(z_k \cond z_{k-1})$, and is learnt using the negative log-likelihood estimate provided by a particle filter with resampling at every step.

The models are trained using ERA5 data from the WeatherBench benchmark dataset \citep{Rasp2020WeatherBench}. We present results for two different settings: a full year setting, where the model is tasked with iterative 4-day forecasts over a full year, and a single-sequence setting, where the model is tasked with iterative 8-day forecasts in January. 

\paragraph{Training details}
We use data from the years 2010-2018 for training (excluding 2016), with 2014 held out for evaluation. The temperature fields are normalised to $[0, 1]$ using global min-max normalisation computed across all years. The UNet takes the current state $z_{k-1}$ and process noise concatenated along the channel dimension. The output is given by $z_k = \mathrm{Sigmoid}\bigl(z_{k-1} + \delta_\theta(z_{k-1}, q_k)\bigr)$, where $\delta_\theta$ is the UNet output, and $q_k \sim \mathcal{N}(0, \sigma_q^2 I)$ with $\sigma_q^2 = 0.1$. We train the model for 3000 iterations using the Adam optimiser with learning rate $2 \times 10^{-4}$ and gradient clipping with global norm threshold $1.0$. The optimiser omits at most 10 consecutive non-finite parameter updates by leaving the current parameter values unchanged. We use $N=16$ particles and trigger resampling at each filtering step. By default, the diffusion resampler jitters the diagonal of the Gaussian reference by $10^{-5}$ for numerical stability, and the log-likelihood at each time step is normalised by the number of observed pixels.

\paragraph{Results and evaluation}
The learnt dynamics models are evaluated on unobserved (masked) pixels from the training sequences and on full sequences from the held-out year. We use the mean square error (MSE) averaged over time steps and 16 individual rollout samples to evaluate the performance. All MSE values are reported at scale $\times 10^{-3}$, averaged over 5 independent training runs. Table \ref{tbl:weather} shows results for the full year setting, where training sequences of length 16 are randomly sampled from the full year of 6-hourly data across all 7 training years. The model is evaluated on 5 evenly spaced windows from the held-out year, covering all seasons. Table \ref{tbl:weather2} shows results for the single-sequence setting, where the training data consists of fixed sequences of length 32 covering the first 8 days of January for each year in the training data. The model is evaluated on the corresponding January sequence from the held-out year.

\begin{table}
  \caption{Prediction quality in terms of MSE ($\times 10^{-3}$, lower is better) for the weather forecasting experiment in the full year setting.
  Results are presented as mean $\pm$ standard deviation over 5 independent runs recorded after 3000 training iterations.
  Train MSE is averaged over 5 evenly spaced windows from each of the 7 training years; eval MSE is computed on 5 evenly spaced windows from the held-out year.}
  \label{tbl:weather}
  \centering
  \wrapbox{.65}{%
    \begin{tabular}{@{}lll@{}}
      \toprule
      Method & Train MSE (masked) & Eval MSE (full) \\ \midrule
      Diffusion, Euler--Maruyama ($T=1,K=4$,ODE) & $2.2 \pm 0.7$ & $2.1 \pm 0.8$ \\
      Diffusion, Euler--Maruyama ($T=1,K=4$,SDE) & $2.5 \pm 0.8$ & $2.4 \pm 0.7$ \\
      Diffusion, Euler--Maruyama ($T=1,K=8$,ODE) & $2.1 \pm 0.5$ & $2.0 \pm 0.4$ \\
      Diffusion, Jentzen--Kloeden ($T=1,K=8$,ODE) & $2.0 \pm 0.5$ & $1.9 \pm 0.4$ \\
      Soft ($0.7$)                       & $1.7 \pm 0.1$ & $1.7 \pm 0.1$ \\
      Soft ($0.9$)                       & $1.6 \pm 0.1$ & $1.7 \pm 0.2$ \\
      Gumbel ($0.1$)                       & $4.6 \pm 2.0$ & $4.7 \pm 2.3$ \\
      Gumbel ($0.3$)                       & $1.7 \pm 0.1$ & $1.7 \pm 0.1$ \\
      OT ($\varepsilon=0.5$)                    & $1.9 \pm 0.1$ & $1.8 \pm 0.1$ \\
      OT ($\varepsilon=1.0$)                    & $1.8 \pm 0.2$ & $1.8 \pm 0.1$ \\
      OT ($\varepsilon=1.5$)                    & $2.4 \pm 1.2$ & $2.4 \pm 1.4$ \\ \bottomrule
    \end{tabular}
  }
\end{table}

\begin{table}
  \caption{Prediction quality in terms of MSE ($\times 10^{-3}$, lower is better) for the weather forecasting experiment in the single-sequence setting.
  Results are presented as mean $\pm$ standard deviation over 5 independent runs recorded after 3000 training iterations.
  Train MSE is averaged over all 7 training years; eval MSE is computed on the January sequence from the held-out year.}
  \label{tbl:weather2}
  \centering
  \wrapbox{.65}{%
    \begin{tabular}{@{}lll@{}}
      \toprule
      Method & Train MSE (masked) & Eval MSE (full) \\ \midrule
      Diffusion, Euler--Maruyama ($T=1,K=4$,ODE) & $0.6 \pm 0.1$ & $1.9 \pm 0.1$ \\
      Diffusion, Euler--Maruyama ($T=1,K=4$,SDE) & $0.6 \pm 0.1$ & $2.0 \pm 0.1$ \\
      Diffusion, Euler--Maruyama ($T=1,K=8$,ODE) & $0.5 \pm 0.0$ & $2.0 \pm 0.0$ \\
      Diffusion, Jentzen--Kloeden ($T=1,K=8$,ODE) & $0.5 \pm 0.0$ & $2.0 \pm 0.1$ \\
      Soft ($0.7$)                        & $0.5 \pm 0.1$ & $2.1 \pm 0.1$ \\
      Soft ($0.9$)                        & $0.5 \pm 0.1$ & $2.0 \pm 0.1$ \\
      Gumbel ($0.1$)                        & $0.6 \pm 0.1$ & $2.0 \pm 0.0$ \\
      Gumbel ($0.3$)                        & $0.6 \pm 0.1$ & $1.9 \pm 0.0$ \\
      OT ($\varepsilon=0.5$)                     & $0.6 \pm 0.1$ & $1.9 \pm 0.1$ \\
      OT ($\varepsilon=1.0$)                     & $0.6 \pm 0.0$ & $1.9 \pm 0.1$ \\
      OT ($\varepsilon=1.5$)                     & $0.6 \pm 0.1$ & $2.0 \pm 0.1$ \\ \bottomrule
    \end{tabular}
  }
\end{table}

Our results show that diffusion resampling can be used inside the SMC learning pipeline to learn atmosphere dynamics in this complex, high-dimensional setting. Tables~\ref{tbl:weather} and Table~\ref{tbl:weather2} show that almost all configurations achieve relatively low MSE on both unobserved pixels in sequences drawn from the training data, and on unseen sequences drawn from the held-out evaluation data. The worst results are obtained by Gumbel ($\tau=0.1$) in the full year setting.
In this setting, which contains data from the annual temperature variations, all resamplers struggle to learn fine-grained details in the current setup. Soft resampling achieves the lowest eval MSE with low variance, while the best diffusion configuration (Jentzen--Kloeden, $K=8$, ODE) achieves competitive performance. We note that these results rely on only 5 independent training runs for each model configuration. The standard deviation is not negligible in the comparison, and some configurations in the full year setting (including some diffusion configurations) exhibit relatively high variance.

Figure~\ref{fig:weather_forecast} shows examples of predicted forecasts using a model trained with diffusion resampling. These results correspond to the second experiment setting, where models are trained on a small dataset from a limited season (early January). Though Table~\ref{tbl:weather2} may indicate some overfitting, the evaluation forecasts show that some generalisable dynamics have been learned, and that the forecasts include fine-scale features.

\begin{figure}[h]
    \centering
    \includegraphics[width=\textwidth]{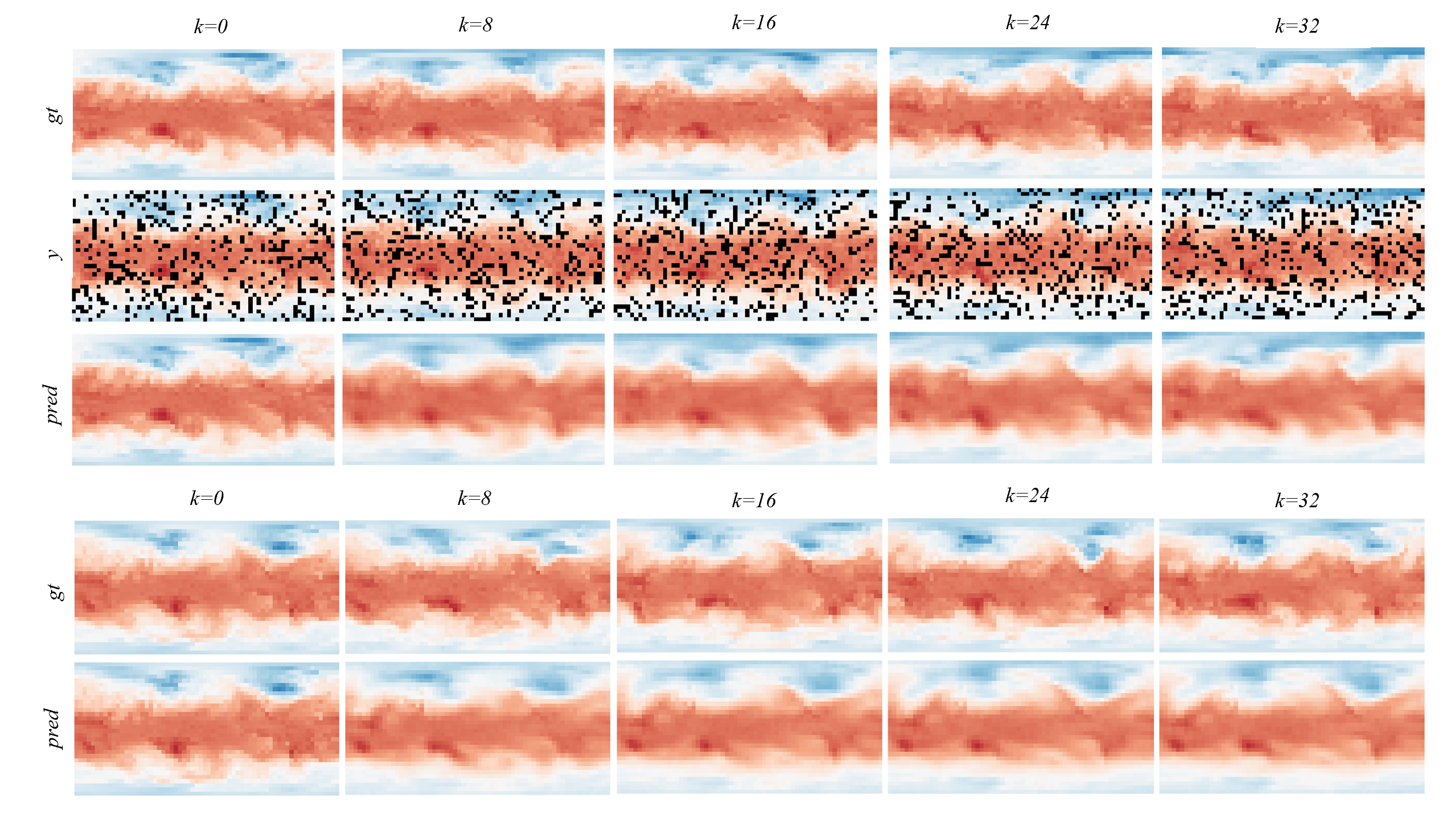}
    \caption{Top panel: Training sequence forecast with corresponding masked observations (middle row). Bottom panel: Evaluation sequence forecast from the held-out year. Each panel shows ground truth (top row) and model prediction (bottom row) at evenly spaced time steps. The model was trained using the diffusion resampler in the learning pipeline.}
    \label{fig:weather_forecast}
\end{figure}

For a more comprehensive comparison, there are limitations to address, such as the limited number of independent runs, choice of evaluation metric and the amount of data used for training. A simple MSE against the ground truth does not fully capture the predictive distribution quality and may not fully capture the forecasting performance in terms of fine-grained details and long forecasting windows. In both settings, the diffusion resampler achieves a performance comparable to the baselines. Training is stable across seeds, which shows that the diffusion resampler scales at latent dimension $d=2048$. 

\section{Choosing the hyperparameters}
Although the diffusion resampling is powerful, it comes with a variety of hyperparameters to tune: the diffusion coefficients $b$ and $\refmeasure$, the diffusion time $T$, the integrator, and integration step $K$. 
Ideally, $\refmeasure$ should well approximate the underlying continuous distribution, and the choice $b$ should correspondingly reflect the approximation error as in Corollary~\ref{corollary:asymp}. 
The setting of diffusion time $T$ is arbitrary if one can sample $p_T$ exactly, but smaller $T$ leads to potentially lower integration steps required. 
As for the integrator and integration steps, our empirical results indicate that Jentzen--Kloeden and Euler--Maruyama are often a good start, and $K\leq 8$ is often sufficient for learning large neural network-parametrised systems.
To retain a good computational complexity, it is suggested that the number of integration steps should be smaller than the number of samples: $K < N$.

\section{Additional related work}
\label{app:additional-related-work}
In addition to the discussion in Section \ref{sec:related-work}, there are also connections to other, less central, lines of work. For instance, the ensemble score estimator is related to the kernel projection used in Stein variational gradient flow \citep{Liu2017stein}. 
The controlled gain term in \citet{Yang2013} also plays a role analogous to resampling in a continuous-time limit, but involves a computationally demanding PDE-solving step. 
While there are interesting connections to our work, the approach focuses on continuous-time flow filtering~\citep[cf.][]{Kang2025} and leads to algorithms that are no longer standard SMC.
\citet{BaoEnsScore2024} propose an ensemble filter where the update step was replaced by a conditional version of the diffusion in Equation~\eqref{equ:diffres-rev}. 
This too targets at Equation~\eqref{equ:resampling-dist} but additionally contains errors from likelihood score approximation. 

\section{Take-away messages}
\label{app:tldr}
We here provide a TL;DR summary of the main findings, with Table~\ref{tbl:app-compare-all} giving an overall comparison among commonly used differentiable resampling schemes. 

\begin{itemize}
    \item Diffusion resampling largely excels for parameter estimation in state-space models (SSMs) and is useful for practical applications. 
    It is computationally fast, and provides consistent and stable gradient estimates.
    \item Diffusion resampling is useful not only for differentiation, but also for reducing resampling error in general. 
    The primary reason for this is that diffusion resampling can take additional information (e.g., $\refmeasure$) of the target into account. 
    When we have a complete information $\refmeasure=\pi$, the diffusion resampling becomes an optimal resampling algorithm. 
    Corollary~\ref{corollary:asymp} explicitly shows how the resampling error can be controlled depending on how well $\refmeasure$ approximates $\pi$, calibrated by $b$. 
    In contrast, multinomial resampling only uses information from the given samples $\lbrace (w_i, X_i) \rbrace_{i=1}^N$ which contains incomplete and less information compared to diffusion resampling using the continuous distribution approximation.
    \item For pure filtering tasks, even without focusing on parameter estimation, the diffusion resampling generally outperforms the peer methods. 
    \item If a good reference distribution is hard to construct, diffusion resampling may not be optimal for a pure resampling problem outside the SSM context. 
    For instance, as shown in the Gaussian mixture experiment, the large discrepancy between the Gaussian reference and the highly non-Gaussian target may prolong the diffusion time. 
    However, for SSMs, thanks to their sequential structure, constructing an effective reference becomes straightforward. 
\end{itemize}

\begin{table}[t!]
    \centering
    \caption{Comparison among commonly used resampling schemes with their calibration parameters. 
    The computational complexity is analysed per sample which can be embarrassingly parallelised over the samples. 
    The complexity of OT is given by~\citet{Luo2023OT}, where we here parametrise the Sinkhorn precision with the regularisation parameter $\varepsilon$ to unify comparison. 
    By ``fully differentiable'' we mean that the pathwise gradient is well defined, for instance, the soft resampling is only partially differentiable. 
    }
    \label{tbl:app-compare-all}
    \wrapbox{.9}{%
        \begin{tabular}{@{}lllll@{}}
        \toprule
        Method           & Fully differentiable      & Consistent                         & Unbiased & \begin{tabular}[c]{@{}l@{}}Computational complexity per re-sample\end{tabular} \\ \midrule
        Diffusion $K$    & Yes                       & Yes, as $N(K) \to\infty$           & No       & $O(K \, N)$                                                                      \\
        OT $\varepsilon$ & Yes                       & Yes, as $N(\varepsilon) \to\infty$ & No       & $O(N^2 \log(N)^{-1} \, \varepsilon^{-1})$                                             \\
        Gumbel $\tau$    & Yes, but not $\tau \to 0$ & No, except at $\tau \to 0$               & No       & $O(N)$                                                                           \\
        Soft $\alpha$    & No                        & Yes, as $N\to\infty$                                & No       & $O(N)$                                                                           \\
        Multinomial      & No                        & Yes                                & Yes, as $N\to\infty$      & $O(N)$                                                                           \\ \bottomrule
        \end{tabular}
    }
\end{table}

\end{document}